\documentclass{article}
\usepackage[margin = 1in]{geometry}
\usepackage{natbib}
\usepackage{microtype}
\usepackage{graphicx}
\usepackage{subcaption}
\usepackage{booktabs} % for professional tables
\usepackage{multirow}
\usepackage{algorithm}
\usepackage{algorithmic}
\usepackage{hyperref}
\usepackage{float}
\usepackage{overpic}
\usepackage{amsmath}
\usepackage{amssymb}
\usepackage{mathtools}
\usepackage{amsthm}
\usepackage{comment}

\usepackage[capitalize,noabbrev]{cleveref}

\theoremstyle{plain}
\newtheorem{theorem}{Theorem}[section]
\newtheorem{proposition}[theorem]{Proposition}
\newtheorem{lemma}[theorem]{Lemma}

\theoremstyle{definition}

\newtheorem{assumption}[theorem]{Assumption}
\theoremstyle{remark}

\newcommand{\x}[0]{{\mathbf{x}}}
\newcommand{\w}[0]{{\mathbf{w}}}

\renewcommand{\hat}{\widehat}
\renewcommand{\tilde}{\widetilde}

\newcommand{\E}[0]{\mathbb{E}}

\newcommand{\Sk}[0]{\mathbf{S}}
\newcommand{\Hm}[0]{\mathbf{H}}

\newcommand{\shuran}[1]{{}}
\newcommand{\dr}[1]{{}}

\newcommand{\df}[0]{\mathrm{d}}
\newcommand{\tr}[0]{\text{tr}}
\newcommand{\lh}[0]{\underline{h}}
\newcommand{\B}[0]{\mathbf{B}}
\usepackage[textsize=tiny]{todonotes}

\title{Explaining Data Mixing Scaling Laws}

\usepackage{authblk}

\author[1]{Rui Dai}
\author[2]{Shuran Zheng}

\affil[1]{Beijing Institute of Technology\\ \texttt{dairuircs@gmail.com}}
\affil[2]{IIIS, Tsinghua University\\ \texttt{shuranzheng@mail.tsinghua.edu.cn}}

\date{} 

\begin{document}
\maketitle



  % no special notice (required even if empty)

\begin{abstract}
Recent research has established empirical scaling laws to predict model performance on multi-domain data mixtures. However, a theoretical understanding of these model loss behaviors remains absent. In this work, we propose a unified framework to explain the underlying mechanics of data mixing. Our approach extends theoretical perspectives originally developed for standard neural scaling laws (e.g., Kaplan and Chinchilla) to the multi-domain setting. Based on the distributional assumption that domains overlap on fundamental skills while diverging on specialized skills, we identify two key factors that govern the domain losses of models trained on different data mixtures: \textit{Capacity Competition}, where the allocation of finite model capacity couples domain losses globally, and \textit{Noise Reduction}, where optimal weights shift toward harder-to-learn domains to minimize overall noise.
Empirical evaluations show that our framework outperforms existing baselines by fitting the loss landscape with a lower Mean Relative Error and identifying higher-performing training mixtures. Most importantly, our model successfully extrapolates across scales, predicting highly effective mixtures for large, unseen scales using parameters fitted on smaller ones. In addition, our model achieves these results using significantly fewer free parameters than previous empirical laws. Our code is available at \url{https://github.com/meiqwq/Explaining-Data-Mixing-Scaling-Laws}.

\end{abstract}

\section{Introduction}

Large foundation models are typically trained on data from multiple domains, with the data mixture---the proportion of each domain used---playing a critical role in model performance. However, discovering the optimal data mixture is often a highly costly process that lacks principled methodologies. Practitioners often rely on expensive trial-and-error or static heuristics.

To address this, recent research has moved toward principled methods that generally fall into two distinct paradigms: online adaptation and offline prediction. Online methods attempt to adjust domain weights dynamically during the training process based on the model's ongoing learning trajectory~\citep{albalak2023efficient, chen2023skillitdatadrivenskillsframework, jiang2024adaptive, chen2024aioli, li2025pike}. While these approaches can be effective in dynamic settings, they often incur computational overhead and remain theoretically opaque regarding how dynamic weighting ultimately impacts model capabilities.

In parallel, a prominent line of offline research has focused on predicting the loss landscape \textit{a priori} through data mixing scaling laws~\citep{shukor2025scaling,ye2024data,ge2024bimix,kang2024autoscale} and determining the mixture before training begins. These empirical frameworks attempt to predict a model's test loss on specific domains as a function of the mixing weights used during training. Several functional forms have been proposed to model this relationship, as summarized in~\Cref{tab:mixing_laws_intro}. These laws deviate from the standard power-form scaling laws and present non-trivial domain interaction: the loss on a domain depends not only on the weight of the domain itself, but also the weights of other domains, and this correlation does not exhibit a simple functional form. 

\begin{table}[h]
\centering
\caption{Examples of empirical functional forms for predicting domain loss $L_{i}(h)$ based on mixture weights $h$. $N$ and $D$ represent model parameters and training tokens, respectively.}
\label{tab:mixing_laws_intro}
\begin{tabular}{@{}ll@{}}
\toprule
\textbf{Reference} & \textbf{Functional Form ($f_i(h, N, D)$)} \\ \midrule
\citep{shukor2025scaling} & 
% 使用 aligned 环境进行手动换行和对齐
$\begin{aligned} 
L_{i} &\approx E_{i}+\left(\sum_{j=1}^{K}C_{ij}h_{j}^{\gamma_{ij}}\right)^{-1} \\
\end{aligned}$ \\ 
\addlinespace % 增加一点行间距，避免拥挤
\citep{ye2024data} & $L_{i}\approx c_{i}+k_{i}\exp\left(\sum_{j=1}^{K}t_{ij}h_{j}\right)$ \\ \bottomrule
\end{tabular}
\end{table}

Despite the practical utility of both dynamic algorithms and offline empirical fits, a rigorous theoretical understanding of the mechanics driving domain interaction remains largely absent. Focusing specifically on offline mixing laws, relying solely on empirically fitted curves presents a significant bottleneck. Not only is the fitting process resource-intensive, but the resulting laws act as ``black boxes'': it is unclear whether they generalize to larger scales or different datasets, nor is it obvious how to map the predicted domain loss to downstream task performance.

In this work, we propose a unified theoretical framework to explain the underlying mechanics of data mixing. Our framework extends previous theoretical frameworks for standard neural scaling laws---specifically the Quantization Model~\citep{michaud2023quantization,liu2025physics} and the Projected Linear Regression Model~\citep{lin2024, bordelon2024dynamical}---to the multi-domain setting. Based on a natural distributional assumption that different domains overlap on fundamental skills and diverge on specialized skills, we identify two key factors that determine the loss of models trained on different data mixtures:
\begin{itemize}
\vspace{-1mm}
    \item \textbf{Model Capacity Competition:} The model has a finite capacity and can only learn a finite number of skills. The specialized skills from different domains compete for the model capacity. Adjusting domain weights will change the importance of the skills and thus change the model capacity allocated to each domain. The resulting model capacity allocation introduces a non-trivial domain interaction, and is a key factor that determines the trained model's loss on a domain. 
\vspace{-1mm}
	\item \textbf{Noise Driven by Data Amount:} For each skill within the model's capacity, the loss incurred by the skill depends on the number of times that the model has seen the skill. As skills from different domains have different difficulty levels, the loss decreases at different rates as the domain weight increases; this dynamic shifts the optimal mixture weights toward domains that are harder to learn.
\end{itemize}

Leveraging this theoretical framework, we formulate loss prediction as a convex program that yields numerical estimates for arbitrary mixtures. Furthermore, we frame the search for the optimal training mixture as a bi-level optimization problem, which can be efficiently solved using Online Mirror Descent.

Empirically, our results validate the theoretical framework  across several key dimensions:
\begin{itemize}
\vspace{-1mm}
    \item \textbf{Superior Fitting Accuracy:} Our models fit the observed loss landscape with a lower Mean Relative Error (MRE) than existing empirical scaling laws.
\vspace{-1mm}
    \item \textbf{Optimal Mixture Prediction:} Our framework effectively identifies optimal training mixtures that yield the lowest test loss on the target average distribution.
    \vspace{-1mm}
\item \textbf{Cross-Scale Extrapolation:} Most importantly, our framework successfully extrapolates across scales, predicting highly effective mixtures for large, unseen model and dataset sizes using parameters fitted exclusively on smaller ones.
\vspace{-1mm}
    \item \textbf{Parameter Efficiency:} Crucially, we achieve these results while utilizing significantly fewer free parameters compared to leading empirical laws.

\end{itemize}

\section{Related Work}\label{relatedwork}

\paragraph{Data Mixture Selection.}
Optimizing the pre-training data mixture is critical for maximizing downstream model performance.  Data mixture can operate at various granularities, ranging from fine-grained token-level selection \citep{lin2024rho} to coarser domain-level mixture. Domain-level approaches are particularly advantageous due to their superior computational efficiency.  While early domain-level data mixture approaches relied on static heuristics or expensive trial-and-error, recent research has gravitated towards principled, compute-efficient strategies. These approaches generally fall into two paradigms: offline selection prior to training and online adaptation during training.
\begin{description}
	\item[Offline: Data Mixing Laws.] This line of work establishes empirical scaling laws for data mixtures to predict the loss landscape as a function of mixture ratios, from which the optimal mixture is derived. \citet{ye2024data} proposed an exponential-form data mixing law, which is extrapolated to larger model and data scales via the standard power-law scaling. \citet{ge2024bimix} and \citet{kang2024autoscale} introduced laws that jointly consider mixture ratios and data size (or training steps). Most recently, \citet{shukor2025scaling} formulated a unified law that explicitly incorporates model size, dataset size, and domain mixture ratios into a single scaling law. Moving beyond independent domain contributions, \citet{hamidieh2025domainaware} extended domain-aware scaling laws to explicitly model dataset interactions, quantifying both direct domain-to-benchmark effects and the second-order synergy or interference that arises when multiple data domains co-occur. \citet{medvedev2025shiftgoodmismatcheddata} analytically demonstrate that training on mismatched data proportions (positive distribution shift) almost always improves test performance. Other studies address specialized settings, including high-quality domain data \citep{gu2025data}, and data-constrained scenarios involving repeated tokens \citep{Muennighoff2023scaling}.
	\item[Offline: Proxy-Based Selection.] These methods determine optimal domain weights on small proxy models and transfer them to larger target models \citep{xie2023doremi, fan2023doge, liu2024regmix, zhang2025domain2vec, diao2025climb, wettig2025organize}. Recently, \citet{magnusson2025datadecide} investigated the efficacy of proxy mixtures when scaling to large models. 
	\item[Online Approaches:] In contrast to static offline selection, online methods adjust domain weights dynamically during the training process, including ODM~\citep{albalak2023efficient}, Skill-it~\citep{chen2023skillitdatadrivenskillsframework}, Aioli~\citep{chen2024aioli}, ADO~\citep{jiang2024adaptive}, and PiKE~\citep{li2025pike}. While these approaches can be effective, they often incur computational overhead and require additional validation loops during training.
\end{description}

\paragraph{Theoretical Foundations of Neural Scaling.}
While empirical research has established that neural network loss scales as a power law with respect to model size, data size, and compute \citep{kaplan2020scaling, hoffmann2022training}, deriving these exponents from first principles remains a central challenge. Two primary lines of research seek to explain the underlying mechanics of neural scaling laws:
\begin{itemize}
    \item \textbf{Linear Model Analysis:} One stream of research has sought to derive scaling laws through the analysis of linear models. Initial efforts focused on simplified environments, such as regression on fixed-dimension manifolds \citep{sharma2020neural}. \citet{maloney2022solvable} and \citet{bahri2021} expanded this setting to include generative data models and random feature models, demonstrating that power-law scaling arises in the dual limit of infinite data and infinite parameters. A more recent body of work focused on the training dynamics, specifically tracking one-pass Stochastic Gradient Descent within linear frameworks \citep{fonseca2024exactly, atanasov2024scaling, lin2024, bordelon2024dynamical, bordelon2025feature, li2025functional}. Using a solvable one-pass SGD model, \citet{paquette2024phases} characterized four compute-optimal scaling phases and three subphases governed by model capacity, optimizer noise, and feature embedding. \citet{bordelon2024dynamical} applied dynamic mean field theory to randomly projected linear models, recovering power laws in the asymptotic limit. \citet{lin2024} utilized a similar randomly projected linear model to reconcile neural scaling with traditional statistical theory, offering an explanation for why the classical variance error is unobservable when fitting the neural scaling law empirically. 

    \item \textbf{Skill Learning:}
Another line of work abstracts from complex training dynamics, instead viewing scaling laws as a consequence of ``skill learning'' \citep{michaud2023quantization,liu2025physics,arora2023theory,pan2025understanding}. \citet{michaud2023quantization} proposed the Quantization Model, which frames learning as the sequential acquisition of discrete ``quanta'' of skills distributed according to a power law, thereby recovering the observed power-form scaling. This framework was recently extended by \citet{liu2025physics} to three models with different levels of complexity.
\end{itemize}

\section{Preliminaries and Problem Description}

This section establishes the theoretical background and introduces the problem of data mixing. We begin by reviewing two primary theoretical frameworks that explain the standard neural scaling laws in the single-domain setting: the Quantization Model and the Linear Regression Model. Subsequently, we formally define the data mixing problem and survey existing empirical data mixing laws used to predict multi-domain loss.

\subsection{Theoretical Foundations of Standard Neural Scaling Laws}
\subsubsection{Quantization Model} \label{sec:pre_quant}
Standard neural scaling laws describe the predictable power-law relationship between model performance and scale. Empirically, the test loss $L$ is typically modeled as a function of both model size $N$ and dataset size $D$ \citep{kaplan2020scaling, hoffmann2022training}:
\begin{equation*}
    L(N, D) \approx \frac{A}{N^\alpha} + \frac{B}{D^\beta} + E.
\end{equation*} 
While these laws are well-established empirically, theoretical understanding of their origin remains an active area of research. Below, we review two major frameworks that attribute this power-law scaling to the intrinsic power-law structure of the data distribution.

\citet{michaud2023quantization} propose the \textit{Quantization Model}, which posits that the knowledge contained within large-scale training corpora decomposes into discrete knowledge units or skills termed ``quanta,'' $\mathcal{Q} = \{q_1, q_2, \dots\}$. These quanta are assumed to follow a Zipfian distribution where the $k$-th most frequent quantum has probability:
\begin{equation*}
    p(q_k) \propto k^{-\alpha}, \quad (\alpha > 1).
\end{equation*}
Crucially, the framework asserts that when trained on this dataset, a model learns skills in a deterministic order based on frequency to minimize the expected loss. Specifically, a model of effective capacity $N$ is assumed to learn the top $N$ most frequent quanta (i.e., $\{q_1, \dots, q_N\}$). Assuming that each unlearned quantum contributes a constant error $c$, the total expected loss is determined by the cumulative probability mass of the unlearned quanta in the tail:
\begin{equation*}
    L(N) = c \sum_{k=N+1}^\infty p(q_k) \approx c\int_{N}^{\infty} (\alpha- 1) k^{-\alpha} \, dk = c \cdot N^{-(\alpha-1)}.
\end{equation*}
As a result, the loss scaling exponent regarding model size $N$ is a direct consequence of the power-law distribution of skills inherent in the data.

\subsubsection{The Linear Regression Model.} \label{sec:pre_linear}
Despite its elegance, the Quantization Model is limited: it does not capture the stochasticity inherent in training dynamics.
Building on similar intuitions, \citep{bordelon2024dynamical, lin2024}  provide a rigorous derivation of scaling laws by analyzing the training dynamics of linear regression under one-pass Stochastic Gradient Descent (SGD). In this framework, neural scaling laws are governed by the spectral decay of the data. As training progresses, the model ``resolves'' eigenmodes in descending order of their eigenvalues---learning the dominant patterns first before fitting the fine-grained details.
Below, we adopt the formal framework from \cite{lin2024} to provide a simplified theoretical explanation.

    \textbf{Data Generation:} Consider a linear regression problem where the input covariates $\mathbf{x} \in \mathbb{R}^d$ (where $d$ can be infinite) are drawn from a distribution with zero mean and covariance matrix $\mathbf{H} = \mathbb{E}[\mathbf{x}\mathbf{x}^\top]$. The target label $y$ is generated by a linear teacher with additive noise:
    \begin{equation*}
        y = \langle \theta_*, \mathbf{x} \rangle + \epsilon,
    \end{equation*}
    where $\theta_*$ is the ground-truth parameter and $\epsilon \sim \mathcal{N}(0, \sigma^2)$ is independent Gaussian noise.

    \textbf{Spectral Assumptions:} When training a linear regression model with one-pass SGD, the learning dynamics are determined by the spectrum of the covariance matrix $\mathbf{H}  = \mathbb{E}[\mathbf{x}\mathbf{x}^\top]$. Intuitively, the eigenvalues $\lambda_k$ of $\mathbf{H}$ represent the variance (or signal strength) of the data along the $k$-th principal component.
    Empirical studies on natural data (e.g., images and text) consistently observe that these eigenvalues follow a power distribution \citep{field1987, bahri2021}. It is therefore assumed that the eigenvalues, sorted in descending order, follow a power law:
    \begin{equation*}
        \lambda_k \propto k^{-\alpha}, \quad \text{for } \alpha > 1.
    \end{equation*}

    \textbf{Finite Parameter Projection:} To model a neural network with a finite capacity of $N$ parameters, we project the high-dimensional input $\mathbf{x}$ into a lower-dimensional feature space using a ``sketching matrix'' $\mathbf{S} \in \mathbb{R}^{N \times d}$. The model learns a weight vector $\theta \in \mathbb{R}^N$ by minimizing the squared error on the projected features $\tilde{\mathbf{x}} = \mathbf{S}\mathbf{x}$:
    \begin{equation*}
        \widehat\theta = \operatorname*{arg\,min}_{\theta \in \mathbb{R}^N} \frac{1}{D} \sum_{i=1}^D (y_i - \theta^\top \mathbf{S} \mathbf{x}_i)^2.
    \end{equation*}

    \textbf{Result:} Under this setup, \cite{lin2024} derive the scaling law for the test loss $L(N,D)$ of the projected linear model $\widehat{\theta}$ (trained via one-pass SGD) as a function of the model size $N$ and number of training samples $D$:
    \begin{equation*}
        L(N, D) \approx \underbrace{O\left(\frac{1}{N^{a_1}}\right)}_{\text{Model Scaling}} + \underbrace{O\left(\frac{1}{D^{a_2}}\right)}_{\text{Data Scaling}} + E.
    \end{equation*}

\subsection{Empirical Data Mixing Laws}

In this section, we formally define the problem of data mixing and review prior empirical work on scaling laws in the multi-domain setting.

\paragraph{Problem setup.} 
Consider a set of $K$ data domains $\mathcal{D} = \{\mathcal{D}_1, \dots, \mathcal{D}_K\}$, where distinct domains represent different data sources such as GitHub, books, and Wikipedia. We construct a training data mixture by sampling from these domains according to a domain weight vector $h \in \Delta^{K-1}$, where $\Delta^{K-1}$ denotes the probability simplex
    $\Delta^{K-1} = \left\{ h \in \mathbb{R}^K \mid \sum_{i=1}^K h_i = 1, \ h_i \ge 0 \right\}$.
Let $\theta(h)$ denote the parameters of a model trained on this mixture with fixed model size $N$ and training token count $D$.

The primary objective of data mixing is to determine the optimal domain weights $h^*$ that minimize the trained model's loss on a specific target distribution defined by importance weights $w \in \Delta^{K-1}$, under a fixed compute budget parameterized by model size $N$ and training token count $D$. Because exhaustively training large models to find $h^*$ is computationally prohibitive, practitioners fit empirical scaling laws to predict the held-out test loss $L_i$ for each domain as a function of the mixing weights $h$ given the scale parameters $N$ and $D$:
\begin{equation*}
    L_i(h \mid N, D) \approx f_i(h \mid N, D).
\end{equation*}
Once the functional form $f_i$ is fitted, the optimal training mixture $h^*$ for a target distribution defined by importance weights $w \in \Delta^{K-1}$ is estimated by solving the following minimization problem:
\begin{equation*}
    h^* = \operatorname*{arg\,min}_{h \in \Delta^{K-1}} \sum_{i=1}^K w_i f_i(h \mid N, D).
\end{equation*}

\paragraph{Empirical data mixing laws.}
Recent research has proposed various functional forms to model the relationship between the mixing weights $h$ and the resulting domain losses. The standard methodology involves sampling a diverse set of mixture configurations from the simplex $\Delta^{K-1}$, training small-to-medium-scale models on these mixtures, and fitting the coefficients of $f_i$ to the observed test losses. Table~\ref{tab:mixing_laws} summarizes the key functional forms established in the literature.

\section{Theoretical Framework for Data Mixing}

In this work, we propose a unified theoretical framework to explain the underlying mechanics of data mixing. We extend two single-domain perspectives---the Quantization Model \citep{michaud2023quantization} and the Linear Regression Model \citep{lin2024}---to the multi-domain setting. We first introduce the Extended Quantization Model, which frames training as a capacity allocation problem. This serves as the foundation for the Extended Linear Regression Model, which implicitly solves this allocation problem while incorporating data-dependent noise.

\subsection{The ``Shared Head, Disjoint Tail'' Structure} \label{subsec:sharedhead}

Building on \citep{pan2025understanding}, we introduce a natural structural assumption regarding how information overlaps across different data domains (e.g., Code, Math, English). Intuitively, most domains share a common foundation of basic knowledge while diverging in specialized topics. We formalize this as the ``Shared Head, Disjoint Tail'' structure:
\begin{itemize}
    \item \textbf{Power-Law Distribution:} Within each domain $i$, knowledge units (skills) follow a power-law distribution in terms of frequency.
    \item \textbf{Shared Head:} Different domains largely overlap in the head of the distribution---the region of high-probability, fundamental skills (e.g., basic grammar, logic, or arithmetic).
    \item \textbf{Disjoint Tail:} As we move to the tail of the distribution (rare, specialized knowledge), the domains become increasingly distinct and independent.
\end{itemize}
While real-world data is unlikely to be strictly disjoint in the tail, this idealization serves as a useful approximation. Based on our modeling in the following sections, when the data size is fixed and only the mixture weights vary, the change in loss induced by overlapping skills is likely to be small compared to the change induced by disjoint skills. In other words, the loss induced by overlapping skills will behave more like a constant compared to disjoint skills. We discuss the robustness of this approximation in detail in subsequent sections. Furthermore, we empirically validate our model against violations of this assumption through a synthetic stress test (\Cref{fig:stress_test}), demonstrating that predictive accuracy remains stable even as tail overlap increases.

\begin{figure}[t]
    \centering
    % \linewidth 会自动撑满当前的单栏宽度
    \includegraphics[width=0.5\linewidth]{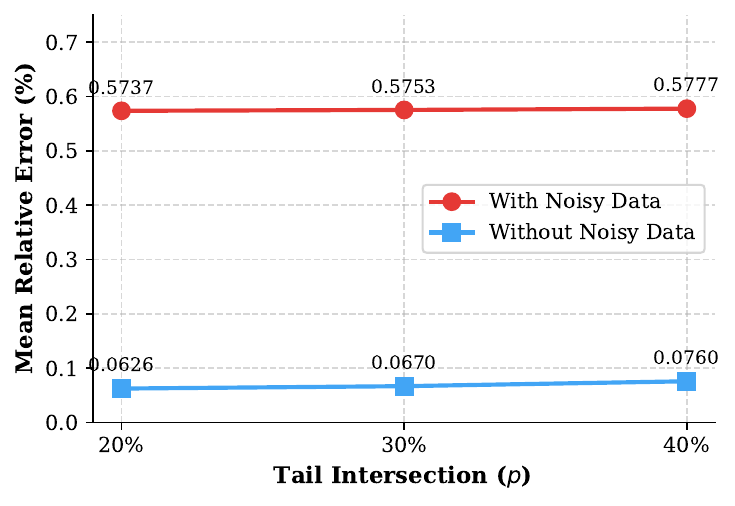}
    \caption{Fitting errors of our disjoint-tail theoretical model on synthetic data under varying degrees of actual tail overlap ($p$), where $p$ represents the proportion of tail skills randomly selected to be shared across different domains. Each skill $k$ is an overlapping skill with probability $p$. We evaluate the model under two experimental settings: a noiseless setting and a noisy setting where Gaussian noise drawn from $\mathcal{N}(0, 1)$ is added to the losses. The plot demonstrates that the Mean Relative Error (MRE) remains consistently low and stable even as tail overlap increases up to $40\%$, highlighting the robustness of our proposed model.}
    \label{fig:stress_test}
\end{figure}

\subsection{Extended Quantization Model} \label{sec:theory_quant}

We first extend the Quantization Model (\Cref{sec:pre_quant}) to a multi-domain formulation. This extension serves as a foundation for the Extended Linear Regression Model in the next section. 
 
\paragraph{Single-Domain Skill Distribution and Loss.} We associate each domain $\mathcal{D}_i$ with a continuous space of skills indexed by $k_i \in [1, \infty)$ with power-law density
\begin{equation*}
    p_i(k_i) = (\alpha_i - 1) k_i^{-\alpha_i}, \quad \text{for } \alpha_i > 1.
\end{equation*}
We consider a model with effective capacity $N$, representing the total volume of unique skills it can learn. Following the Quantization Model, we frame training as a \textbf{capacity allocation problem}. For each domain $i$, the model selects a coverage threshold $x_i \ge 1$,  learning all high-frequency skills up to this cutoff ($k_i \le x_i$) while discarding the tail. Assuming that each unlearned quantum in domain $i$ contributes a constant error $c_i$, the training loss incurred on domain $i$ is determined by the probability mass of the unlearned tail ($k_i > x_i$):
\begin{equation*}
  c_i \int_{x_i}^{\infty} (\alpha_i - 1) k_i^{-\alpha_i} \, dk_i = c_i x_i^{-(\alpha_i - 1)} = c_i x_i^{-b_i},
\end{equation*}
where $b_i = \alpha_i - 1$.

\paragraph{Multi-Domain Data Assumptions and Loss Minimization.} Under the ``Shared Head, Disjoint Tail'' assumption, we assume that the skill spaces are aligned such that the interval $k_i \in [1, H]$ corresponds to foundational knowledge shared across all domains (the ``base'' skills). We assume that the model capacity is sufficient to acquire all shared base skills. Consequently, the capacity allocation problem reduces to allocating the remaining model capacity to the disjoint tails. When trained on a data mixture $h = (h_1, \dots, h_K)$, the model determines the coverage thresholds $x_i$ to minimize the expected training loss:
\begin{equation}
\label{eq:quant_opt}
\begin{aligned}
    \min_{x} \quad & \sum_{i=1}^K h_i c_i x_i^{-b_i} \\
    \text{s.t.} \quad & \sum_{i=1}^K (x_i - H) \le N - H, \\
    & x_i \ge H, \quad \forall i.
\end{aligned}
\end{equation}
%Here, the constraint $\sum (x_i - H) \le N -H$ reflects that learning the unique, domain-specific skills (the ``tails'' beyond $H$) consumes the finite effective capacity $N$ after the shared baseline is accounted for.
Let $x^*(h) = (x^*_1(h), \dots, x^*_K(h))$ denote the optimal solution to the optimization problem in Eq.~\ref{eq:quant_opt}, which depends on the mixture weights $h$. Under this model, the expected test loss on domain $i$ is
\begin{equation*}
    L_i(h) = c_i (x^*_i(h))^{-b_i} + E_i,
\end{equation*}
where $E_i$ is an additional irreducible loss.

\paragraph{Domain Interaction: Capacity Competition.} Under this formulation, domain interaction is driven by model capacity competition. The constraint $\sum (x_i - H) \le N -H$ couples the domains together in a competition for model resources.  In particular, when the scaling exponents are similar ($b_i \approx \bar{b}$) and the model capacity is large ($x_i \gg H$), we can derive an approximate closed-form solution using Lagrange multipliers:
\begin{equation*}
    x^*_i(h) \approx \frac{\big(b_i c_i h_i N^{\bar{b} +1}\big)^{\frac{1}{b_i+1}}}{ \left( \sum_{k=1}^K (b_k c_k h_k)^{\frac{1}{b_k+1}} \right)^{\frac{\bar{b}+1}{b_i+1} }} .
\end{equation*}
The resulting approximate loss on domain $i$ is given by $L_i(h) = c_i (x^*_i(h))^{-b_i}$.
Crucially, this solution shows that the loss $L_i(h)$ on each domain depends on the ``aggregate demand'' of the mixture, represented by the denominator term $\sum (b_k c_k h_k)^{\frac{1}{b_k+1}}$. This term creates an explicit coupling: domain $i$'s loss is driven not just by its own properties, but by the weights and complexities of every competing domain (e.g., $h_j$ for $j \neq i$).

\paragraph{The Effect of Tail Overlap.} Now suppose the tails of different domains overlap. To minimize expected loss, the model learns the $N$ skills with the highest expected loss, denoted by $c(\text{skill}) \cdot p(\text{skill})$, where $p(\text{skill})$ is the aggregate probability density across all domains and $c(\text{skill})$ is the loss incurred by not learning the skill. So the status of a skill is binary: learned vs.\ unlearned. Consequently, when the mixture weights vary,  the loss fluctuation caused by overlapping skills is likely to be smaller than that of disjoint skills. This is because the aggregate probability density $p(\text{skill})$ of an overlapping skill will be relatively stable as the mixture shifts, making its status (learned vs.\ unlearned) less sensitive to weight changes.  In other words, the loss induced by overlapping skills will behave more like a constant compared to disjoint skills. 

\paragraph{Limitation.}
While this formulation naturally extends the Quantization Model, it has a critical limitation when predicting the optimal mixture. 
We formulate the search for the optimal mixture as a bi-level optimization problem. The outer objective minimizes the loss on a target distribution defined by importance weights $w = (w_1, \dots, w_K)$, while the inner optimization determines the resource allocation $x^*(h)$ given the training mixture $h$:
\begin{align*}
    h^* &= \arg\min_h \sum_{i=1}^K w_i L_i(h) \\
        &= \arg\min_h \sum_{i=1}^K w_i
        \left(c_i (x^*_i(h))^{-b_i} + E_i\right).
\end{align*}
Here, $x^*_i(h)$ is the solution to the capacity allocation problem in Eq.~\eqref{eq:quant_opt}, and $\sum_i w_i E_i$ is a constant and can thus be removed, so we have
\begin{equation} \label{eqn:opt_mixture_quan}
\begin{aligned}
    h^* &= \arg\min_h \sum_{i=1}^K w_i L_i(h) \\
        &= \arg\min_h \sum_{i=1}^K w_i c_i (x^*_i(h))^{-b_i}.
\end{aligned}
\end{equation}
Crucially, because the inner optimization (Eq.~\ref{eq:quant_opt}) minimizes a weighted sum of losses with respect to $h$, and the outer objective (Eq.~\ref{eqn:opt_mixture_quan}) minimizes a weighted sum of the same losses with respect to $w$, the bi-level problem is trivially solved by setting the training weights equal to the target weights $h^* \equiv w.$

This result contradicts empirical observations, where the optimal training mixture often deviates significantly from the target distribution. To resolve this discrepancy, we proceed to introduce the Extended Linear Regression Model.

\subsection{Extended Linear Regression Model}\label{sec:elrm}

To address the limitations of the Extended Quantization Model and incorporate training dynamics, we extend the linear regression framework of \citep{lin2024} to the multi-domain setting. This Extended Linear Regression Model can be viewed as an extension of the previous Extended Quantization Model as well: the training process implicitly solves the capacity allocation problem defined in~\eqref{eq:quant_opt}, while introducing an additional noise term.

\paragraph{Problem Formulation.}
Following~\Cref{sec:pre_linear}, we consider a linear regression problem over a union of $K$ domains. For each domain $i \in \{1, \dots, K\}$, input covariates $\mathbf{x}_i \in \mathbb{R}^d$ (where $d$ can be infinite) are drawn from a distribution with zero mean and a covariance matrix defined as $\mathbf{H}_i = \mathbb{E}_{\mathcal{P}_i}[\mathbf{x}_i \mathbf{x}_i^\top]$. The label $y_i$ is generated by a global linear teacher $y_i = \langle \theta^*, \mathbf{x}_i \rangle + \epsilon_i$, where $\theta^* \sim \mathcal{N}(0,\mathbf{I})$ is the ground-truth parameter and noise $\epsilon_i \sim \mathcal{N}(0,\sigma_i^2)$.
We consider a mixture distribution $\mathcal{P}(h)=\sum\limits_{i\in [K]}h_i \mathcal{P}_i$ defined by weights $h \in \Delta^{K-1}$  with covariance matrix $\mathbf{H}(h) = \sum_{i=1}^K h_i \mathbf{H}_i$.
%satisfying Assumption~\ref{assum:hyper} and $\mathbb{E}_{\mathbf{x} \sim \mathcal{P}(h)}[\mathbf{x}\mathbf{x}^{\top}]=\mathbf{H}(h)$.  
Following~\citep{li2025functional}, we model a neural network with $N$ parameters by projecting the high-dimensional input $\mathbf{x}$ into an $N$-dimensional feature space using a ``top-$N$ sketching matrix'' $\mathbf{S} \in \mathbb{R}^{N \times d}$ (detailed in~\Cref{app:theory_des}). The model learns a weight vector $\theta \in \mathbb{R}^N$ by minimizing the squared error on the projected features $\tilde{\mathbf{x}} = \mathbf{S}\mathbf{x}$.

\paragraph{Spectral Assumption: Shared Head and Disjoint Tails.}
We then formalize the ``Shared Head, Disjoint Tail'' assumption. In spectral analysis, an eigenvector represents a pattern of variation in the data (e.g., a specific texture in images or topic in text), while its corresponding eigenvalue quantifies the variance (or strength) of that pattern.
Intuitively, we assume real-world data consists of \textit{universal patterns} shared across all domains (the head) and \textit{specialized nuances} unique to each domain (the tails).
To make the analysis tractable, we assume that all covariance matrices $\{\mathbf{H}_i\}_{i=1}^K$ share a common orthonormal basis of eigenvectors $\mathbf{U} = [u_1, \dots, u_d]$ (i.e., they are simultaneously diagonalizable). In this basis, each matrix $\mathbf{H}_i$ is diagonal with eigenvalues denoted by $\lambda^{(i)}_k$. We classify these eigenvectors into two categories:

\begin{itemize}
    \item \textbf{Shared Head ($k \le H$):} The first $H$ eigenvectors $\{u_1, \dots, u_H\}$ represent universal components. We assume all domains possess non-zero variance along these directions ($\lambda^{(i)}_k > 0$ for all $i$), meaning these patterns are present in every domain.
    
    \item \textbf{Disjoint Tail ($k > H$):} The remaining eigenvectors represent domain-specific components. We assume each domain $i$ possesses a unique set of eigenvectors $\{u_k^{(i)}\}$ that are orthogonal to the specific components of other domains. Following~\Cref{sec:pre_linear}, we assume that the variance along these directions follows a domain-specific power law:
    \begin{equation*}
        u_k^{(i)\top} \mathbf{H}_j u_k^{(i)} = 
        \begin{cases}
            k^{-\alpha_i} & \text{if } j = i \quad  \\
            0 & \text{if } j \neq i \quad .
        \end{cases}
    \end{equation*}
    Consequently, for $k > H$, the spectral structures are completely decoupled: each domain $i$ has positive eigenvalues $\lambda^{(i)}_k = k^{-\alpha_i}$ along its own unique eigenvectors and zero along the eigenvectors of others.
\end{itemize}
Under this assumption, the mixture covariance matrix $\mathbf{H}(h) = \sum h_j \mathbf{H}_j$ exhibits a decoupled structure in the tail. Because the domain-specific eigenvectors do not overlap, the eigenvalue of a unique component from domain $i$ in the mixture is simply its original variance scaled by the domain's proportion $h_i$. 
Specifically, for a tail eigenvector $u_k^{(i)}$ belonging to domain $i$, the corresponding eigenvalue in the mixture is:
\begin{equation*}
    \lambda(\mathbf{H}(h), u_k^{(i)}) = \sum_{j=1}^K h_j \cdot u_k^{(i)\top} \mathbf{H}_j u_k^{(i)} = h_i k^{-\alpha_i}.
\end{equation*}

\paragraph{Connection with the Extended Quantization Model.}
Here, each eigenvector $u_k^{(i)}$ (a pattern of variation in the data) can be viewed as a skill to be learned, and its corresponding mixture eigenvalue $\lambda(\mathbf{H}(h),u_k^{(i)}) = h_i k^{-\alpha_i}$ (the variance of that pattern) can be viewed as the expected loss incurred if this skill is not learned (analogous to the $c(\text{skill}) \cdot p(\text{skill})$ term in our previous discussion). Therefore, our spectral assumption parallels the skill distribution assumption in the Extended Quantization Model. Furthermore, adjusting a domain's weight $h_i$ produces an equivalent effect in both frameworks: reducing $h_i$ linearly shrinks all mixture eigenvalues associated with domain $i$; similarly, in the Extended Quantization Model, reducing $h_i$ linearly shrinks the expected loss of skills in domain $i$ by shrinking $p(\text{skill})$.

\paragraph{Loss Analysis.}
Under the spectral assumptions established above, we analyze the expected test loss $L(h, N, D)$ of a projected linear model $\theta \in \mathbb{R}^N$ trained via one-pass SGD on a dataset of size $D$, sampled from a mixture distribution with weights $h \in \Delta^{K-1}$. 
The expected domain loss is governed by two primary factors. First, the training process implicitly solves the capacity allocation problem defined in the Extended Quantization Model~\eqref{eq:quant_opt}, distributing the capacity $N$ across domains to the most valuable skills; consequently, the unlearned skills in each domain contribute a loss similar to $L_i(h) = c_i (x^*_i(h))^{-b_i}$ in the Extended Quantization Model. Second, for each skill within the model's capacity, the loss will depend on the number of times it has been observed; in other words, the stochastic nature of one-pass SGD introduces a noise term to the domain test loss $L_i$, which is determined by the number of training samples drawn from that domain.

\begin{theorem}[Informal]
\label{thm:multidomain_loss}
Given data size $D$ and model size $N$, assume domains have mutually disjoint tails with eigenvalues $\propto k^{-\alpha_i}$ ($\alpha_i > 1$) and negligible shared head error. For any training mixture $h \in \Delta^{K-1}$, let $b_i = \alpha_i - 1$, $c_i=1/(\alpha_i-1)$, and let $x^*(h,N)$ be the optimal solution of the Extended Quantization Model \eqref{eq:quant_opt} defined by model capacity $N$, mixture weights $h$, parameters $\{b_i\}$ and $\{c_i\}$. 
For a projected linear model $\theta \in \mathbb{R}^N$ trained via one-pass SGD on $D$ samples drawn from a mixture $h$, its expected test loss on domain $i$, denoted $L_i(h \mid N, D)$, satisfies
\begin{equation} \label{eqn:mixture_obj}
    L_i(h \mid N, D) \approx c_i x_i^*(h,N)^{-b_i} + A_i (D h_i)^{-a_i} + E_i
\end{equation}
 where $a_i,A_i,E_i$ are constants that depend on $\alpha_i$.
\end{theorem}
We defer the formal theorem to \Cref{app:main_theorem}.

\paragraph{Optimal Mixture and Symmetry Breaking.} While capacity competition $x_i^*(h,N)^{-b_i}$ drives domain interaction, the noise term $A_i (D h_i)^{-a_i}$ depends solely on $h_i$. When finding the optimal training mixture $h^*$ for a target $w$:
\begin{equation} \label{eqn:opt_mixture_linear}
    h^* = \arg\min_{h} \sum_{i=1}^K w_i \big(c_i x_i^*(h,N)^{-b_i} + A_i (D h_i)^{-a_i} + E_i \big)
\end{equation}
the data-dependent noise term $A_i (D h_i)^{-a_i}$  breaks the symmetry with Eq.~\ref{eq:quant_opt}. Consequently, $h^*$ deviates from $w$, shifting weight toward domains that are ``harder to learn'' (larger $A_i$ and smaller $\alpha_i$). 

We solve this bi-level problem using Online Mirror Descent (OMD) based on the following characterization:

\begin{proposition}[Gradient Characterization]
\label{thm:bilevel_gradient}
Let $x^*(h)$ and $\lambda(h)$ be the optimal solution and Lagrange multiplier of the capacity allocation problem (Eq.~\ref{eq:quant_opt}). The gradient of the outer objective $\mathcal{J}(h) := \sum_{i=1}^K w_i L_i(h \mid N, D)$ with respect to $h_k$ is:
\begin{align*}
    \nabla_k \mathcal{J}(h)
    &= -w_k a_k A_k D^{-a_k} h_k^{-a_k-1} \\
    &\quad + \frac{\lambda x_k^*}{h_k(b_k+1)}
    \left(\bar{R}-\frac{w_k}{h_k}\right).
\end{align*}
where $\bar{R} = \frac{\sum_{j=1}^K \frac{x_j^*}{b_j+1} \left( \frac{w_j}{h_j} \right)}{\sum_{j=1}^K \frac{x_j^*}{b_j+1}}$.
\end{proposition}
We defer the full algorithm to \Cref{app:algorithm}.

\paragraph{Tail Overlap.} 
    As established in \Cref{sec:theory_quant},  overlapping skills contribute less to the fluctuation of $c_i x_i^*(h,N)^{-b_i}$ compared to disjoint skills. Similarly, an overlapping skill's aggregate observation count $D \cdot p(\text{skill})$ is more stable against mixture shifts, so its contribution to the noise term $A_i (D h_i)^{-a_i}$ will be more stable as well. Thus, the overall loss fluctuation induced by overlapping skills remains small relative to disjoint skills.

\section{Experiments}
\label{sec:experiments}

In this section, we empirically validate our proposed theoretical framework. We focus on three primary objectives:
\begin{enumerate}
    \item \textbf{Predictive Accuracy:} We evaluate how well our theoretical model fits the observed loss landscape under various data mixtures compared to existing empirical baselines. We assess accuracy using the domain-averaged Mean Relative Error (MRE) and Mean Absolute Error (MAE).
    \item \textbf{Optimal Mixture Identification:} We leverage the fitted scaling laws to predict the optimal training mixture $h^*$ and evaluate the test loss of the resulting models on the target distribution.
   \item \textbf{Cross-Scale Extrapolation:} We test our model's ability to extrapolate to larger, unseen scales. Specifically, we fit our model parameters exclusively using small-scale test losses. We then substitute the target scale variables (e.g., model size $N$ and token budget $D$) into the fitted model to predict the optimal data mixture at the target scale. Finally, we empirically test this extrapolated mixture to validate its performance.
    
\end{enumerate}

\subsection{Experimental Setup}

\paragraph{Models and Data.} We first introduce the models and datasets used in our experiments. 
\begin{itemize}
\item\textbf{Fitting Accuracy.}	
To evaluate predictive accuracy, we reuse the experimental data from \citep{liu2024regmix}, which contain checkpoints for 64 1B-parameter models trained on different random mixtures. These 1B-parameter models were trained on 25B tokens spanning 17 domains from the training split of the Pile dataset~\citep{gao2020pile} (see Table 1 in \citep{liu2024regmix} for full specifications). Their held-out losses are evaluated on the validation split of the same Pile dataset.\footnote{For four domains where validation splits were unavailable in the original configuration~\citep{liu2024regmix}, we use the training loss as a proxy for the held-out loss.} The models follow the TinyLlama architecture~\citep{zhang2024tinyllama}, which uses the Llama 2 tokenizer and architecture, and were trained using a cosine schedule with a learning rate of $4 \times 10^{-4}$.
\item \textbf{Optimal Mixture.} 
To evaluate the optimal data mixtures predicted by various scaling laws, we conduct experiments across three distinct settings: (1) A 4-domain configuration (Wikipedia, GitHub, StackExchange, and PG-19) adapted from Section 6 of \citep{shukor2025scaling}. Here, we train 25 domain mixtures—sampled from Pile-CC subsets—on 200M-parameter LLaMA-style models \citep{brown2020language} for 8B tokens using a cosine learning rate schedule. (2) The primary experimental setup from Section 3 of \citep{shukor2025scaling}, which evaluates mixtures across varying model and token scales on the SlimPajama dataset with 7 domains using a constant learning rate. For this setting, we utilize the (122M parameters, 10B tokens) and (310M parameters, 20B tokens) configurations. (3) The aforementioned \textbf{17-domain} Pile setting \citep{liu2024regmix}, where we optimize the training mixture to minimize test loss on the target Pile-CC distribution.  
\item \textbf{Cross-Scale Extrapolation.} 
To evaluate the extrapolation capabilities of our theoretical model, we conduct experiments across two distinct settings: (1) Using our 4-domain configuration (Wikipedia, GitHub, StackExchange, and PG-19), we extrapolate from a 200M-parameter model to a 700M-parameter model trained on 16B tokens. (2) Using the primary experimental setup from Section 3 of \citep{shukor2025scaling} with the 7-domain SlimPajama dataset, we extrapolate from the (122M parameters, 10B tokens) configuration to a larger (1B parameters, 30B tokens) configuration.
\end{itemize}

\paragraph{Baselines and Metrics.}
We benchmark our model against four empirical scaling laws from prior work: the Additive Law~\citep{shukor2025scaling}, RegMix~\citep{liu2024regmix}, the Exponential Law~\citep{ye2024data}, and BiMix~\citep{ge2024bimix}. Fitting accuracy is evaluated via Mean Relative Error (MRE), while the quality of the predicted optimal mixtures is measured using the held-out metric reported for each setting (test loss or perplexity).

\subsection{Fitting Accuracy} \label{sec:fit_acc}

\paragraph{Procedure.} 
To evaluate each scaling law's ability to fit the loss landscape, we partition the 64 1B-parameter models, each trained on different mixtures, into a training set for fitting the scaling laws and a test set for assessing predictive accuracy. Goodness-of-fit is measured via the Mean Relative Error (MRE) between the predicted loss $\hat{L}$ and the ground truth loss $L$ on the held-out mixtures, averaged over all domains.  Formally, the MRE is defined as
$\text{MRE} = \frac{1}{| \mathcal{H}_{test} |\cdot K} \sum_{h \in \mathcal{H}_{test}} \sum_{i=1}^K \left| \frac{\hat{L}_i(h) - L_i(h)}{L_i(h)} \right|$,
where $\mathcal{H}_{test}$ is the set of held-out mixture weights.

\paragraph{Fitting Methods.} 
To reliably fit the parameters of each scaling law, we select fitting methods based on model complexity. For simple laws (power and exponential), such as those in~\citep{ge2024bimix,ye2024data}, we estimate parameters using \texttt{curve\_fit} from the \texttt{scipy} library with multiple random initializations. For RegMix~\citep{liu2024regmix}, we use LightGBM following the original implementation, setting the number of trees to $T=100$ and the number of leaves to $L=31$. For the Additive Law~\citep{shukor2025scaling}, our Extended Quantization Model (with parameters $c_i,b_i,E_i$), and our Extended Linear Regression Model (with parameters $c_i,b_i,A_i,a_i,E_i$), we adopt the non-convex optimization strategy from~\citep{shukor2025scaling}. This approach uses the basin-hopping algorithm from \texttt{scipy}, with L-BFGS as the inner routine and multiple random initializations. When running L-BFGS for our models, we obtain the function value by solving the convex program~\eqref{eq:quant_opt} and use Mean Squared Error (MSE) as the fitting objective.
Additionally, for laws containing terms such as $(D h_i)^{-a}$, we exclude observations with $h_i=0$ to prevent division by zero.

\paragraph{Results.}
As shown in Table~\ref{tab:mre_results}, our two proposed models achieve the lowest and second-lowest MRE (and MAE), with average MRE as low as $1.53\%$ and $2.06\%$. Notably, both models outperform the leading baseline~\citep{shukor2025scaling} while utilizing significantly fewer free parameters. This confirms that our theoretically grounded models---based on capacity allocation principles---capture the underlying mechanics of data mixing more accurately than purely heuristic curve-fitting. 

\begin{table}[ht]
    \centering
    \caption{Comparison of fitting accuracy on 64 1B-parameter models trained on $K=17$ domains from the Pile dataset. Our theoretically grounded models achieve the lowest error rates (MRE and MAE) while using significantly fewer free parameters than heuristic baselines.}
    \label{tab:mre_results}
    \begin{tabular}{l c  c c}
    \toprule
    \textbf{Method} & \textbf{MRE ($\%$)} $\downarrow$ & \textbf{MAE} $\downarrow$ & \textbf{\#Param} \\
    \midrule
    \multicolumn{4}{l}{\textit{Empirical Baselines}} \\
    Additive &  2.209 & 0.052 & $K(2K+1)$\\
    Exponential &  6.990 & 0.059 & $K(K+2)$\\
    BiMix  & 2.963 & 0.144 & $2K$\\
    RegMix  & 6.480 & 0.136 & $K^2$\\
    \midrule
    \multicolumn{4}{l}{\textit{Our Models}} \\
    Ours (Eq. \eqref{eq:quant_opt})& 2.064 & 0.051 & $3K$\\
    \textbf{Ours (Eq. 3)}  & \textbf{1.533} & \textbf{0.034} & $5K$ \\
    \bottomrule
    \end{tabular}
\end{table}

\subsection{Predicting the Optimal Data Mixture}\label{exp:4dom}

We then test the ability of our Extended Linear Regression Model to predict the optimal training mixture $h^*$ that minimizes the loss on a target distribution, which is typically chosen as the uniform distribution over the $K$ domains, where $w=(\frac{1}{K}, \dots, \frac{1}{K})$. 

\paragraph{Procedure.} 

We evaluate our Extended Linear Regression Model and the baselines across four distinct experiments (one from Setting (1), two from Setting (2), and one from Setting (3)). For each experiment, we fit every applicable baseline in Table~\ref{tab:mixing_laws} and our theoretical models; we also include the uniform mixture as a non-fitted baseline. Once the laws are fitted, we compute the optimal mixture $h^*$ by solving the following optimization problem: 
\begin{equation}
    h^* = \arg\min_{h \in \Delta^{K-1}} \sum_{i=1}^K w_i \hat{L}_i(h),
\end{equation}
where $\hat{L}_i$ is the loss predicted by the fitted law. To solve this, we utilize standard solvers from \texttt{cvxpy} for convex baselines (e.g., the Exponential Law), and the L-BFGS algorithm with analytical gradients for non-convex formulations (including our Extended Linear Regression Model and the Additive Law from~\citep{shukor2025scaling}). Finally, we train new models using the derived optimal mixtures $h^*$ and evaluate their performance across the four distinct experimental settings.
\begin{figure}[ht]
    \centering
    \includegraphics[width=0.95\linewidth]{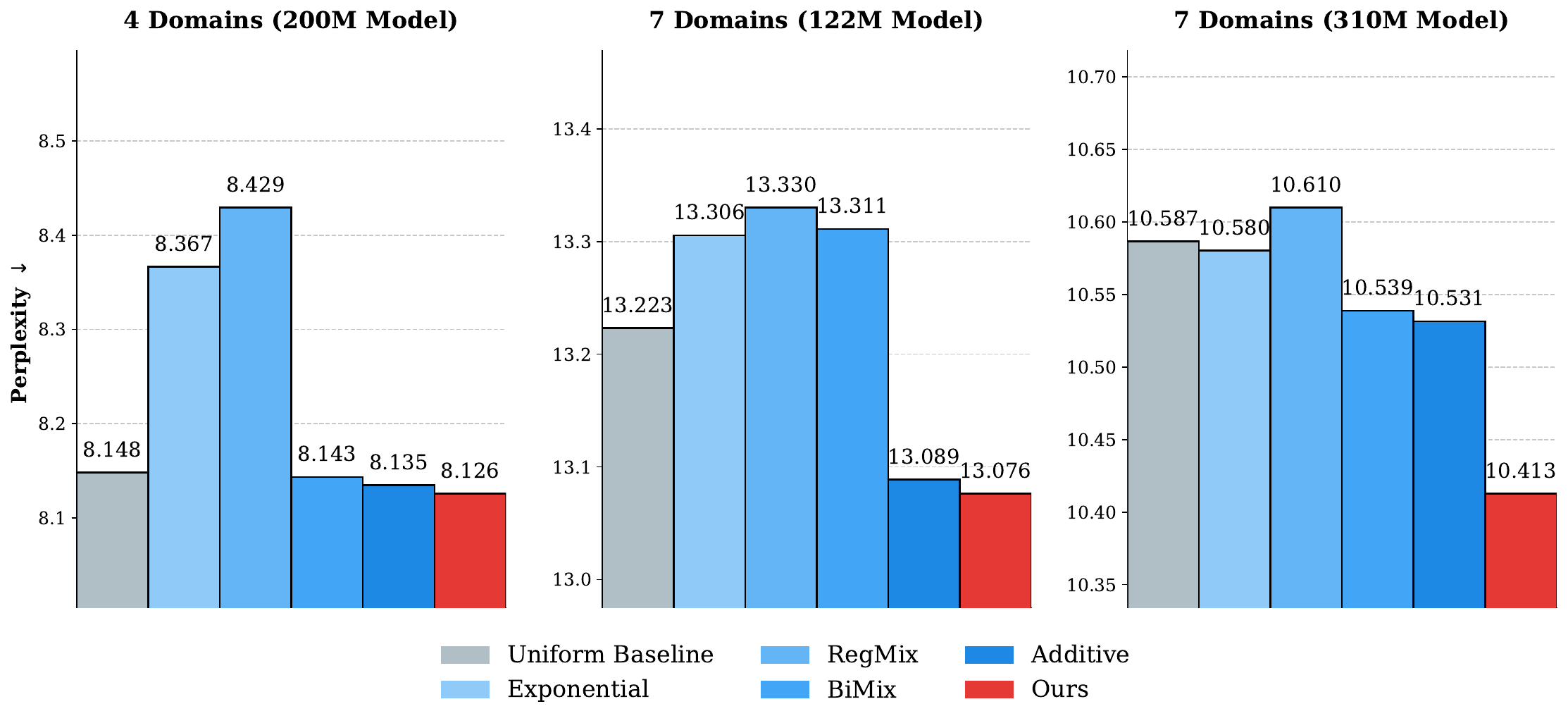}
    \caption{Performance comparison (test perplexity) of models trained with optimal mixtures predicted by different scaling laws. From left to right: 200M model on 4 domains, 122M model on 7 domains, and 310M model on 7 domains. Across all settings, our proposed method (highlighted in red) consistently achieves the lowest test perplexity compared to existing scaling laws, including the Exponential Law~\citep{ye2024data}, RegMix~\citep{liu2024regmix}, BiMix~\citep{ge2024bimix}, and the Additive Law~\citep{shukor2025scaling}.}
    \label{fig:loss_comparison}
\end{figure}
\paragraph{Results.} Figure~\ref{fig:loss_comparison} presents the final test perplexity evaluated on the target distribution for the 200M (left), 122M (middle), and 310M (right) models. For our fourth setting on the 17-domain Pile dataset, Figure~\ref{fig:mixture17} illustrates the test loss on the Pile-CC domain. %Here, we compare the optimal training mixture predicted by our fitted scaling law (from Table~\ref{tab:mre_results}) against the optimal mixture derived using the baseline law from~\citep{shukor2025scaling}, showing that our approach successfully achieves a lower loss. 
As observed, the mixture predicted by our Extended Linear Regression Model consistently yields the best held-out performance across all configurations. This demonstrates the superior accuracy and robustness of our approach in predicting optimal data mixtures compared to the uniform baseline and existing empirical laws. Crucially, we achieve this with significantly fewer free parameters.

\subsection{Extrapolation to Unseen Scales}\label{sec:extrapolation}

\paragraph{Procedure.} 
To evaluate our model's extrapolation capabilities, we compare the performance of models trained at larger, unseen target scales using three distinct data mixtures. We conduct this evaluation across two experimental settings adapted from \citep{shukor2025scaling}: a 4-domain configuration extrapolating from a small scale of 200M/8B (parameters/tokens) to a target scale of 700M/16B, and a 7-domain configuration extrapolating from 122M/10B to 1B/30B. The three evaluated mixtures are:

\begin{enumerate}
    \item \textbf{Our Static Mixture:} We first fit our Extended Linear Regression Model exclusively on the small-scale proxy regime. We then compute the optimal mixture for that same small scale by solving
    \begin{equation*}
    \begin{aligned}
    h^*_{\text{static}}
    &=\arg\min_{h\in\Delta^{K-1}}
    \sum_{i=1}^K w_i\widehat{L}_i(h\mid N,D).
    \end{aligned}
    \end{equation*}
    \item \textbf{Our Extrapolated Mixture:} Using the exact same parameters fitted exclusively on the small-scale proxy regime, we derive the optimal data mixture for the unseen target scale. We substitute $(N,D)$ with the target scales---$(700\text{M}, 16\text{B})$ for the 4-domain setting and $(1\text{B}, 30\text{B})$ for the 7-domain setting---and solve
    \begin{equation*}
    \begin{aligned}
    h^*_{\text{extrapolated}}
    &=\arg\min_{h\in\Delta^{K-1}}
    \sum_{i=1}^K w_i\widehat{L}_i(h\mid N,D).
    \end{aligned}
    \end{equation*}
    \item \textbf{Baseline Mixtures (Target Scale):} We compare our mixtures against the state-of-the-art Additive Law from \citep{shukor2025scaling}. For the 4-domain setting, because the optimal mixture predicted by the Additive Law is theoretically scale-invariant, we use the mixture predicted by fitting the 200M/8B test losses. For the 7-domain setting, we use the \textbf{full data optimal mixture} reported in Table 15 of \citep{shukor2025scaling}. This mixture acts as a strong baseline, as it was derived by fitting their Additive Law \textbf{across all available model and data sizes} ($N \in [412\text{M}, 1.4\text{B}]$, $D \in [4\text{B}, 46\text{B}]$)—crucially including the test losses from the target 1B/30B scale itself.
\end{enumerate}

\paragraph{Results.}
We train models from scratch at the target scales using the aforementioned mixtures and evaluate their final performance. 

For the 4-domain setting extrapolating to 700M/16B (Figure~\ref{fig:700average}), our extrapolated mixture achieves the lowest test loss. Notably, the extrapolated mixture specifically adjusted for the 700M/16B target scale strictly outperforms the optimal mixture derived for the smaller 200M/8B scale, confirming that our framework successfully predicts how the optimal data mixture shifts as model size increases.

This advantage extends to the 7-domain 1B/30B extrapolation task (Figure~\ref{fig:mixture_deviation}). Most importantly, our framework—fitted exclusively with 122M/10B proxy losses—achieves the same final test loss as the state-of-the-art empirical baseline. This result is highly significant because the baseline is exceptionally strong, having been fitted on a wide spectrum of model and data sizes that explicitly includes the target 1B/30B scale. Matching this benchmark demonstrates that our law accurately predicts mixture shifts and reaches state-of-the-art performance at an unseen large scale, relying entirely on small-scale proxy data.

\begin{figure}[H]
    \centering
    \includegraphics[width=0.65\textwidth]{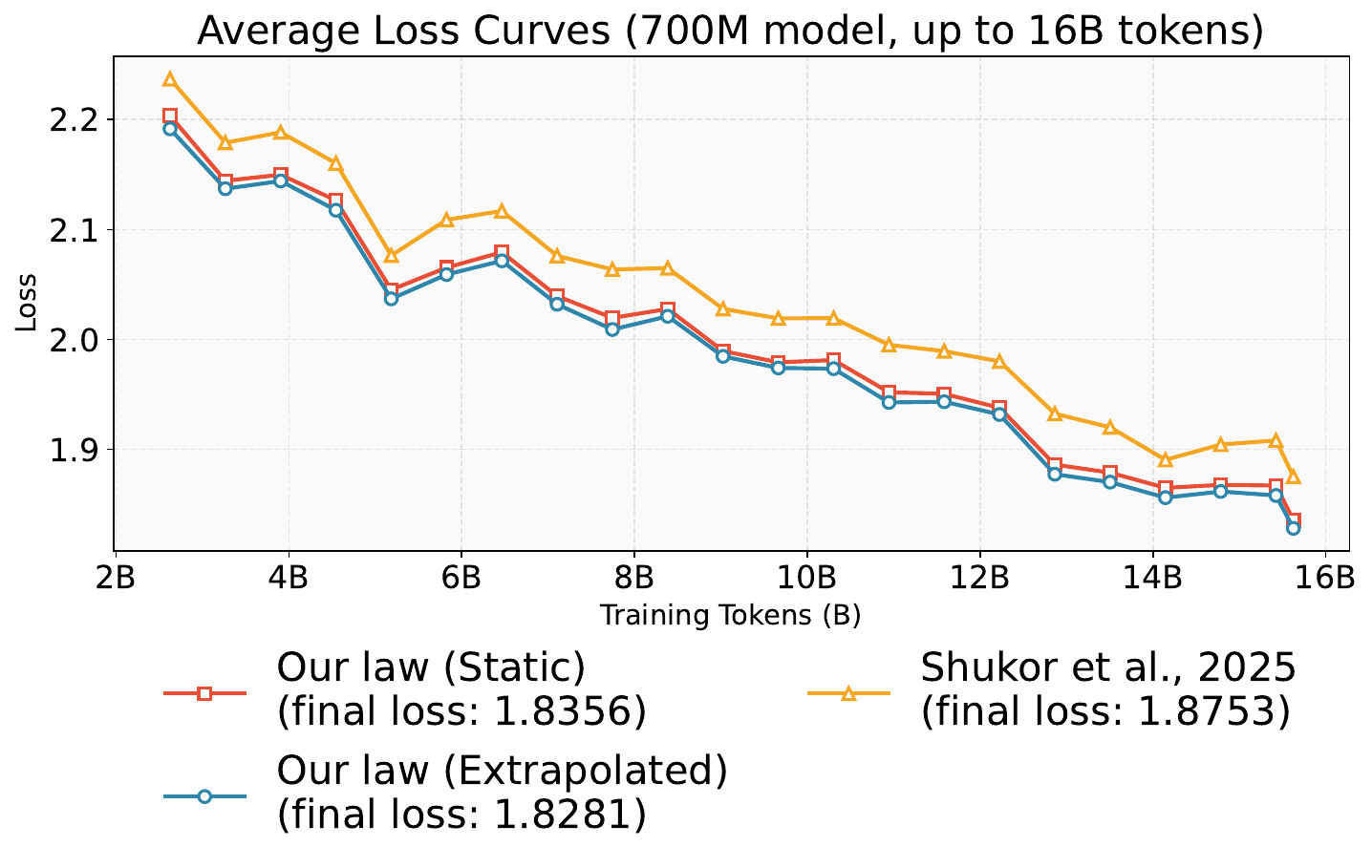} 
    \caption{\textbf{Test loss throughout training for the 4-domain 200M/8B to 700M/16B extrapolation setting.} Our extrapolated mixture for the target 700M/16B scale strictly outperforms both the static mixture derived for the 200M/8B scale and the mixture predicted by the Additive Law~\citep{shukor2025scaling}. This indicates that our framework correctly predicts the shift in the optimal data mixture as the scale increases.}
    \label{fig:700average}
\end{figure}

\begin{figure}[H]
    \centering
    \includegraphics[width=0.85\textwidth]{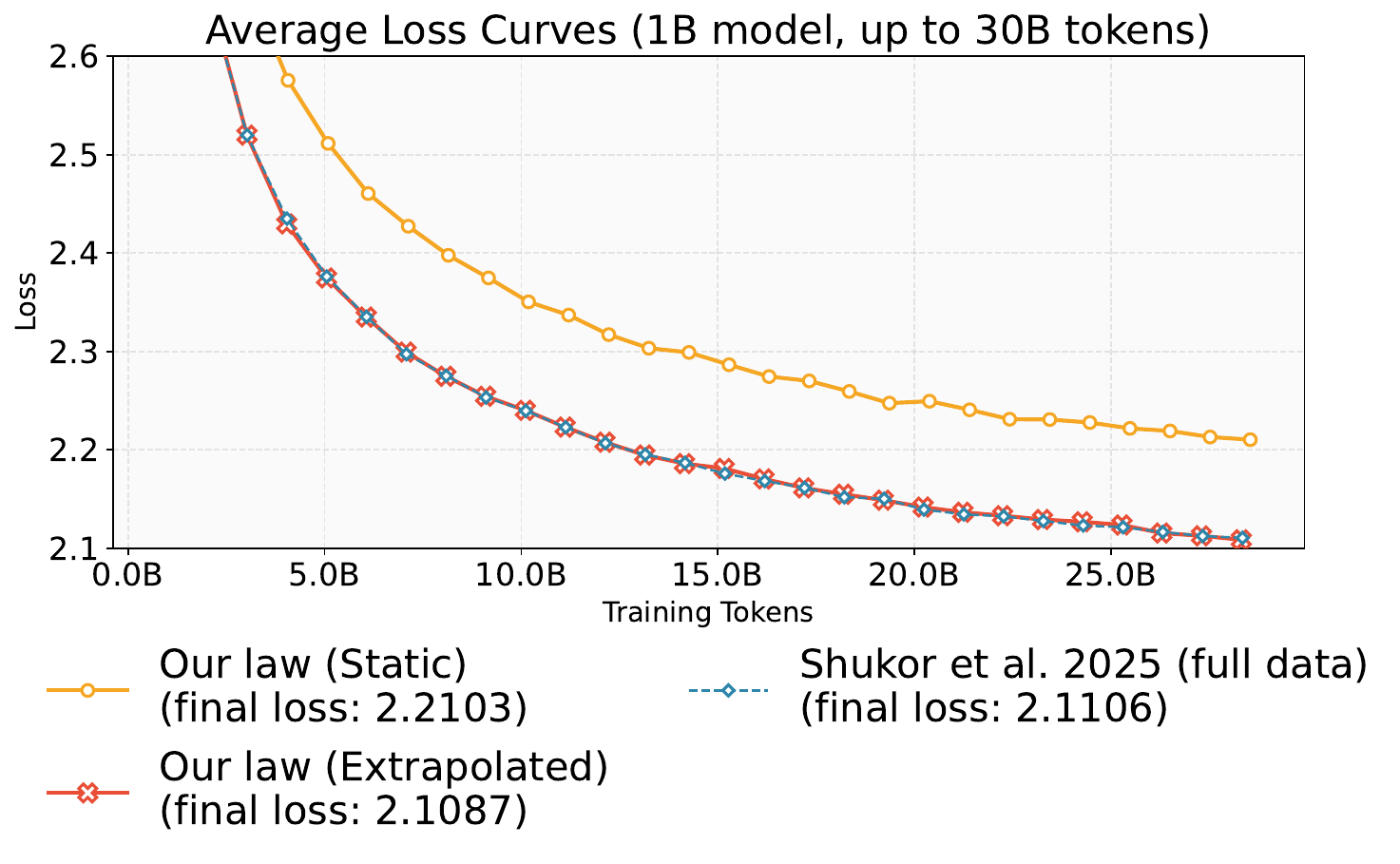}
    \caption{\textbf{Test loss throughout training for the 7-domain 122M/10B to 1B/30B extrapolation setting.} Our extrapolated mixture strictly outperforms the optimal static mixture derived for the 122M/10B scale, correctly capturing the scale-dependent shift in the optimal data mixture. Most importantly, our extrapolated law—fitted \textit{only} with 122M/10B losses—achieves the same test loss as the optimal mixture predicted by the state-of-the-art empirical law fitted with a massive range of model and data sizes (including the target 1B/30B scale itself). This demonstrates that our law not only accurately predicts mixture shifts, but also matches state-of-the-art performance at an unseen large scale using only small-scale proxy data.}
    \label{fig:mixture_deviation}
\end{figure}

\section{Conclusion and Future Work}
In this work, we proposed a theoretical model to explain data mixing scaling laws. Our framework accurately captures the loss landscape under various mixtures with low MRE and effectively identifies optimal mixtures. However, several promising directions remain for future exploration:

    \textbf{Unseen Domains and Downstream Tasks.} Our current framework primarily addresses loss prediction within the span of training domains. %However, practitioners are often concerned with generalization to unseen domains or performance on downstream tasks. 
    A valuable future direction is to extend this framework to predict the effects of different data mixtures on unseen domains or downstream tasks. 
    
    \textbf{Explicit Domain Overlap.} While the disjoint tail assumption appears effective in predicting the loss landscape, incorporating explicit modeling of information overlap may further refine predictions of the optimal mixture, particularly for highly correlated domains.
    
    \textbf{Reliable Fitting Algorithms.} Fitting our model parameters currently requires solving a non-convex optimization problem, which can be computationally intensive and sensitive to initialization. A future direction is to develop more reliable estimation techniques, such as convex relaxations or analytical approximations.

\section*{Acknowledgement}
We gratefully acknowledge Kaifeng Lyu and Xinran Gu for valuable discussions, Pierre Ablin and Xiaosen Zheng for their clarification of experimental details and data, and the ICML reviewers for their insightful feedback.

%\section*{Impact Statement}
%This work is primarily theoretical in nature. We do not foresee any specific ethical concerns or immediate negative societal consequences arising from this study.


\newpage

\bibliography{arxiv_main}

@misc{korthikanti2022reducingactivationrecomputationlarge,
      title={Reducing Activation Recomputation in Large Transformer Models}, 
      author={Vijay Korthikanti and Jared Casper and Sangkug Lym and Lawrence McAfee and Michael Andersch and Mohammad Shoeybi and Bryan Catanzaro},
      year={2022},
      eprint={2205.05198},
      archivePrefix={arXiv},
      primaryClass={cs.LG},
      url={https://arxiv.org/abs/2205.05198}, 
}

@article{brown2020language,
  title={Language models are few-shot learners},
  author={Brown, Tom and Mann, Benjamin and Ryder, Nick and Subbiah, Melanie and Kaplan, Jared D and Dhariwal, Prafulla and Neelakantan, Arvind and Shyam, Pranav and Sastry, Girish and Askell, Amanda and others},
  journal={Advances in neural information processing systems},
  volume={33},
  pages={1877--1901},
  year={2020}
}

@article{gao2020pile,
  title={The pile: An 800gb dataset of diverse text for language modeling},
  author={Gao, Leo and Biderman, Stella and Black, Sid and Golding, Laurence and Hoppe, Travis and Foster, Charles and Phang, Jason and He, Horace and Thite, Anish and Nabeshima, Noa and others},
  journal={arXiv preprint arXiv:2101.00027},
  year={2020}
}

@article{zhang2024tinyllama,
  title={Tinyllama: An open-source small language model},
  author={Zhang, Peiyuan and Zeng, Guangtao and Wang, Tianduo and Lu, Wei},
  journal={arXiv preprint arXiv:2401.02385},
  year={2024}
}

@inproceedings{
pan2025understanding,
title={Understanding {LLM} Behaviors via Compression: Data Generation, Knowledge Acquisition and Scaling Laws},
author={Zhixuan Pan and Shaowen Wang and Liao Pengfei and Jian Li},
booktitle={The Thirty-ninth Annual Conference on Neural Information Processing Systems},
year={2026},
url={https://openreview.net/forum?id=853SwC2dMZ}
}

@article{arora2023theory,
  title={A theory for emergence of complex skills in language models},
  author={Arora, Sanjeev and Goyal, Anirudh},
  journal={arXiv preprint arXiv:2307.15936},
  year={2023}
}

@article{liu2025physics,
  title={Physics of skill learning},
  author={Liu, Ziming and Liu, Yizhou and Michaud, Eric J and Gore, Jeff and Tegmark, Max},
  journal={arXiv preprint arXiv:2501.12391},
  year={2025}
}

@inproceedings{
magnusson2025datadecide,
title={DataDecide: How to Predict Best Pretraining Data with Small Experiments},
author={Ian Magnusson and Nguyen Tai and Ben Bogin and David Heineman and Jena D. Hwang and Luca Soldaini and Akshita Bhagia and Jiacheng Liu and Dirk Groeneveld and Oyvind Tafjord and Noah A. Smith and Pang Wei Koh and Jesse Dodge},
booktitle={Forty-second International Conference on Machine Learning},
year={2025},
url={https://openreview.net/forum?id=p9YlQPF8fE}
}

@inproceedings{
zhang2025domain2vec,
title={Domain2Vec: Vectorizing Datasets to Find the Optimal Data Mixture without Training},
author={Mozhi Zhang and Howe Tissue and Lu Wang and Xipeng Qiu},
booktitle={Forty-second International Conference on Machine Learning},
year={2025},
url={https://openreview.net/forum?id=kJ5i29FejW}
}

@inproceedings{
diao2025climb,
title={Nemotron-{CLIMB}: Clustering-based Iterative Data Mixture Bootstrapping for Language Model Pre-training},
author={Shizhe Diao and Yu Yang and Yonggan Fu and Xin Dong and Dan SU and Markus Kliegl and ZIJIA CHEN and Peter Belcak and Yoshi Suhara and Hongxu Yin and Mostofa Patwary and Yingyan Celine Lin and Jan Kautz and Pavlo Molchanov},
booktitle={The Thirty-ninth Annual Conference on Neural Information Processing Systems Datasets and Benchmarks Track},
year={2026},
url={https://openreview.net/forum?id=aBlqKPkc4a}
}

@inproceedings{
liu2024regmix,
title={RegMix: Data Mixture as Regression for Language Model Pre-training},
author={Qian Liu and Xiaosen Zheng and Niklas Muennighoff and Guangtao Zeng and Longxu Dou and Tianyu Pang and Jing Jiang and Min Lin},
booktitle={The Thirteenth International Conference on Learning Representations},
year={2025},
url={https://openreview.net/forum?id=5BjQOUXq7i}
}

@inproceedings{
fan2023doge,
title={{DOGE}: Domain Reweighting with Generalization Estimation},
author={Simin Fan and Matteo Pagliardini and Martin Jaggi},
booktitle={Second Agent Learning in Open-Endedness Workshop},
year={2023},
url={https://openreview.net/forum?id=qiKqsqwYXm}
}

@article{xie2023doremi,
  title={Doremi: Optimizing data mixtures speeds up language model pretraining},
  author={Xie, Sang Michael and Pham, Hieu and Dong, Xuanyi and Du, Nan and Liu, Hanxiao and Lu, Yifeng and Liang, Percy S and Le, Quoc V and Ma, Tengyu and Yu, Adams Wei},
  journal={Advances in Neural Information Processing Systems},
  volume={36},
  pages={69798--69818},
  year={2023}
}

@inproceedings{
shukor2025scaling,
title={Scaling Laws for Optimal Data Mixtures},
author={Mustafa Shukor and Louis B{\'e}thune and Dan Busbridge and David Grangier and Enrico Fini and Alaaeldin El-Nouby and Pierre Ablin},
booktitle={The Thirty-ninth Annual Conference on Neural Information Processing Systems},
year={2026},
url={https://openreview.net/forum?id=vVU1KTOsju}
}

@inproceedings{
kang2024autoscale,
title={AutoScale: Scale-Aware Data Mixing for Pre-Training {LLM}s},
author={Feiyang Kang and Yifan Sun and Bingbing Wen and Si Chen and Dawn Song and Rafid Mahmood and Ruoxi Jia},
booktitle={Second Conference on Language Modeling},
year={2025},
url={https://openreview.net/forum?id=rujwIvjooA}
}

@misc{
ge2024bimix,
title={BiMix: Bivariate Data Mixing Law for Language Model Pretraining},
author={Ce Ge and Zhijian Ma and Daoyuan Chen and Yaliang Li and Bolin Ding},
year={2025},
url={https://openreview.net/forum?id=JsM46OZix7}
}

@inproceedings{
ye2024data,
title={Data Mixing Laws: Optimizing Data Mixtures by Predicting Language Modeling Performance},
author={Jiasheng Ye and Peiju Liu and Tianxiang Sun and Jun Zhan and Yunhua Zhou and Xipeng Qiu},
booktitle={The Thirteenth International Conference on Learning Representations},
year={2025},
url={https://openreview.net/forum?id=jjCB27TMK3}
}

@article{hoffmann2022training,
  title={Training compute-optimal large language models},
  author={Hoffmann, Jordan and Borgeaud, Sebastian and Mensch, Arthur and Buchatskaya, Elena and Cai, Trevor and Rutherford, Eliza and Casas, Diego de Las and Hendricks, Lisa Anne and Welbl, Johannes and Clark, Aidan and others},
  journal={arXiv preprint arXiv:2203.15556},
  year={2022}
}

@article{kaplan2020scaling,
  title={Scaling laws for neural language models},
  author={Kaplan, Jared and McCandlish, Sam and Henighan, Tom and Brown, Tom B and Chess, Benjamin and Child, Rewon and Gray, Scott and Radford, Alec and Wu, Jeffrey and Amodei, Dario},
  journal={arXiv preprint arXiv:2001.08361},
  year={2020}
}

@article{fonseca2024exactly,
  title={An exactly solvable model for emergence and scaling laws in the multitask sparse parity problem},
  author={Fonseca, Nayara and Lee, Seok Hyeong and Mingard, Chris and Louis, Ard and others},
  journal={Advances in Neural Information Processing Systems},
  volume={37},
  pages={39632--39693},
  year={2024}
}

@article{michaud2023quantization,
  title={The quantization model of neural scaling},
  author={Michaud, Eric and Liu, Ziming and Girit, Uzay and Tegmark, Max},
  journal={Advances in Neural Information Processing Systems},
  volume={36},
  pages={28699--28722},
  year={2023}
}

@article{atanasov2024scaling,
  title={Scaling and renormalization in high-dimensional regression},
  author={Atanasov, Alexander and Zavatone-Veth, Jacob A and Pehlevan, Cengiz},
  journal={arXiv preprint arXiv:2405.00592},
  year={2024}
}

@inproceedings{
li2025functional,
title={Functional Scaling Laws in Kernel Regression: Loss Dynamics and Learning Rate Schedules},
author={Binghui Li and Fengling Chen and Zixun Huang and Lean Wang and Lei Wu},
booktitle={The Thirty-ninth Annual Conference on Neural Information Processing Systems},
year={2026},
url={https://openreview.net/forum?id=dpllevHMbc}
}

@inproceedings{
bordelon2025feature,
title={How Feature Learning Can Improve Neural Scaling Laws},
author={Blake Bordelon and Alexander Atanasov and Cengiz Pehlevan},
booktitle={The Thirteenth International Conference on Learning Representations},
year={2025},
url={https://openreview.net/forum?id=dEypApI1MZ}
}

@article{sharma2020neural,
  title={Scaling Laws from the Data Manifold Dimension},
  author={Utkarsh Sharma},
  journal={J. Mach. Learn. Res.},
  year={2022},
  volume={23},
  pages={9:1-9:34},
  url={https://api.semanticscholar.org/CorpusID:246559072}
}

@article{maloney2022solvable,
  title={A solvable model of neural scaling laws},
  author={Maloney, Alexander and Roberts, Daniel A and Sully, James},
  journal={arXiv preprint arXiv:2210.16859},
  year={2022}
}

@inproceedings{bordelon2024dynamical,
  title={A Dynamical Model of Neural Scaling Laws},
  author={Bordelon, Blake and Atanasov, Alexander and Pehlevan, Cengiz},
  booktitle={International Conference on Machine Learning},
  pages={4345--4382},
  year={2024},
  organization={PMLR}
}

@inproceedings{
wettig2025organize,
title={Organize the Web: Constructing Domains Enhances Pre-Training Data Curation},
author={Alexander Wettig and Kyle Lo and Sewon Min and Hannaneh Hajishirzi and Danqi Chen and Luca Soldaini},
booktitle={Forty-second International Conference on Machine Learning},
year={2025},
url={https://openreview.net/forum?id=boSqwdvJVC}
}

@inproceedings{
lin2024rho,
title={Not All Tokens Are What You Need for Pretraining},
author={Zhenghao Lin and Zhibin Gou and Yeyun Gong and Xiao Liu and yelong shen and Ruochen Xu and Chen Lin and Yujiu Yang and Jian Jiao and Nan Duan and Weizhu Chen},
booktitle={The Thirty-eighth Annual Conference on Neural Information Processing Systems},
year={2024},
url={https://openreview.net/forum?id=0NMzBwqaAJ}
}

@inproceedings{
li2025pike,
title={Pi{KE}: Adaptive Data Mixing for Large-Scale Multi-Task Learning Under Low Gradient Conflicts},
author={Zeman Li and Yuan Deng and Peilin Zhong and Meisam Razaviyayn and Vahab Mirrokni},
booktitle={The Thirty-ninth Annual Conference on Neural Information Processing Systems},
year={2026},
url={https://openreview.net/forum?id=xNJenVNmzL}
}

@inproceedings{
jiang2024adaptive,
title={Adaptive Data Optimization: Dynamic Sample Selection with Scaling Laws},
author={Yiding Jiang and Allan Zhou and Zhili Feng and Sadhika Malladi and J Zico Kolter},
booktitle={The Thirteenth International Conference on Learning Representations},
year={2025},
url={https://openreview.net/forum?id=aqok1UX7Z1}
}

@inproceedings{
chen2024aioli,
title={Aioli: A Unified Optimization Framework for Language Model Data Mixing},
author={Mayee F Chen and Michael Y. Hu and Nicholas Lourie and Kyunghyun Cho and Christopher Re},
booktitle={The Thirteenth International Conference on Learning Representations},
year={2025},
url={https://openreview.net/forum?id=sZGZJhaNSe}
}

@inproceedings{
chen2023skillitdatadrivenskillsframework,
title={Skill-it! A data-driven skills framework for understanding and training language models},
author={Mayee F Chen and Nicholas Roberts and Kush Bhatia and Jue WANG and Ce Zhang and Frederic Sala and Christopher Re},
booktitle={Thirty-seventh Conference on Neural Information Processing Systems},
year={2023},
url={https://openreview.net/forum?id=IoizwO1NLf}
}

@inproceedings{
loshchilov2017sgdrstochasticgradientdescent,
title={{SGDR}: Stochastic Gradient Descent with Warm Restarts},
author={Ilya Loshchilov and Frank Hutter},
booktitle={International Conference on Learning Representations},
year={2017},
url={https://openreview.net/forum?id=Skq89Scxx}
}

@inproceedings{
qiu2025scalingcollapserevealsuniversal,
title={Scaling Collapse Reveals Universal Dynamics in Compute-Optimally Trained Neural Networks},
author={Shikai Qiu and Lechao Xiao and Andrew Gordon Wilson and Jeffrey Pennington and Atish Agarwala},
booktitle={Forty-second International Conference on Machine Learning},
year={2025},
url={https://openreview.net/forum?id=Fvq9ogLnLN}
}

@article{albalak2023efficient,
  title={Efficient online data mixing for language model pre-training},
  author={Albalak, Alon and Pan, Liangming and Raffel, Colin and Wang, William Yang},
  journal={arXiv preprint arXiv:2312.02406},
  year={2023}
}

@inproceedings{
gu2025data,
title={Data Mixing Can Induce Phase Transitions in Knowledge Acquisition},
author={Xinran Gu and Kaifeng Lyu and Jiazheng Li and Jingzhao Zhang},
booktitle={ICLR 2025 Workshop on Navigating and Addressing Data Problems for Foundation Models},
year={2025},
url={https://openreview.net/forum?id=ZKA4yiGdrA}
}

@inproceedings{Muennighoff2023scaling,
author = {Muennighoff, Niklas and Rush, Alexander M. and Barak, Boaz and Le Scao, Teven and Piktus, Aleksandra and Tazi, Nouamane and Pyysalo, Sampo and Wolf, Thomas and Raffel, Colin},
title = {Scaling data-constrained language models},
year = {2023},
publisher = {Curran Associates Inc.},
address = {Red Hook, NY, USA},
booktitle = {Proceedings of the 37th International Conference on Neural Information Processing Systems},
articleno = {2191},
numpages = {19},
location = {New Orleans, LA, USA},
series = {NIPS '23}
}

@inproceedings{lin2024,
author = {Lin, Licong and Wu, Jingfeng and Kakade, Sham M. and Bartlett, Peter L. and Lee, Jason D.},
title = {Scaling laws in linear regression: compute, parameters, and data},
year = {2024},
isbn = {9798331314385},
publisher = {Curran Associates Inc.},
address = {Red Hook, NY, USA},
booktitle = {Proceedings of the 38th International Conference on Neural Information Processing Systems},
articleno = {1937},
numpages = {51},
location = {Vancouver, BC, Canada},
series = {NIPS 24}
}

@article{field1987,
	Author = {David J. Field},
	Doi = {10.1364/JOSAA.4.002379},
	Journal = {J. Opt. Soc. Am. A},
	Month = {Dec},
	Number = {12},
	Pages = {2379--2394},
	Publisher = {Optica Publishing Group},
	Title = {Relations between the statistics of natural images and the response properties of cortical cells},
	Url = {https://opg.optica.org/josaa/abstract.cfm?URI=josaa-4-12-2379},
	Volume = {4},
	Year = {1987}}

@article{bahri2021,
	Author = {Yasaman Bahri and Ethan Dyer and Jared Kaplan and Jaehoon Lee and Utkarsh Sharma},
	Doi = {10.1073/pnas.2311878121},
	Eprint = {https://www.pnas.org/doi/pdf/10.1073/pnas.2311878121},
	Journal = {Proceedings of the National Academy of Sciences},
	Number = {27},
	Pages = {e2311878121},
	Title = {Explaining neural scaling laws},
	Url = {https://www.pnas.org/doi/abs/10.1073/pnas.2311878121},
	Volume = {121},
	Year = {2024}}

@inproceedings{
hamidieh2025domainaware,
title={Domain-Aware Scaling Laws Uncover Data Synergy},
author={Kimia Hamidieh and Lester Mackey and David Alvarez-Melis},
booktitle={NeurIPS 2025 Workshop on Evaluating the Evolving LLM Lifecycle: Benchmarks, Emergent Abilities, and Scaling},
year={2025},
url={https://openreview.net/forum?id=FndNAs9s0d}
}

@misc{medvedev2025shiftgoodmismatcheddata,
      title={Shift is Good: Mismatched Data Mixing Improves Test Performance}, 
      author={Marko Medvedev and Kaifeng Lyu and Zhiyuan Li and Nathan Srebro},
      year={2025},
      eprint={2510.25108},
      archivePrefix={arXiv},
      primaryClass={cs.LG},
      url={https://arxiv.org/abs/2510.25108}, 
}

@misc{paquette2024phases,
      title={{4+3} Phases of Compute-Optimal Neural Scaling Laws},
      author={Elliot Paquette and Courtney Paquette and Lechao Xiao and Jeffrey Pennington},
      year={2024},
      eprint={2405.15074},
      archivePrefix={arXiv},
      primaryClass={stat.ML},
      doi={10.48550/arXiv.2405.15074},
      url={https://arxiv.org/abs/2405.15074}
}
\bibliographystyle{icml2026}

%%%%%%%%%%%%%%%%%%%%%%%%%%%%%%%%%%%%%%%%%%%%%%%%%%%%%%%%%%%%%%%%%%%%%%%%%%%%%%%
%%%%%%%%%%%%%%%%%%%%%%%%%%%%%%%%%%%%%%%%%%%%%%%%%%%%%%%%%%%%%%%%%%%%%%%%%%%%%%%
% APPENDIX
%%%%%%%%%%%%%%%%%%%%%%%%%%%%%%%%%%%%%%%%%%%%%%%%%%%%%%%%%%%%%%%%%%%%%%%%%%%%%%%
%%%%%%%%%%%%%%%%%%%%%%%%%%%%%%%%%%%%%%%%%%%%%%%%%%%%%%%%%%%%%%%%%%%%%%%%%%%%%%%
\newpage
\appendix
\onecolumn
\section{Empirical Data Mixing Laws}
As detailed in Table~\ref{tab:mixing_laws}, these baselines offer distinct functional forms for predicting the domain loss $L_i(h)$ given the data mixture weights $h$. Specifically, we consider: (1) the \textbf{Additive Law}~\citep{shukor2025scaling}, which models the loss using an inverse polynomial combination of mixture weights alongside model and data scale parameters ($N$ and $D$); (2) the \textbf{Exponential Law}~\citep{ye2024data}, which assumes the loss decays exponentially based on a linear combination of the mixture weights; (3) \textbf{BiMix}~\citep{ge2024bimix}, which factors the scaling effects of the target domain's weight against the total training tokens; and (4) \textbf{RegMix}~\citep{liu2024regmix}, which employs a simple yet effective linear regression directly over the mixture proportions.

\begin{table}[!h]
    \centering
    \caption{Comparison of empirical functional forms for predicting domain loss $L_i(h)$ based on mixture weights $h$. $N$ and $D$ represent model parameters and training tokens, respectively.}
    \label{tab:mixing_laws}
    \renewcommand{\arraystretch}{1.8} 
    \begin{tabular}{p{0.3\textwidth} p{0.6\textwidth}}
        \toprule
        \textbf{Law Name} & \textbf{Functional Form ($f_i(h, N, D)$)} \\
        \midrule
        \textbf{Additive} \newline \citep{shukor2025scaling} & 
        $L_i \approx E_i + \left(\sum_{j=1}^K C_{ij} h_j^{\gamma_{ij}}\right)^{-1} + \frac{A}{N^\alpha} + \frac{B}{D^\beta}$ \\
        \midrule
        \textbf{Exponential} \newline \citep{ye2024data} & 
        $L_i \approx c_i + k_i \exp\left(\sum_{j=1}^K t_{ij} h_j\right)$ \\
        \midrule
        \textbf{BiMix} \newline \citep{ge2024bimix} & 
        $L_i \approx (\frac{B}{D^\beta} + E)\frac{C}{h_i^\gamma}$ \\
        \midrule
        \textbf{RegMix} \newline \citep{liu2024regmix} & $L_i \approx w_0+\sum_{j=1}^K w_j h_j$ \\

        \bottomrule
    \end{tabular}
\end{table}

\section{Algorithm} \label{app:algorithm}
We propose Algorithm~\ref{alg:omd} with exponentiated-gradient updates to find the optimal mixture $h^*$.

\begin{algorithm}[!h]
   \caption{Bi-Level Mixture Optimization via OMD}
   \label{alg:omd}
\begin{algorithmic}[1]
   \STATE \textbf{Input:} Problem parameters, target weights $w$, constraints $N,D$, and step size $\eta$.
   \STATE \textbf{Initialize:} $h^{(0)} \leftarrow [1/K, \dots, 1/K]$, $t \leftarrow 0$.
   
   \WHILE{not converged}
      \STATE \textit{// 1. Inner Level: Capacity Response}
      \STATE Solve the inner optimization (Eq.~\ref{eq:quant_opt}) given $h^{(t)}$ to obtain optimal allocation $x^*$ and multiplier $\lambda$.
      
      \STATE \textit{// 2. Outer Level: Gradient Calculation}
      \STATE Compute the gradient vector $\nabla \mathcal{J}(h^{(t)})$ according to the closed-form solution in \textbf{Proposition~\ref{thm:bilevel_gradient}}.
      
      \STATE \textit{// 3. Optimization: Mirror Descent Step}
      \STATE Update weights (multiplicative): $\tilde{h} \leftarrow h^{(t)} \odot \exp\left(-\eta \nabla \mathcal{J}(h^{(t)})\right)$.
      \STATE Normalize: $h^{(t+1)} \leftarrow \tilde{h} / \|\tilde{h}\|_1$.
      \STATE $t \leftarrow t + 1$.
   \ENDWHILE
   
   \STATE \textbf{Return:} Optimal mixture $h^* \leftarrow h^{(t)}$.
\end{algorithmic}
\end{algorithm}

\section{Implementation Details}

For the experiments in Section~\ref{exp:4dom}, we implement our pretraining pipeline using the Megatron-LM framework in~\cite{korthikanti2022reducingactivationrecomputationlarge} on a cluster of 8 NVIDIA H200 GPUs. To investigate the scaling laws with respect to data mixture ratios, we construct two GPT-style transformer models at different scales: a 200M-parameter model and a 700M-parameter model. 

\paragraph{Model Architecture.}
Both models follow the standard decoder-only transformer architecture. 
The 200M model consists of 24 layers, a hidden size of 768, and 12 attention heads. 
The 700M model scales the hidden size to 1536 while maintaining 24 layers and 12 heads. 
We use a sequence length of 1024 and employ FlashAttention to improve training efficiency.

\paragraph{Training Setup.}
We train our models using the AdamW optimizer with $\beta_1=0.9$, $\beta_2=0.95$, and a weight decay of $0.1$. 
We use a cosine learning-rate schedule with a peak learning rate of $3.5 \times 10^{-4}$, a minimum learning rate of $3.5 \times 10^{-5}$, and a 1\% warmup phase. 
The global batch size is fixed at 128, with a micro-batch size of 16 and gradient accumulation. 
Training is performed in BF16 precision. We train the 200M model on 8B tokens and the 700M model on 16B tokens.

\paragraph{Data Configuration.}
Our training corpus is composed of four domains: GitHub, PG-19 (Project Gutenberg), StackExchange, and Wikipedia. 
Since our primary focus is on data mixing laws, the mixture weights (ratios) of these domains vary across different experimental runs. 
Detailed hyperparameters for the model architecture and optimization are summarized in Table~\ref{tab:hyperparams}.

\begin{table}[!h]
    \centering
    \caption{Hyperparameters for the 200M and 700M models used in our data mixing experiments.}
    \label{tab:hyperparams}
    \begin{tabular}{l|cc}
    \toprule
    \textbf{Hyperparameter} & \textbf{200M Model} & \textbf{700M Model} \\
    \midrule
    \textit{Architecture} & & \\
    \quad Layers ($L$) & 24 & 24 \\
    \quad Hidden Size ($d_{\text{model}}$) & 768 & 1536 \\
    \quad Attention Heads & 12 & 12 \\
    \quad Sequence Length & 1024 & 1024 \\
    \midrule
    \textit{Optimization} & & \\
    \quad Global Batch Size & 128 & 128 \\
    \quad Micro-Batch Size & 16 & 16 \\
    \quad Learning Rate & $3.5 \times 10^{-4}$ & $3.5 \times 10^{-4}$ \\
    \quad Minimum Learning Rate & $3.5 \times 10^{-5}$ & $3.5 \times 10^{-5}$ \\
    \quad LR Schedule & Cosine & Cosine \\
    \quad Warmup Ratio & 0.01 & 0.01 \\
    \quad Optimizer & AdamW & AdamW \\
    \quad Weight Decay & 0.1 & 0.1 \\
    \quad Precision & BF16 & BF16 \\
    \quad Training Tokens & 8B & 16B \\
    \bottomrule
    \end{tabular}
\end{table}

\begin{figure}[ht]
    \centering
    \includegraphics[width=\linewidth]{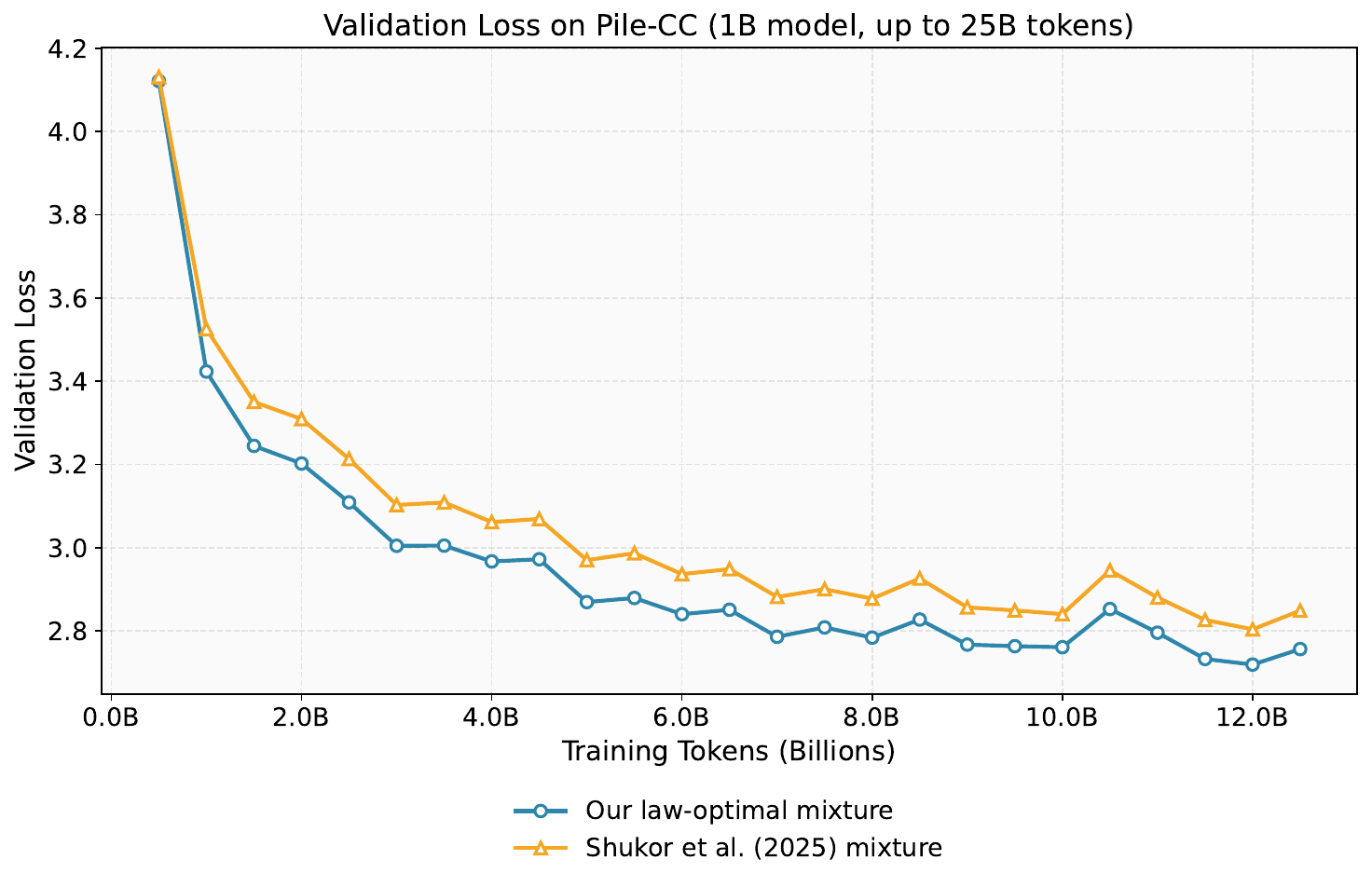}
    \caption{Test loss on Pile-CC\@. The model trained with the optimal mixture predicted by our fitted scaling law achieves a lower test loss than the mixture derived using the baseline law from~\citet{shukor2025scaling}.}
    \label{fig:mixture17}
\end{figure}

\section{Sketched Linear Regression}

Following~\cite{lin2024,bordelon2024dynamical,li2025functional}, we consider a supervised learning setting where the goal is to learn a linear relationship from a stream of data. We analyze the dynamics of stochastic gradient descent (SGD) under a \textit{sketched observation} model, which characterizes the scenario where model capacity (or feature access) is limited relative to the complexity of the data-generating process. Our theoretical framework builds upon the work of \citet{li2025functional}, who demonstrated a scaling law with respect to the learning-rate schedule (LRS). In this work, we extend their analysis to derive the scaling law with respect to the mixture ratio $h$. Crucially, we no longer assume that the entire second-order moment of $\mathbf{x}$ follows a single power law; instead, the moment is partitioned into several blocks, with each block exhibiting distinct power-law scaling. In this section, we compute the test loss on each domain given a data mixture $h \in \Delta^{K-1}$. We investigate the linear model proposed in Section~\ref{sec:elrm}. Following the notation in Section~\ref{sec:elrm}, we use $\lambda^{(i)}_k$ to represent the $k$-th eigenvalue in domain $i$.

\subsection{Single-domain Framework} \label{app:LR_single}

In the single-domain setting, \citet{bordelon2024dynamical,lin2024} provide a rigorous derivation of scaling laws by analyzing the training dynamics of linear regression under one-pass SGD\@. In this framework, neural scaling laws are governed by the spectral decay of the data. As training progresses, the model ``resolves'' eigenmodes in descending order of their eigenvalues---learning the dominant patterns before fitting the fine-grained details.
Below, we adopt the formal framework from \citet{lin2024} to provide a simplified theoretical explanation.

\begin{itemize}

    \item  \textbf{Data Generation:} Consider a linear regression problem where the input covariates $\mathbf{x} \in \mathbb{R}^d$ (where $d$ can be infinite) are drawn from a distribution with zero mean and covariance matrix $\mathbf{H} = \mathbb{E}[\mathbf{x}\mathbf{x}^\top]$. The target label $y$ is generated by a linear teacher with additive noise:
    \begin{equation*}
        y = \langle \theta_*, \mathbf{x} \rangle + \epsilon,
    \end{equation*}
    where $\theta_*$ is the ground-truth parameter and $\epsilon \sim \mathcal{N}(0, \sigma^2)$ is independent Gaussian noise.

    \item \textbf{Spectral Assumptions:} When training a linear regression model with one-pass SGD, the learning dynamics are determined by the spectrum of the covariance matrix $\mathbf{H}  = \mathbb{E}[\mathbf{x}\mathbf{x}^\top]$. Intuitively, the eigenvalues $\lambda_k$ of $\mathbf{H}$ represent the variance (or signal strength) of the data along the $k$-th principal component.
    Empirical studies on natural data (e.g., images and text) consistently observe that these eigenvalues follow a power distribution \citep{field1987, bahri2021}. It is therefore assumed that the eigenvalues, sorted in descending order, follow a power law:
    \begin{equation*}
        \lambda_k \propto k^{-\alpha}, \quad \text{for } \alpha > 1.
    \end{equation*}

    \item \textbf{Finite Parameter Projection:}  To model a neural network with a finite capacity of $N$ parameters, we project the high-dimensional input $\mathbf{x}$ into a lower-dimensional feature space using a ``sketching matrix'' $\mathbf{S} \in \mathbb{R}^{N \times d}$. The model learns a weight vector $\theta \in \mathbb{R}^N$ by minimizing the squared error on the projected features $\tilde{\mathbf{x}} = \mathbf{S}\mathbf{x}$:
    \begin{equation*}
        \widehat\theta = \operatorname*{arg\,min}_{\theta \in \mathbb{R}^N} \frac{1}{D} \sum_{i=1}^D (y_i - \theta^\top \mathbf{S} \mathbf{x}_i)^2.
    \end{equation*}

    \item \textbf{Result:} Under this setup, \cite{lin2024} derive the scaling law for the test loss $L(N,D)$ of the projected linear model $\widehat{\theta}$ (trained via one-pass SGD) as a function of the model size $N$ and number of training samples $D$:
    \begin{equation*}
        L(N, D) \approx \underbrace{O\left(\frac{1}{N^{a_1}}\right)}_{\text{Model Scaling}} + \underbrace{O\left(\frac{1}{D^{a_2}}\right)}_{\text{Data Scaling}} + E.
    \end{equation*}
\end{itemize}
\subsection{Multi-domain Framework}\label{app:theory_des}

In this section, we formally describe the projected linear regression model in the multi-domain setting. 

Following~\Cref{sec:pre_linear}, we consider a linear regression problem over a union of $K$ domains. For each domain $i \in \{1, \dots, K\}$, input covariates $\mathbf{x}_i \in \mathbb{H}$ are feature vectors in a countably infinite-dimensional Hilbert space $\mathbb{H}$, drawn from a distribution with zero mean and covariance matrix $\mathbf{A}_i = \mathbb{E}_{\mathcal{P}_i}[\mathbf{x}_i \mathbf{x}_i^\top]$. The label $y_i$ is generated by a global linear teacher,
\begin{equation*}
    y_i = \langle \theta^*, \mathbf{x}_i \rangle + \epsilon_i,
\end{equation*}
where $\theta^* \sim \mathcal{N}(0,\mathbf{I})$ is the ground-truth parameter and $\epsilon_i \sim \mathcal{N}(0,\sigma^2_i)$ is independent Gaussian noise.
%We consider a mixture distribution $\mathcal{P}(h)=\sum\limits_{i\in [K]}h_i \mathcal{P}_i$ defined by weights $h \in \Delta^{K-1}$  with covariance matrix $\mathbf{H}(h) = \sum_{i=1}^K h_i \mathbf{H}_i$.
%satisfying Assumption~\ref{assum:hyper} and $\mathbb{E}_{\mathbf{x} \sim \mathcal{P}(h)}[\mathbf{x}\mathbf{x}^{\top}]=\mathbf{H}(h)$.  

Given a data mixture weight $h \in \Delta^{K-1}$, the input data $\mathbf{x} \in \mathbb{H}$ are drawn from a distribution $\mathcal{D}=\sum\limits_{i=1}^K h_i\mathcal{P}_i$ with zero mean and covariance matrix $\mathbf{H}(h) := \mathbb{E}_{\mathbf{x} \sim \mathcal{D}}[\mathbf{x}\mathbf{x}^\top]=\sum\limits_{i=1}^K h_i \mathbf{A}_i$, where $\mathbf{A}_i:=\mathbb{E}_{\mathbf{x} \sim \mathcal{P}_i}[\mathbf{x}\mathbf{x}^\top]$.
We assume that every domain contributes a non-zero proportion to data mixture. \begin{assumption}\label{assm:d2}
    There is a universal constant $\lh>0$ such that $h_i\ge \lh$ holds for $i\in [K].$ For example, we can set $\lh=10^{-2}$.
\end{assumption}

\subsubsection{Data Generation}
We make the following assumptions regarding the data distribution.

Firstly, we make the ``Shared Head, Disjoint Tail'' assumption. In spectral analysis, an eigenvector represents a specific direction or pattern of variation in the data (e.g., a specific texture in images or topic in text), while its corresponding eigenvalue quantifies the variance (or strength) of that pattern.
Intuitively, we assume real-world data consists of \textit{universal patterns} shared across all domains (the head) and \textit{specialized nuances} unique to each domain (the tails).
To make the analysis tractable, we assume that all covariance matrices $\{\mathbf{A}_i\}_{i=1}^K$ share a common orthonormal basis of eigenvectors $\mathbf{U} = [u_1, u_2,\dots]$ (i.e., they are simultaneously diagonalizable). In this basis, each matrix $\mathbf{A}_i$ is diagonal with eigenvalues denoted by $\lambda^{(i)}_k$. We classify these eigenvectors into two categories:

\begin{itemize}
    \item \textbf{Shared Head ($k \le H$):} The first $H$ eigenvectors $\{u_1, \dots, u_H\}$ represent universal components. We assume all domains possess non-zero variance along these directions ($\lambda^{(i)}_k > 0$ for all $i$), meaning these patterns are present in every domain.
    
    \item \textbf{Disjoint Tail ($k > H$):} The remaining eigenvectors represent domain-specific components. We assume each domain $i$ possesses a unique set of eigenvectors $\{u_k^{(i)}\}$ that are orthogonal to the specific components of other domains. Following~\Cref{sec:pre_linear}, we assume that the variance along these directions follows a domain-specific power law:
    \begin{equation*}
        u_k^{(i)\top} \mathbf{A}_j u_k^{(i)} = 
        \begin{cases}
            k^{-\alpha_i} & \text{if } j = i \quad  \\
            0 & \text{if } j \neq i \quad .
        \end{cases}
    \end{equation*}
    Consequently, for $k > H$, the spectral structures are completely decoupled: each domain $i$ has positive eigenvalues $\lambda^{(i)}_k = k^{-\alpha_i}$ along its own unique eigenvectors and zero along the eigenvectors of others.
\end{itemize}
Under this assumption, the mixture covariance matrix $\mathbf{H}(h) = \sum h_j \mathbf{A}_j$ exhibits a decoupled structure in the tail. Because the domain-specific eigenvectors do not overlap, the eigenvalue of a unique component from domain $i$ in the mixture is simply its original variance scaled by the domain's proportion $h_i$. 
Specifically, for a tail eigenvector $u_k^{(i)}$ belonging to domain $i$, the corresponding eigenvalue in the mixture is:
\begin{equation*}
    \lambda(\mathbf{H}(h), u_k^{(i)}) = \sum_{j=1}^K h_j \cdot u_k^{(i)\top} \mathbf{A}_j u_k^{(i)} = h_i k^{-\alpha_i}.
\end{equation*}

In the following sections, we always work in the basis
\begin{equation*}
\begin{aligned}
\mathbf{U}=\big[&
u_1,\ldots,u_H, \\
&u^{(1)}_{H+1},\ldots,u^{(K)}_{H+1},
u^{(1)}_{H+2},\ldots,u^{(K)}_{H+2},\ldots
\big].
\end{aligned}
\end{equation*}
For ease of notation, we absorb the domain-specific superscripts by flattening the doubly-indexed tail eigenvectors into a singly-indexed sequence, relabeling the set $\{ u_k^{(i)} \mid k > H, i \in [K] \}$ sequentially as $u_{H+1}, u_{H+2}, \dots$. Formally, this establishes a mapping $\Gamma: \mathbb{N} \times \mathbb{N} \to \mathbb{N}$ for the tail indices ($k > H$ and $i \in [K]$), where the $k$-th index of domain $i$ is mapped to its new flattened index $\Gamma(i,k)$ via $\Gamma(i, k) = H + K(k - H - 1) + i$. In the following section, $\lambda_k$ denotes the eigenvalue associated with the eigenvector $u_k$, i.e., $\Hm u_k = \lambda_k u_k$ (but not necessarily the $k$-th largest eigenvalue).
In this basis with flattened indices, $\mathbf{A}_i$ is diagonal with $\mathbf{A}_i=\text{diag}(a_1, a_2, \dots)$ and \begin{equation*}
    a_k= u^{\top}_k\mathbf{A}_{i}u_k =\begin{cases}
        l^{-\alpha_{i}} & k=\Gamma(i,l) \\
        0 & \text{otherwise.}
    \end{cases}
    \end{equation*}
 In addition, $\mathbf{H}(h) = \sum h_j \mathbf{A}_j$ is diagonal as well.

We then make a standard assumption regarding hypercontractivity, which is also used in~\cite{li2025functional}. Intuitively, it requires the fourth moments of the data distribution to be reasonably bounded by its second moments, which ensures that the distribution $\mathcal{P}_i$ does not exhibit overly heavy tails.
\begin{assumption}[Hypercontractivity] \label{assm:d3}
For any domain $i \in [K]$ and any positive semi-definite (PSD) matrix $\mathbf{M}$, there exists a constant $C_0 > 0$ such that 
$$ \E_{\x\sim \mathcal{P}_i} \left[ \x\x^{\top}\mathbf{M}\x\x^{\top} - \mathbf{A}_i\mathbf{M}\mathbf{A}_i \right] \preceq C_0 \tr(\mathbf{A}_i\mathbf{M})\mathbf{A}_i. $$
\end{assumption}
We note that this is a very mild and standard assumption in theoretical analysis. A wide variety of standard distributions satisfy this condition for a small constant $C_0$, including but not limited to Gaussian distributions, sub-Gaussian distributions, and bounded distributions (such as uniform distributions on a sphere or a hypercube). Consequently, this assumption provides a robust foundation for concentration of measure without requiring strict distributional forms.

We assume the following regarding the prior of $\theta^*$. 
\begin{assumption}\label{assm:secd1}
    Assume that $\theta^*$ satisfies a prior such that $\E[(\theta^*)^{\otimes 2}]=\mathbf{I}$.
\end{assumption}
This isotropic prior assumption is widely adopted in the theoretical analysis of overparameterized models \cite{lin2024}. Conceptually, it postulates that the optimal parameter $\theta^*$ exhibits no preferential directional bias in the feature space, meaning that the signal energy is uniformly distributed across all coordinate components. From a technical standpoint, this formulation isolates the effect of the parameter distribution, ensuring that the generalization error and the resulting scaling laws are fundamentally driven by the spectral properties of the data covariance matrix rather than the parameter alignment.

We assume the following regarding the range of power law exponents, based on empirical observations from \citep{kaplan2020scaling, hoffmann2022training}.
\begin{assumption}\label{assm:c4}
    For all $i\in [K]$, $1<\alpha_i<3$.
\end{assumption}

%Assumption~\ref{assm:c4} ensure that the exponents do not vary too much, which is also empirically verified in \citep{kaplan2020scaling}.

\subsubsection{Model Training}

We then formalize the training process.

To model a neural network with a finite capacity of $N$ parameters, we project the high-dimensional input $\mathbf{x}$ into a lower-dimensional feature space using a \textbf{sketching operator} $\mathbf{S} \in \mathbb{R}^{N \times d}$. The model learns a weight vector $\theta \in \mathbb{R}^N$ by minimizing the squared error on the projected features $\tilde{\mathbf{x}} = \mathbf{S}\mathbf{x}$:
    \begin{equation*}
        \widehat\theta = \operatorname*{arg\,min}_{\theta \in \mathbb{R}^N} \frac{1}{D} \sum_{i=1}^D (y_i - \theta^\top \mathbf{S} \mathbf{x}_i)^2.
    \end{equation*}

Now consider an arbitrary mixture $h \in \Delta^{K-1}$ and data covariance $\mathbf{H}=\sum\limits_{i=1}^K h_i \mathbf{A}_i$. Following~\citep{li2025functional}, we consider the \textit{spectral truncation} sketch, which projects the input data onto the subspace spanned by the top-$N$ eigenvectors of $\mathbf{H}$. Formally, 
\begin{equation*}
    \mathbf{S}_{i,j} = \begin{cases}
        1, & \text{ if } \lambda_j \text{ is the } i^{\text{th}} \text{ largest diagonal value, } \\
        0, & \text{otherwise,}
    \end{cases}
\end{equation*} 
such that the eigenvalues of $\mathbb{E}[\tilde{\mathbf{x}} \tilde{\mathbf{x}}^\top] = \mathbb{E}[\mathbf{S}\mathbf{x} \mathbf{x}^\top \mathbf{S}^\top]$ are the largest $N$ eigenvalues of $\mathbf{H}$.

\paragraph{Intuition behind $\mathbf{S}$.} As discussed in \Cref{sec:elrm}, the eigenvalue $\Hm_{j,j}$ can be viewed as the frequency of the $j$-th skill. Therefore, by the top-$N$ sketching operator $\mathbf{S}$, the model essentially learns the $N$ skills with the highest frequencies.

Then the optimal model parameter $\widehat\theta = \operatorname*{arg\,min}_{\theta \in \mathbb{R}^N} \frac{1}{D} \sum_{i=1}^D (y_i - \theta^\top \mathbf{S} \mathbf{x}_i)^2$ is trained by one-pass SGD with cosine learning rate decay~\cite{loshchilov2017sgdrstochasticgradientdescent} on the sketched inputs. More specifically, we initialize the model parameter at $\mathbf{\theta}_0 = \mathbf{0}$. At step $k$, the algorithm samples a data pair $(\mathbf{x}_k, y_k)$, constructs the observation $\tilde{\mathbf{x}}_k = \mathbf{S}\mathbf{x}_k$, and performs a gradient descent update on the squared loss $\ell(\mathbf{\theta}) = \frac{1}{2}(\tilde{\mathbf{x}}_k^\top \mathbf{\theta} - y_k)^2$:
\begin{align*}
    \mathbf{\theta}_{k+1} &= \mathbf{\theta}_k - \eta_k \nabla_{\mathbf{\theta}} \ell(\mathbf{\theta}_k; \tilde{\mathbf{x}}_k, y_k) \nonumber \\
    &= \mathbf{\theta}_k - \eta_k \tilde{\mathbf{x}}_k (\tilde{\mathbf{x}}_k^\top \mathbf{\theta}_k - y_k),
\end{align*}
where $\eta_k=\eta_0\left( 1+\cos(\pi k/D)\right).$

We assume that the training loss is always bounded. 
\begin{assumption}\label{assm:c5}

    The train loss over the training process has a finite upper bound $L$. Formally, for all $t$, $\E\left[(\langle \Sk \x,\theta(t) \rangle-y)^2\right]\le L$.
\end{assumption}

We also impose an assumption on the parameter size $N$ and data size $D$:
\begin{assumption}\label{assm:c6}
    There is a constant $\varepsilon>0$ such that $N^{\alpha_i} \ge D^{1+\varepsilon}$ for all $i \in[K]$.
\end{assumption}
Assumption~\ref{assm:c6} is a remarkably mild condition that is naturally satisfied in standard large-scale training regimes. Following empirical scaling laws~\cite{kaplan2020scaling}, compute-optimal models typically scale the parameter size $N$ linearly with the dataset size $D$ (e.g., $D \approx 20N$). Under this linear scaling relationship ($N \propto D$), the left-hand side of our inequality scales as $\mathcal{O}(D^{\alpha_i})$. Since $\alpha_i > 1$ by definition, there strictly exists a sufficiently small constant $\varepsilon > 0$ such that $1 + \varepsilon \le \alpha_i$. Consequently, the polynomial growth of $N^{\alpha_i}$ will trivially dominate $D^{1+\varepsilon}$ for large-scale $D$. Therefore, rather than imposing a restrictive capacity requirement, this assumption merely formalizes the standard overparameterized operational regime of modern language models, while simultaneously providing the necessary analytical bounds for our subsequent theoretical proofs.

\subsection{SDE Approximation}
Following previous work~\cite{li2025functional,qiu2025scalingcollapserevealsuniversal}, we use a continuous-time approximation of one-pass SGD, which simplifies the analysis.

Since
\begin{align*}
    \theta_{k+1}-\theta_k&= -\eta_k \left( \E [\nabla_{\mathbf{\theta}} \ell(\theta_k)]+\nabla_{\mathbf{\theta}} \ell(\theta_k)-\E[\nabla_{\mathbf{\theta}} \ell(\theta_k)] \right),
\end{align*}
we have
\begin{align*}
    \theta_{D}-\theta_0= -\sum\limits_{k=0}^{D-1} \eta_k \E [\nabla_{\mathbf{\theta}} \ell(\theta_k)] - \sum\limits_{k=0}^{D-1}\eta_k \left( \nabla_{\mathbf{\theta}} \ell(\theta_k)-\E[\nabla_{\mathbf{\theta}} \ell(\theta_k)] \right).
\end{align*}
We generalize the discrete sequence $\{\theta_0, \ldots, \theta_D\}$ to a continuous function $\theta(\cdot)$, and similarly extend $\eta_k$ to $\eta(\cdot)$.

Now we compute $\mathbf{\Sigma}(\theta_k):=\left( \nabla_{\mathbf{\theta}} \ell(\theta_k)-\E[\nabla_{\mathbf{\theta}} \ell(\theta_k)] \right)^{\otimes2}$.

Since
\begin{align*}
    \nabla_{\mathbf{\theta}} \ell(\theta_k)=\mathbf{S}\x_k \x_k^{\top}\left(\mathbf{S}^{\top}\theta_k-\theta^* \right)-\epsilon_k\mathbf{S}\x_k,
\end{align*}
we have
\begin{align*}
    \mathbf{\Sigma}(\theta_k)=\E \left[ \left( \nabla_{\mathbf{\theta}} \ell(\theta_k)-\E[\nabla_{\mathbf{\theta}} \ell(\theta_k)] \right)\right]^{\otimes2}&=\left[\Sk \left(\x_k \x_k^{\top}-\Hm \right)\left(\Sk^{\top}\theta-\theta^* \right)  \right]^{\otimes 2}+\sum\limits_{i=1}^K h_i \sigma_i^2 \Sk \mathbf{A}_i \Sk^{\top}.
\end{align*}
Using Euler-Maclaurin equation, we have
\begin{align}\label{eq:c1}
    \theta_{D}-\theta_0\approx\theta(D)-\theta(0)=\int_0^{D}-\eta(k)\Sk \Hm \left(\Sk^{\top} \theta(k)-\theta^* \right)  \df k + \int_0^{D} \eta(k) \sqrt{\mathbf{\Sigma}(\theta(k))} \df \mathbf{B}_k, 
\end{align}
where $\B _k \in \mathbb{R}^{N}$ is a N-dimensional  Brownian motion.
Following~\cite{li2025functional}, we define \textbf{intrinsic time} $$\tau(t):=\int_0^{t} \eta(k)\df k.$$
We can rewrite equation~\ref{eq:c1} as
\begin{align}
\label{eq:c2}
    \theta(D)-\theta(0)=\int_0^D -\Sk \Hm (\Sk^{\top}\theta(k)-\theta^*)\df \tau(k)+\sqrt{\eta(k)\mathbf{\Sigma}(\theta(k))}\df \B_\tau.
\end{align}
 Taking the derivative of Equation~\ref{eq:c2}, we have the following lemma.
\begin{lemma}\label{lemac.7}

Define $\w(t):=\Sk^{\top}\theta(\tau^{-1}(t))-\theta^*$ and $\gamma(t):=\eta(\tau^{-1}(t))$. Then
    \begin{align*}\label{eq:c3}
    \df \w=-\Sk^{\top} \Sk\Hm \w \df t+ \Sk^{\top}\sqrt{\gamma(t)\mathbf{\Sigma}(\theta(\tau^{-1}(t)))} \Sk \df \B_t,
\end{align*}
where $\B _t \in \mathbb{R}^{d}$ is a $d$-dimensional Brownian motion.

\end{lemma}

\subsection{Test Loss}
We first compute the expected test loss on a certain domain $\tau\in [K]$ when the model is trained with mixture $h \in \Delta^{K-1}$.

In our projected linear regression framework, the test loss on domain $\tau$ equals
\begin{align*}
    L_{\tau}&=\E_{\x \sim \mathcal{P}_{\tau}}\left[ (\langle \Sk \x,\theta\rangle -\langle \x,\theta^* \rangle)^2\right]+\sigma_\tau^2 \\
    &=\E \left[\w^{\top}(D\eta_0) \mathbf{A}_{\tau}\w(D\eta_0)\right]+\sigma_\tau^2.
\end{align*}

\begin{theorem}\label{thm:c3}

    Consider the projected linear regression model defined in \Cref{app:theory_des} and an arbitrary domain $\tau$. Let $r_k$ be the rank of the eigenvalue $\lambda_k$ among all eigenvalues of $\Hm=\sum h_j \mathbf{A}_j$, sorted in descending order. Let $\mathbf{A}_\tau=\text{diag}(a_1, a_2, \dots)$ be the data covariance of domain $\tau$ (i.e., $\mathbf{A}_{\tau}=\E_{\mathcal{P}_{\tau}}[\x \x^{\top}]$), where
    \begin{equation*}
    a_k =\begin{cases}
        l^{-\alpha_{\tau}} & k=\Gamma(\tau,l) \\
        0 & \text{otherwise}
    \end{cases}
\end{equation*} 
    Then the expected test loss on domain $\tau$, denoted by $L_\tau$, satisfies
    \begin{align}
        \text{Bias}+\text{Approx}+\text{Var}_1 + \sigma^2_{\tau} \le  L_\tau \le \text{Bias}+\text{Approx}+\text{Var}_2 + \sigma^2_{\tau},
    \end{align}

    where \begin{align*}
    \text{Bias}&:=\sum\limits_{k:\, r_k\le N} \exp(-2\lambda_kD\eta_0)a_k,\\
    \text{Approx}&:=\sum\limits_{k: \, r_k>N} a_k,\\
    \text{Var}_1&:=\sum\limits_{k: \, r_k \le N,k>H} h_{\tau}\sigma^2_{\tau} a_k\lambda_k \int_0^{D\eta_0} \exp(-2\lambda_k(D\eta_0-t))\gamma(t)\df t ,\\
    \text{Var}_2&:=\sum\limits_{k: \, r_k\le N,k>H} a_k \lambda_k\int_0^{D\eta_0}\exp(-2\lambda_k(D\eta_0-t))\gamma(t)\left(h_{\tau} \sigma^2_{\tau}+ \frac{C_0}{\lh} \E \left[\w(t)^{\top}\Hm \w(t)\right]\right)\df t \\
        &+\sum\limits_{k: \, r_k\le N,k\le H} a_k \lambda_k\int_0^{D\eta_0}\exp(-2\lambda_k(D\eta_0-t))\gamma(t)\left(\sigma_{\max}^2+ \frac{C_0}{\lh} \E \left[\w(t)^{\top}\Hm \w(t)\right]\right)\df t,
\end{align*}
and $\gamma(t) =\eta(\tau^{-1}(t)), \lambda_k = u_k^{\top}\Hm u_k = \Hm_{k,k}, \sigma_{\max}=\max_{i\in[K]}\sigma_i.$

\end{theorem}
\begin{proof}
We consider the quadratic function $f(\w) = \w^{\top}\mathbf{A}\w$ for any diagonal matrix $\mathbf{A}=\text{diag}(a'_1,a'_2,\ldots)$ (not necessarily $\mathbf{A}_i$). Recall that for a multidimensional stochastic process, Itô's lemma expands a twice-differentiable function $f(\w)$ as
\begin{equation*}
    \df f(\w(t)) = (\nabla f(\w(t)))^{\top} \df \w(t) + \frac{1}{2}\tr\left[\nabla^2 f(\w(t)) \cdot (\df \w(t))^{\otimes 2}\right].
\end{equation*}
For our specific function, the gradient is $\nabla f = 2\mathbf{A}\w$ and the Hessian is $\nabla^2 f = 2\mathbf{A}$. Substituting these derivatives into Itô's formula and taking the expectation yields
\begin{align*}
\df \E[\w^{\top}(t)\mathbf{A}\w(t)]
&= 2\E\left[\w^{\top}(t)\mathbf{A}\df\w(t)\right]
 + \E\left[\tr\left(\mathbf{A}(\df\w(t))^{\otimes2}\right)\right] \\
&= -2\E\left[\w^{\top}(t)\mathbf{A}\Sk^{\top}\Sk\Hm\w(t)\right]\df t \\
&\quad +\gamma(t)\E\left[
\w^{\top}(t)(\x\x^{\top}-\Hm)\Sk^{\top}\mathbf{A}\Sk
(\x\x^{\top}-\Hm)\w(t)\right]\df t \\
&\quad +\gamma(t)\tr\left[
\mathbf{A}\sum_{i=1}^K h_i\sigma_i^2
\Sk\mathbf{A}_i\Sk^{\top}\right]\df t.
\end{align*}
     We first take the expectation over $\x$. Using Assumptions~\ref{assm:d3} and~\ref{assm:d2}, we bound the term $\w^{\top}(\x \x^{\top}-\Hm)\Sk^{\top} \mathbf{A}\Sk(\x\x^{\top}-\Hm)\w$ as follows:
     \begin{align*}
         \E_{\x \sim \mathcal{D}(h)}\left[\w^{\top}(\x \x^{\top}-\Hm)\Sk^{\top} \mathbf{A}\Sk(\x\x^{\top}-\Hm)\w\right] &= \w^{\top}\E\left[\x\x^{\top}\Sk^{\top}\mathbf{A}\Sk \x\x^{\top}-\Hm \Sk^{\top}\mathbf{A}\Sk \Hm\right] \w \\
         &= \w^{\top}\sum h_i\E_{\x \sim \mathcal{P}_i}\left[\x\x^{\top}\Sk^{\top}\mathbf{A}\Sk \x\x^{\top}-\mathbf{A}_i \Sk^{\top}\mathbf{A}\Sk \mathbf{A}_i\right] \w \\
         & \le \w^{\top} \sum h_i C_0 \tr[\Sk^{\top}\mathbf{A}\Sk \mathbf{A}_i]\mathbf{A}_i \w \\
         & \le \w^{\top}\frac{C_0}{\lh}\sum \tr[\Sk^{\top}\mathbf{A}\Sk \mathbf{A}_i] h_i \mathbf{A}_i \w \\
         &=\frac{C_0}{\lh} \tr[\Sk^{\top}\mathbf{A}\Sk \Hm] \w(t)^{\top}\Hm \w(t).
     \end{align*}
     Setting $\mathbf{A}=\mathbf{E}_{k,k}$ and letting $y_k(t)=\E[\w_k^2(t)]$, we obtain the following two-sided bounds for $y_k$: 
     \begin{itemize}
         \item For $r_k \le N$ and $k>H$, \begin{align*}
             -2\lambda_k y_k(t)+\gamma(t)\lambda_k\cdot h_{\tau}\sigma_{\tau}^2\le y'_k(t) \le -2\lambda_k y_k(t)+\gamma(t)\lambda_k\left(h_{\tau} \sigma^2_{\tau}+\frac{C_0}{\lh} \E[ \w(t)^{\top}\Hm \w(t) ] \right).
         \end{align*}
         \item For $r_k \le N$ and $k\le H$, \begin{align*}
             -2\lambda_ky_k(t) \le y'_k(t) \le -2\lambda_ky_k(t)+\gamma(t) \cdot \lambda_k \cdot \sigma^2_{\max} +\gamma(t)\lambda_k\frac{C_0}{\lh}\E[\w(t)^{\top}\Hm \w(t)],
         \end{align*}
	         where $\sigma_{\max}=\max_{i\in [K]} \sigma_i$.
         \item For $r_k >N$, $y'_k(t)=0.$
     \end{itemize}
	     Solving the differential equation, for all $k>H$ such that $r_k\le N$, we have \begin{align*}
	         y_k(D\eta_0) &\ge \exp(-2\lambda_k D\eta_0) y_k(0)+ h_{\tau} \sigma^2_{\tau}\lambda_k\int_0^{D\eta_0}\exp(-2\lambda_k(D\eta_0-t))\gamma(t) \df t ,\\
         y_k(D\eta_0) &\le \exp(-2\lambda_k D\eta_0) y_k(0)+ \lambda_k\int_0^{D\eta_0}\exp(-2\lambda_k(D\eta_0-t))\gamma(t)\left(h_{\tau} \sigma^2_{\tau}+ \frac{C_0}{\lh}                                     \E[\w(t)^{\top}\Hm \w(t)]\right)\df t .
     \end{align*}
	     Since $\E[\w^{\top}(t)\mathbf{A}\w(t)]=\sum_k y_k(t)a_k$ and $\E[\w^2_i(0)]=1$ by Assumption~\ref{assm:secd1}, 
     we have
    \begin{align*}
        \E[\w^{\top}(D\eta_0)\mathbf{A}\w(D\eta_0)]&=\sum y_k(D\eta_0)a_k\\
        &\ge \sum\limits_{k:r_k>N}a_k \\
        &+ \sum\limits_{k:\, r_k\le N} \exp(-2\lambda_kD\eta_0)a_k \\
        &+\sum\limits_{k: \, r_k \le N,k>H} h_{\tau}\sigma^2_{\tau} a_k\lambda_k \int_0^{D\eta_0} \exp(-2\lambda_k(D\eta_0-t))\gamma(t)\df t 
    \end{align*}
    and
    \begin{align*}
        \E[\w^{\top}(D\eta_0)\mathbf{A}\w(D\eta_0)]&=\sum y_k(D\eta_0)a_k\\
        &\le \sum\limits_{k:r_k>N}a_k \\
        &+ \sum\limits_{k:\, r_k\le N} \exp(-2\lambda_kD\eta_0)a_k \\
        &+\sum\limits_{k: \, r_k\le N,k>H} a_k \lambda_k\int_0^{D\eta_0}\exp(-2\lambda_k(D\eta_0-t))\gamma(t)\left(h_{\tau} \sigma^2_{\tau}+ \frac{C_0}{\lh} \E \left[\w(t)^{\top}\Hm \w(t)\right]\right)\df t \\
        &+\sum\limits_{k: \, r_k\le N,k\le H} a_k \lambda_k\int_0^{D\eta_0}\exp(-2\lambda_k(D\eta_0-t))\gamma(t)\left(\sigma_{\max}^2+ \frac{C_0}{\lh} \E \left[\w(t)^{\top}\Hm \w(t)\right]\right)\df t
    \end{align*}
\end{proof}

In the following sections, we estimate each term in the bounds separately.

\subsection{The Approximation Error Term}

For an arbitrary domain $\tau$, we first provide a bound for the approximation error term $\text{Approx}=\sum\limits_{k: \, r_k>N} a_k$.
We demonstrate that, aside from a gap introduced by discretization, this term is equivalent to the domain loss $L_\tau(h) = c_\tau x_\tau^*(h)^{-b_\tau}$ in our Extended Quantization Model, with $c_\tau = \frac{1}{\alpha_\tau-1}$ and $b_\tau = \alpha_\tau-1$.

\begin{theorem}\label{thm:d4}
    Consider the projected linear regression model defined in \Cref{app:theory_des}. Let $x^*$ be the solution to Problem~\ref{eq:quant_opt} with parameters $c_i=\frac{1}{\alpha_i-1}$ and $b_i=\alpha_i-1$. Then, for any domain $\tau$ and sufficiently large $N$, the term $\text{Approx}$ in \Cref{thm:c3} satisfies
    \begin{equation*}
        \left|\text{Approx}-c_\tau (x^*_{\tau})^{-b_\tau}\right| \le \frac{C_{13}}{N^{\alpha_{\min}}},
    \end{equation*}
    where $C_{13}$ is a constant that depends only on $\alpha$.
\end{theorem}

%In our framework, the loss induced by model capacity is proportional to the sum of the spectrum not captured by the model. In the following proof, we show that the Top-$N$ feature model is asymptotically equivalent to Problem~\ref{eq:quant_opt}.

\subsubsection{Proof of the Approximation-Error Bound}
To understand why the term $\text{Approx}$ is equivalent to the domain loss in the Extended Quantization Model for an arbitrary domain $\tau$, note that Theorem~\ref{thm:c3} establishes:
$$\text{Approx} = \sum_{k: \, r_k > N} a_k,$$
where $r_k$ is the rank of the eigenvalue $\lambda_k$ among all eigenvalues of $\Hm=\sum h_j \mathbf{A}_j$, sorted in descending order, and
\begin{equation*}
    a_k=\begin{cases}
        l^{-\alpha_{\tau}} & k=\Gamma(\tau,l) \\
        0 & \text{otherwise.}
    \end{cases}
\end{equation*}
If we interpret the eigenvalues as the frequencies of skills, this formulation indicates that $\text{Approx}$ is precisely the sum of the frequencies of the unlearned skills in domain $\tau$, assuming the model learns the $N$ skills with the highest frequencies (under a training mixture $h$). This aligns perfectly with the Extended Quantization Model if we define $c_\tau$ as the normalization constant such that $\sum_k a_k/c_\tau = 1$. Under this formulation, the optimization objective (Problem~\ref{eq:quant_opt}) minimizes the total loss by learning the highest-frequency skills across all $K$ domains, and the loss for any domain $\tau$ is incurred  by its unlearned skills. However, since the Extended Quantization Model considers a continuous skill space, there will be a gap due to discretization.

We first prove a lower bound for $x^*_\tau$ as a function of $N$, ensuring that $x^*_\tau$ grows strictly monotonically with $N$. The lemma also guarantees the asymptotic relationship $\text{Approx}=\frac{(x^*_{\tau})^{1-\alpha_{\tau}}}{\alpha_{\tau}-1}(1+o_N(1))$.

\begin{lemma}\label{thm:propertyx}
Consider the Extended Quantization Model in \Cref{sec:theory_quant} (Problem~\ref{eq:quant_opt}):
\begin{equation}
\begin{aligned}
    \min_{\mathbf{x}} \quad & L(\mathbf{x}) = \sum_{i=1}^K h_i c_i x_i^{-b_i} \\
    \text{s.t.} \quad & \sum_{i=1}^K (x_i - H) \le N - H, \\
    & x_i \ge H, \quad \forall i \in \{1, \dots, K\},
\end{aligned}
\end{equation}
where $h_i, c_i, b_i > 0$ for all $i$. Let $b_{\min} = \min_{1 \le i \le K} b_i$. As $N \to \infty$, the optimal solution  $x^*$ satisfies the asymptotic lower bound:
$$x_{\tau}^* = \Omega\left( N^{\frac{b_{\min} + 1}{b_{\tau} + 1}} \right).$$
\end{lemma}

\begin{proof}
Let $A_i = h_i c_i > 0$ for all $i \in \{1, \dots, K\}$. Since $A_i > 0$ and $b_i > 0$, the objective function $L(\mathbf{x})$ is strictly monotonically decreasing with respect to each variable $x_i$. Consequently, to minimize the objective function, the variables $x_i$ must take the largest possible values permitted by the feasible region. This implies that the sum constraint must be active at the optimal solution, yielding the equality $\sum_{i=1}^K x_i = N + (K-1)H$. Furthermore, as the total available resource $N$ approaches infinity, the optimal values $x_i^*$ will also approach infinity. Thus, for sufficiently large $N$, the lower bound constraints $x_i \ge H$ become strictly inactive and can be omitted from the asymptotic analysis.

We proceed by applying the method of Lagrange multipliers. The Lagrangian associated with the equality constraint is given by
$$\mathcal{L}(\mathbf{x}, \lambda) = \sum_{i=1}^K A_i x_i^{-b_i} + \lambda \left( \sum_{i=1}^K x_i - N - (K-1)H \right),$$
where $\lambda > 0$ is the Lagrange multiplier. Taking the partial derivative of $\mathcal{L}$ with respect to $x_i$ and equating it to zero yields the first-order necessary conditions for optimality:
$$\frac{\partial \mathcal{L}}{\partial x_i} = -b_i A_i x_i^{-(b_i + 1)} + \lambda = 0, \quad \forall i \in \{1, \dots, K\}.$$
Rearranging this expression, we obtain a relationship between the optimal variable $x_i$ and the multiplier $\lambda$:
$$\lambda = b_i A_i x_i^{-(b_i + 1)}.$$
Since $\lambda$ is a global constant across all dimensions, we can equate the expressions for an arbitrary index $i$ and the specific index $\tau$, yielding
$$b_i A_i x_i^{-(b_i + 1)} = b_{\tau} A_{\tau} x_{\tau}^{-(b_{\tau} + 1)}.$$
Solving this equation for $x_i$ in terms of $x_{\tau}$, we find
$$x_{i} = \left( \frac{b_i A_i}{b_{\tau} A_{\tau}} \right)^{\frac{1}{b_i + 1}} x_{\tau}^{\frac{b_{\tau} + 1}{b_i + 1}}.$$
Substituting this relationship back into the active resource constraint gives
$$\sum_{i=1}^K \left( \frac{b_i A_i}{b_{\tau} A_{\tau}} \right)^{\frac{1}{b_i + 1}} x_{\tau}^{\frac{b_{\tau} + 1}{b_i + 1}} = N + (K-1)H.$$
We now analyze the asymptotic behavior of this equation as $N \to \infty$. On the right-hand side, the constant term $(K-1)H$ becomes negligible, so the right-hand side is asymptotically equivalent to $N$. On the left-hand side, we have a sum of fractional powers of $x_{\tau}$. As $N \to \infty$ implies $x_{\tau} \to \infty$, the behavior of the sum is completely dominated by the term with the highest exponent. The exponent of $x_{\tau}$ for the $i$-th term is $\frac{b_{\tau} + 1}{b_i + 1}$. This exponent is maximized when its denominator, $b_i + 1$, is minimized, which occurs exactly when $b_i = b_{\min} = \min_{1 \le j \le K} b_j$.

Let $\mathcal{I}_{\min} = \{ i \mid b_i = b_{\min} \}$ be the index set of all terms achieving this minimum exponent. Extracting these dominant terms, we establish the asymptotic equivalence
$$\sum_{i \in \mathcal{I}_{\min}} \left( \frac{b_{\min} A_i}{b_{\tau} A_{\tau}} \right)^{\frac{1}{b_{\min} + 1}} x_{\tau}^{\frac{b_{\tau} + 1}{b_{\min} + 1}} \sim N.$$
Letting $C = \sum_{i \in \mathcal{I}_{\min}} \left( \frac{b_{\min} A_i}{b_{\tau} A_{\tau}} \right)^{\frac{1}{b_{\min} + 1}}$, which is a strictly positive constant, the relation simplifies to
$$C x_{\tau}^{\frac{b_{\tau} + 1}{b_{\min} + 1}} \sim N.$$
Solving this asymptotic equivalence for $x_{\tau}$ yields
$$x_{\tau} \sim \left( \frac{1}{C} \right)^{\frac{b_{\min} + 1}{b_{\tau} + 1}} N^{\frac{b_{\min} + 1}{b_{\tau} + 1}}.$$
This demonstrates that the growth rate of $x_{\tau}^*$ is proportional to $N^{\frac{b_{\min} + 1}{b_{\tau} + 1}}$. Therefore, we conclude that the optimal solution $x_{\tau}^*$ satisfies the strict asymptotic lower bound $x_{\tau}^* = \Omega\left( N^{\frac{b_{\min} + 1}{b_{\tau} + 1}} \right)$, completing the proof.
\end{proof}

We are now ready to prove \Cref{thm:d4}. We prove the equivalence of $\text{Approx}$ and the domain loss under the Extended Quantization Model, and we bound the discretization gap as follows.

    By setting $c_i=\frac{1}{\alpha_i-1}$, Problem~\ref{eq:quant_opt} can be written as:
    \begin{align*}
    \min_{x} \quad & L=\sum_{i=1}^K h_i \frac{1}{\alpha_i-1} x_i^{1-\alpha_i} \\
    \text{s.t.} \quad & \sum_{i=1}^K (x_i - H) \le N - H, \\
    & x_i \ge H, \quad \forall i.
\end{align*}
    Let $x^*$ be the optimal continuous solution of the optimization problem above, and let $x'_{\tau}:=\max\{k:r_{\Gamma(\tau,k)}\le N\}$ be the number of eigenvalues from domain $\tau$ that are selected by the sketching operator $\mathbf{S}$ (with the shared head counted as well). We are to bound the difference in $x^*_{\tau}$ and $x'_{\tau}$ as $|x^*_{\tau}-x'_{\tau}|\le K$, which is tight enough to obtain the conclusion in this theorem.
    \begin{lemma}
    	For any domain $\tau$, $|x^*_{\tau}-x'_{\tau}|\le K$, where $K$ is the number of domains.
    \end{lemma}

\begin{proof}
We first characterize the optimal continuous solution $x^*$ via KKT conditions. By the KKT conditions, there exist multipliers $\lambda \ge 0$ and $\mu_i \ge 0$ such that:
    \begin{equation}\label{eq:kkt_sum}
    \sum_{i=1}^K x^*_i = N + (K-1)H,
    \end{equation}
    \begin{equation} \label{eq:kkt_deriv}
    \frac{h_i}{(x^*_i)^{\alpha_i}} - \mu_i = \lambda, \quad \forall i,    
    \end{equation}
    \begin{equation} \label{eq:kkt_slack}
    \mu_i(x^*_i - H) = 0, \quad \forall i. 
    \end{equation}
    Note that equality holds in \eqref{eq:kkt_sum} because the objective function strictly decreases as $x_i$ increases.
    
    By Lemma~\ref{thm:propertyx}, for sufficiently large $N$, we have $x_i^* > H$ for all $i \in [K]$, which by \Cref{eq:kkt_slack} implies that $\mu_i = 0$ for all $i \in [K]$. As a result, by \Cref{eq:kkt_deriv}, we have 
    $$
    \frac{h_i}{(x^*_i)^{\alpha_i}}= \lambda, \quad \forall i
    $$
	that is,
	$$
    x^*_i = \left(\frac{h_i}{\lambda} \right)^{\frac{1}{\alpha_i}}, \quad \forall i.
    $$
Then by the capacity constraint \Cref{eq:kkt_sum}, we easily find the optimal $x^*$ by finding the $\lambda$ that satisfies the capacity constraint. Define $S(\lambda) := \sum_{i \in [K]} \left( \frac{h_i}{\lambda} \right)^{\frac{1}{\alpha_i}}$. It is easy to verify that the equation $S(\lambda) = N + (K-1)H$ has a unique solution $\lambda^*$. Then for all $i \in [K]$, 
$$h_i (x^*_i)^{-\alpha_i} = \lambda^*.$$

Next we bound the difference between $x^*_\tau$ and $x'_\tau$, where $x'_\tau$ represents the number of eigenvalues from domain $\tau$ that are selected by the sketching operator $\mathbf{S}$ (with the shared head counted). Recall that the sketching operator $\mathbf{S}$ picks the largest $N$ eigenvalues and the $k$-th largest eigenvalue from domain $i$ equals $h_i k^{-\alpha_i}$. We connect $x^*_\tau$ and $x'_\tau$ by picking eigenvalues instead by a common threshold $\lambda^* = h_i (x^*_i)^{-\alpha_i}$. Define 
$$
U = \{j : \Hm_{j,j} \ge \lambda^*\}.
$$
Then $U$ must pick the first $y_i = \lfloor x_i^* \rfloor$ eigenvalues from domain $i$ for all $i$ (with the shared head counted). Since $x_i^* - 1 < y_i \le x_i^*$ and $H + \sum_{i=1}^K (x^*_i - H) = N$, we must have the following for the size of $U$:
$$ N - K < |U| = \left(H + \sum_{i=1}^K (y_i-H) \right) \le N. $$
Thus, the threshold $\lambda^* = h_i (x^*_i)^{-\alpha_i}$ selects at least $N - K+1$ eigenvalues, and these eigenvalues must be the largest ones. This implies that the sketching operator $\mathbf{S}$ must select all eigenvalues in $U$, while adding at most $K-1$ additional eigenvalues from any given domain $\tau$. Therefore, we have: 
$$
x_\tau^* - K \le y_\tau \le x'_\tau \le y_\tau + (K-1) \le x_\tau^* + K.
$$

\end{proof}

We can now bound the difference between $\text{Approx}$ and $\frac{(x^*_\tau)^{1-\alpha_\tau}}{\alpha_\tau - 1}$. 
 By $\text{Approx} = \sum_{k:r_k>N} a_k=\sum\limits_{j>x'_{\tau}}j^{-\alpha_{\tau}}$, we have

$$ \sum_{j > y_\tau + K} j^{-\alpha_\tau} \le \text{Approx} \le \sum_{j > y_\tau} j^{-\alpha_\tau} $$

To bound these discrete sums, define the continuous integral tail $F_\tau(z) = \int_z^\infty x^{-\alpha_\tau} \df x = \frac{z^{1-\alpha_\tau}}{\alpha_\tau - 1}$. Bounding the sums with integrals gives:
$$ F_\tau(y_\tau + K + 2) \le \text{Approx} \le F_\tau(y_\tau) $$

Because $y_\tau \le x_\tau^* < y_\tau + 1$, the deviation between the discrete approximation and the continuous ideal $F_\tau(x_\tau^*)$ is constrained by the maximum index gap. Therefore, we have:
\begin{align*}
    \left| \text{Approx} - \frac{(x^*_\tau)^{1-\alpha_\tau}}{\alpha_\tau - 1} \right| &\le \frac{K+2}{(x^*_\tau)^{\alpha_\tau}} \\
    &\le \frac{C_{13}}{N^{\alpha_{\min}}}
\end{align*}
where the final inequality follows from Lemma~\ref{thm:propertyx}, and $C_{13}$ is a constant depending only on $\alpha$. This completes the proof of \Cref{thm:d4}.

    %Now we begin to proof Theorem~\ref{thm:d4}.

\subsection{The \text{Bias} Term}

In this section, we bound the bias term $\text{Bias}=\sum\limits_{k:\, r_k\le N} \exp(-2\lambda_kD\eta_0)a_k$ where \begin{equation*}
    a_k=\begin{cases}
        l^{-\alpha_{\tau}} & k=\Gamma(\tau,l) \\
        0 & \text{otherwise.}
    \end{cases}
    \end{equation*} in the expected test loss for an arbitrary domain $\tau$.

\begin{theorem}\label{thm:d5}
   Consider the projected linear regression model defined in \Cref{app:theory_des}. For any model size $N$, data size $D$, any mixture $h \in \Delta^{K-1} $, and any domain $\tau \in [K]$, when the model is trained with mixture $h$, the bias term in the expected loss for domain $\tau$ in \Cref{thm:c3} has
    \begin{align*}
        \text{Bias} &= \frac{\Gamma\left(1-\frac{1}{\alpha_{\tau}}\right)}{\alpha_{\tau}(2\eta_0)^{1-\frac{1}{\alpha_{\tau}}}} \frac{1}{\left(Dh_{\tau}\right)^{1-\frac{1}{\alpha_{\tau}}}} + \mathcal{E},
    \end{align*}
    where the error $\mathcal{E}$ is bounded by
    \begin{align*}
        |\mathcal{E}| &\le \frac{C_9}{D}+\frac{C_{10}}{N^{\alpha_{\min}\left(1-\frac{1}{\alpha_{\tau}}\right)}},
    \end{align*}
    where $C_9$ and $C_{10}$ are constants that only depend on $\alpha$ and $\lh$, but not $h, N, D$.
\end{theorem}

\begin{proof}
    We decompose the bias term for domain $\tau$ into two components: the shared head ($k \le H$) and the disjoint tail ($H < k \le x_{\tau}^*$). By definition, the eigenvalue for a tail component of domain $\tau$ in the mixture covariance is $\lambda_k^{(\tau)} = h_{\tau}k^{-\alpha_{\tau}}$. Thus, we can write:
    \begin{align} \label{eq:bias_decomp}
        \text{Bias} &= \sum_{k=1}^H k^{-\alpha_{\tau}} \exp\left(-2D\eta_0\sum_{j=1}^K h_j\lambda^{(j)}_k\right) + \sum_{k=H+1}^{x_{\tau}^*} k^{-\alpha_{\tau}} \exp\left(-2D\eta_0 h_{\tau}k^{-\alpha_{\tau}}\right).
    \end{align}
    
    \textbf{Step 1: Simplify Notation and Bound the Shared Head} \\
    For the shared head ($k \le H$), the mixture eigenvalue is bounded below by $H^{-\alpha_{\max}}$. Let $A$ denote the upper bound for this head term:
    \begin{align*}
        A := \sum_{k=1}^H k^{-\alpha_{\tau}} \exp(-2D\eta_0 H^{-\alpha_{\max}}) \le H \exp(-2D\eta_0 H^{-\alpha_{\max}}).
    \end{align*}
    To simplify the tail term, we introduce the constant $c := 2h_{\tau}D\eta_0$. Let $B$ represent the disjoint tail summation:
    \begin{align*}
        B := \sum_{k=H+1}^{x_{\tau}^*} k^{-\alpha_{\tau}} \exp(-c k^{-\alpha_{\tau}}).
    \end{align*}
    It trivially follows that $B \le \text{Bias} \le A + B$, meaning the gap is bounded by $|\text{Bias} - B| \le A$.
    
    \textbf{Step 2: Convert the Discrete Sum to a Continuous Integral} \\
    We approximate $B$ using the continuous function $f(x) := x^{-\alpha_{\tau}} \exp(-c x^{-\alpha_{\tau}})$. 
    Taking the derivative $f'(x)$, we find that $f(x)$ increases, peaks, and then decreases, achieving its absolute maximum at $x_0 = c^{\frac{1}{\alpha_{\tau}}}$. The approximation error between the discrete sum and the continuous integral is strictly bounded by the total variation of $f(x)$ across the interval:
    \begin{align*}
        \left| B - \int_{H}^{x_{\tau}^*} f(x) \df x \right| \le \int_{H}^{x_{\tau}^*} |f'(x)| \df x \le 2f(x_0) = \frac{2}{c \cdot e} = \frac{1}{eh_{\tau}D\eta_0}.
    \end{align*}
    
    \textbf{Step 3: Evaluate the Continuous Integral} \\
    Applying the change of variables $p = c x^{-\alpha_{\tau}}$, we have $\df p = -\alpha_{\tau} c x^{-\alpha_{\tau}-1} \df x$. The integral transforms into a lower incomplete Gamma function:
    \begin{align*}
        \int_{H}^{x_{\tau}^*} f(x) \df x &= \frac{1}{\alpha_{\tau} c^{1-\frac{1}{\alpha_{\tau}}}} \int_{p_{\min}}^{p_{\max}} p^{-\frac{1}{\alpha_{\tau}}}\exp(-p) \df p,
    \end{align*}
    where the integration limits are $p_{\min} = c (x_{\tau}^*)^{-\alpha_{\tau}}$ and $p_{\max} = c H^{-\alpha_{\tau}}$.
    This integral can be evaluated as the complete Gamma function minus the two truncation tails:
    \begin{align*}
        \int_{p_{\min}}^{p_{\max}} p^{-\frac{1}{\alpha_{\tau}}}\exp(-p) \df p = \Gamma\left(1-\frac{1}{\alpha_{\tau}}\right) - \underbrace{\int_0^{p_{\min}} p^{-\frac{1}{\alpha_{\tau}}}\exp(-p) \df p}_{\text{Lower Tail}} - \underbrace{\int_{p_{\max}}^{\infty} p^{-\frac{1}{\alpha_{\tau}}}\exp(-p) \df p}_{\text{Upper Tail}}.
    \end{align*}
    
    \textbf{Step 4: Bound the Truncation Tails} \\
    For the lower tail, since $\exp(-p) \le 1$:
    \begin{align*}
        \int_0^{p_{\min}} p^{-\frac{1}{\alpha_{\tau}}}\exp(-p) \df p \le \int_0^{p_{\min}} p^{-\frac{1}{\alpha_{\tau}}} \df p = \frac{\alpha_{\tau}}{\alpha_{\tau}-1} p_{\min}^{1-\frac{1}{\alpha_{\tau}}} = \frac{\alpha_{\tau}}{\alpha_{\tau}-1} \left( \frac{c}{(x_{\tau}^*)^{\alpha_{\tau}}} \right)^{1-\frac{1}{\alpha_{\tau}}}.
    \end{align*}
    For the upper tail, since $p_{\max} > 1$, we have $p^{-\frac{1}{\alpha_{\tau}}} \le p_{\max}^{-\frac{1}{\alpha_{\tau}}}$ for all $p \ge p_{\max}$:
    \begin{align*}
        \int_{p_{\max}}^{\infty} p^{-\frac{1}{\alpha_{\tau}}}\exp(-p) \df p \le p_{\max}^{-\frac{1}{\alpha_{\tau}}} \int_{p_{\max}}^{\infty} \exp(-p) \df p = p_{\max}^{-\frac{1}{\alpha_{\tau}}} \exp(-p_{\max}) = \frac{H}{c^{\frac{1}{\alpha_{\tau}}}} \exp\left(-\frac{c}{H^{\alpha_{\tau}}}\right).
    \end{align*}
    
    \textbf{Step 5: Combine Errors to Bound $\mathcal{E}$} \\
    Let $M$ be the principal order (the main term) derived from the complete Gamma function:
    \begin{align*}
        M := \frac{\Gamma\left(1-\frac{1}{\alpha_{\tau}}\right)}{\alpha_{\tau} c^{1-\frac{1}{\alpha_{\tau}}}} = \frac{\Gamma\left(1-\frac{1}{\alpha_{\tau}}\right)}{\alpha_{\tau}(2\eta_0)^{1-\frac{1}{\alpha_{\tau}}}} \frac{1}{\left(Dh_{\tau}\right)^{1-\frac{1}{\alpha_{\tau}}}}.
    \end{align*}
    The total error $\mathcal{E} = \text{Bias} - M$ is bounded by the sum of all accumulated discrepancies: the head error ($A$), the discrete-to-continuous gap, and the integral truncation tails (multiplied by the prefactor $\frac{1}{\alpha_{\tau} c^{1-1/\alpha_{\tau}}}$). 
    \begin{align*}
        |\mathcal{E}| &\le A + \left| B - \int f(x)\df x \right| + (\text{Lower Tail Error}) + (\text{Upper Tail Error}) \\
        &\le H \exp(-2D\eta_0 H^{-\alpha_{\max}}) + \frac{1}{e h_{\tau}D\eta_0} + \frac{1}{\alpha_{\tau}-1} \frac{1}{(x_{\tau}^*)^{\alpha_{\tau}-1}} + \frac{H}{\alpha_{\tau} c} \exp\left(-\frac{c}{H^{\alpha_{\tau}}}\right).
    \end{align*}
    By applying the inequality $\exp(x) \ge ex$ to the exponential terms and utilizing $x_{\tau}^* = \Omega(N^{\frac{\alpha_{\min}}{\alpha_{\tau}}})$ from Lemma~\ref{thm:propertyx}, these error terms are bounded asymptotically. Gathering the constants into $C_9$ and $C_{10}$ (which depend strictly on $\alpha$ and the mixture lower bound $\underline{h}$), we obtain the final explicit bound:
    \begin{align*}
        |\mathcal{E}| \le \frac{C_9}{D} + \frac{C_{10}}{N^{\alpha_{\min}\left(1-\frac{1}{\alpha_{\tau}}\right)}}.
    \end{align*}
\end{proof}

\subsection{The Variance Term}

In this section, we analyze the variance terms $\text{Var}_1, \text{Var}_2$ for an arbitrary domain $\tau$. We provide a bound as follows. 
\begin{theorem}\label{thm:c10}
Consider the projected linear regression model defined in \Cref{app:theory_des}. For any $N,D$, any mixture $h\in\Delta^{K-1},h_i\ge \lh$, and for an arbitrary domain $\tau$, the terms $\text{Var}_1$, $\text{Var}_2$ in \Cref{thm:c3} satisfy
    \begin{align*}
        \text{Var}_1 &= h_{\tau}\sigma^2_{\tau} C_{\alpha_{\tau}} I_\gamma \eta_0^{1/\alpha_{\tau}} (Dh_{\tau})^{1/\alpha_{\tau} - 1} + \mathcal{E}_1\\
    \text{Var}_1 \le \text{Var}_2 &\le \text{Var}_1+\frac{C_{14}}{D^{1-\frac{1}{\alpha_{\tau}}+\varepsilon}}.
    \end{align*}
    where %$C_{\alpha_{\tau}}$ and $I_\gamma$ are the absolute constants defined in Lemma~\ref{lemac12} and 
    \begin{align*}
        |\mathcal{E}_1|\le C_1(Dh_{\tau})^{-2/3}+C_2DN^{1-2\alpha_{\tau}},
    \end{align*}
    and $C_{\alpha_{\tau}}, I_\gamma, C_{14}$ are constants that depend only on $\eta_0,\alpha,H$ but not on mixture $h$.
\end{theorem}
\subsubsection{Proof of Theorem~\ref{thm:c10}}

Recall that

\begin{align*}
     \text{Var}_1&:=\sum\limits_{k: \, r_k \le N,k>H} h_{\tau}\sigma^2_{\tau} a_k\lambda_k \int_0^{D\eta_0} \exp(-2\lambda_k(D\eta_0-t))\gamma(t)\df t ,\\
    \text{Var}_2&:=\sum\limits_{k: \, r_k\le N,k>H} a_k \lambda_k\int_0^{D\eta_0}\exp(-2\lambda_k(D\eta_0-t))\gamma(t)\left(h_{\tau} \sigma^2_{\tau}+ \frac{C_0}{\lh} \E \left[\w(t)^{\top}\Hm \w(t)\right]\right)\df t \\
        &+\sum\limits_{k: \, r_k\le N,k\le H} a_k \lambda_k\int_0^{D\eta_0}\exp(-2\lambda_k(D\eta_0-t))\gamma(t)\left(\sigma_{\max}^2+ \frac{C_0}{\lh} \E \left[\w(t)^{\top}\Hm \w(t)\right]\right)\df t.
\end{align*}

We start with analyzing $\gamma(t)$. As we defined in Lemma~\ref{lemac.7}, $\gamma(t):=\eta(\tau^{-1}(t))$. In the following lemma, we derive the exact closed form of $\gamma(\cdot )$.
\begin{lemma}
    Let $G: \mathbb{R} \to \mathbb{R}$ be the inverse of the map $y \mapsto y + \sin y$. We have $$\gamma(t)=\eta_0\left(1+\cos G\left(\frac{t\pi}{D\eta_0} \right) \right).$$
\end{lemma}
\begin{proof}
    Since $\eta(x)=\eta_0\left( 1+\cos \left(\frac{\pi k}{D} \right)\right),$ we have \begin{align*}
        \tau(x)&=\int_0^x \eta(a)\df a \\
        &=\eta_0\frac{D}{\pi}\left(\frac{\pi}{D}x +\sin \left(\frac{\pi}{D}x \right)\right).
    \end{align*}
    When $\tau(x)=t$, we have
    \begin{align*}
        x=\frac{D}{\pi}G\left(\frac{t\pi}{D\eta_0} \right).
    \end{align*}
    Therefore
    \begin{align*}
        \gamma(t)&=\eta_0\left(1+\cos G\left(\frac{t\pi}{D\eta_0} \right) \right).
    \end{align*}
\end{proof}

We first analyze the common term $\sum_{k=1}^N h k^{-2\alpha} \int_0^{D\eta_0} \exp\left(-2hk^{-\alpha}(D\eta_0-t)\right) \gamma(t) \sigma^2\df t$ that is shared by $\text{Var}_1$ and $\text{Var}_2$. 
\begin{lemma}\label{lemac12}
Let $h, \alpha, \eta_0, \sigma > 0$ be constants such that $1 < \alpha < 2$. For an integer $N \in \mathbb{N}^+$ and a continuous variable $D > 0$, consider the sum: 
$$S_1 = \sum_{k=1}^N h k^{-2\alpha} \int_0^{D\eta_0} \exp\left(-2hk^{-\alpha}(D\eta_0-t)\right) \gamma(t) \sigma^2\df t,$$
where $G: \mathbb{R} \to \mathbb{R}$ is the inverse function of $y \mapsto y + \sin y$, and $\gamma(t) = \eta_0\left(1+\cos G\left(\frac{t\pi}{D\eta_0} \right) \right)$. 
Then we have
$$S_1 = \sigma^2 C_\alpha I_\gamma \eta_0^{1/\alpha} (Dh)^{1/\alpha - 1} + \mathcal{E},$$
where $C_\alpha = \frac{2^{1/\alpha - 2}}{\alpha} \Gamma(2 - 1/\alpha)$ and $I_\gamma = \int_0^1 \left[1+\cos G(\pi(1-v))\right] v^{1/\alpha - 2} \df v$ are finite constants. 
The approximation error $\mathcal{E}$ is bounded by 
\begin{align*}
    |\mathcal{E}| &\le C_6(Dh)^{-2/3}+C_2DN^{1-2\alpha},
\end{align*}
where $C_6$ and $C_2$ depend only on $\alpha,\sigma,\eta_0$.
\end{lemma}
\begin{proof}
Let $T = D\eta_0$ and let $s = T-t$. Define the function 
$$f(x) := h x^{-2\alpha} \int_0^T \exp\left(-2hx^{-\alpha}s\right) \gamma(T-s) \sigma^2 ds.$$
The target summation can be written as $S_1 = \sum_{k=1}^N f(k)$. 

First, we need to bound the schedule function $\gamma(T-s)$ near $s=0$. By definition, $\gamma(T-s) = \eta_0(1 - \cos x_s)$ where $x_s - \sin x_s = \frac{\pi s}{T}$. For $x_s \in [0, \pi]$, the function $(x_s-\sin x_s)/x_s^3$ achieves its minimum $1/\pi^2$ at $x_s = \pi$. Thus, $x_s - \sin x_s \ge x_s^3/\pi^2$, which implies $x_s \le \pi (s/T)^{1/3}$. Utilizing $1 - \cos x_s \le x_s^2/2$, we obtain a strict global bound:
\begin{align}
    \gamma(T-s) \le \eta_0 \frac{\pi^2}{2} \left(\frac{s}{T}\right)^{2/3}. \label{eq:gamma_bound}
\end{align}

Now we bound the sum-to-integral gap $|S_1 - \int_1^N f(x) \df x|$ by $|S_1 - \int_1^N f(x) \df x| \le \int_1^N |f'(x)| \df x$. This inequality follows from analyzing the error on each subinterval $[k, k+1]$. Specifically, the difference can be expressed as $\int_k^{k+1} (f(k) - f(x)) \df x$. Since $|f(k) - f(x)| \le \int_k^{k+1} |f'(t)| dt$ for any $x \in [k, k+1]$, summing these local bounds from $k=1$ to $N-1$ gives the desired result.

Now we take the derivative of $f(x)$ with respect to $x$:
\begin{align*}
    |f'(x)| = \int_0^T h x^{-2\alpha-1} \exp\left(-2hx^{-\alpha}s\right) \left( 2\alpha + 2\alpha h s x^{-\alpha} \right) \gamma(T-s) \sigma^2 \df s.
\end{align*}
Therefore, we have
\begin{align*}
    \int_1^{N}|f'(x)| \df x & = \int_1^{N} \int_0^T h x^{-2\alpha-1} \exp\left(-2hx^{-\alpha}s\right) \left( 2\alpha + 2\alpha h s x^{-\alpha} \right) \gamma(T-s) \sigma^2 \df s \df x \\
    & = \int_0^{T}\gamma(T-s)\sigma^2 \int_{1}^N h x^{-2\alpha-1} \exp\left(-2hx^{-\alpha}s\right) \left( 2\alpha + 2\alpha h s x^{-\alpha} \right)  \df x \df s \\
    &\le \int_0^{T}\gamma(T-s)\sigma^2 \int_1^{\infty}\left(2\alpha h x^{-2\alpha-1} \exp({-2hx^{-\alpha}s})+2\alpha s h^2 x^{-3\alpha-1} \exp({-2hx^{-\alpha}s}) \right) \df x \df s.
\end{align*}
Applying the change of variable $v = x^{-\alpha}$ gives  $$\int_1^\infty 2\alpha h x^{-2\alpha-1} \exp({-2hx^{-\alpha}s}) \df x = 2h \int_0^1 v e^{-2hsv} dv \le \min\left(h, \frac{1}{2hs^2}\right),$$ and similarly the second term yields $\min\left(\frac{2}{3}h^2 s, \frac{1}{2hs^2}\right)$.  Thus, the total variation is bounded by: 
\begin{align*}
    \int_1^N |f'(x)| \df x \le \sigma^2 \int_0^T \gamma(T-s) \left[ \min\left(h, \frac{1}{2hs^2}\right) + \min\left(\frac{2}{3}h^2s, \frac{1}{2hs^2}\right) \right] \df s.
\end{align*}
Splitting the integral at $s = 1/h$ and substituting the bound \eqref{eq:gamma_bound}, we get:
\begin{align}\label{eq:291}
    \int_1^N |f'(x)| \df x &\le \sigma^2 \int_0^{1/h} \eta_0 \frac{\pi^2}{2} \left(\frac{s}{T}\right)^{2/3} \left(\frac{5}{3}h\right) \df s + \sigma^2 \int_{1/h}^T \eta_0 \frac{\pi^2}{2} \left(\frac{s}{T}\right)^{2/3} \frac{1}{hs^2} \df s \\
    &= \frac{5\pi^2}{6} \sigma^2 \eta_0 T^{-2/3} h \left[ \frac{3}{5} \left(\frac{1}{h}\right)^{5/3} \right] + \frac{\pi^2}{2} \sigma^2 \eta_0 T^{-2/3} \frac{1}{h} \left[ 3 \left(\frac{1}{h}\right)^{-1/3} \right] \\
    &= 2\pi^2 \sigma^2 \eta_0 (hT)^{-2/3} = 2\pi^2 \sigma^2 \eta_0^{1/3} (Dh)^{-2/3}.
\end{align}

Next, we evaluate the continuous continuous integral $\int_1^N f(x) \df x$. Since integrate over $(1,N)$ requires incomplete gamma function, we split $\int_1^N$ into $\int_0^{\infty}-\int_0^1-\int_N^{\infty}$. We first integrate over $x \in (0, \infty)$ and extract the main order $M:=\int_0^{\infty}f(x) \df x$. Swapping the integration order yields:
\begin{align*}
    \int_0^\infty f(x) \df x &= \sigma^2 \int_0^T \gamma(T-s) \left( \int_0^\infty h x^{-2\alpha} \exp(-2hx^{-\alpha}s) \df x \right) \df s \\
    &= \sigma^2 C_\alpha h^{1/\alpha - 1} \int_0^T \gamma(T-s) s^{1/\alpha - 2} \df s.
\end{align*}
Using the dimensionless variable $v = s/T$, we factor out $T$ to match the stated integral $I_\gamma$:
\begin{align*}
    \int_0^\infty f(x) \df x &= \sigma^2 C_\alpha I_\gamma \eta_0^{1/\alpha} (Dh)^{1/\alpha - 1} =: M.
\end{align*}

Let $\mathcal{E}:=|S_1-M|$, we decompose the error as $\mathcal{E} = \left(S_1 - \int_1^N f(x)\df x\right) - \int_0^1 f(x)\df x - \int_N^\infty f(x)\df x$. We bound the upper and lower truncation tails individually.
For the upper integral tail $x \in [0, 1]$, we substitute \eqref{eq:gamma_bound} into $f(x)$ and evaluate the $s$-integral exactly:
\begin{align*}
    f(x) &\le \sigma^2 \int_0^\infty h x^{-2\alpha} \exp(-2hx^{-\alpha}s) \left[ \eta_0 \frac{\pi^2}{2} \left(\frac{s}{T}\right)^{2/3} \right] \df s \\
    &= \frac{\pi^2}{2} \sigma^2 \eta_0 h x^{-2\alpha} T^{-2/3} \Gamma(5/3) (2hx^{-\alpha})^{-5/3} = \frac{\pi^2 \Gamma(5/3)}{2^{8/3}} \sigma^2 \eta_0 T^{-2/3} h^{-2/3} x^{-\alpha/3}.
\end{align*}
Here, $\Gamma(\cdot)$ denotes the Gamma function, defined as$$\Gamma(x) = \int_0^\infty t^{x-1} e^{-t} \, dt$$for $x > 0$ (or $\text{Re}(x) > 0$). Note that this is distinct from the previously defined mapping function $\Gamma(\cdot, \cdot)$, which takes two arguments.
Integrating this bounding function over $x \in [0, 1]$ directly gives:
\begin{align}\label{eq:292}
    \int_0^1 f(x) \df x \le \frac{\pi^2 \Gamma(5/3)}{2^{8/3}} \sigma^2 \eta_0 (hT)^{-2/3} \int_0^1 x^{-\alpha/3} \df x = \frac{3\pi^2 \Gamma(5/3)}{2^{8/3}(3-\alpha)} \sigma^2 \eta_0^{1/3} (Dh)^{-2/3}.
\end{align}

For the lower integral tail $x \in [N, \infty)$, we use the trivial bound $\gamma(T-s) \le 2\eta_0$:
\begin{align}\label{eq:293}
    \int_N^\infty f(x) \df x &\le \sigma^2 \int_N^\infty \left( \int_0^T h x^{-2\alpha} \exp(-2hx^{-\alpha}s) (2\eta_0) \df s \right) \df x \\
    &\le \sigma^2 \int_N^\infty h x^{-2\alpha} (2\eta_0 T) \df x = \frac{2\eta_0 \sigma^2 h T}{2\alpha-1} N^{1-2\alpha} = \frac{2\eta_0^2 \sigma^2 Dh}{2\alpha-1} N^{1-2\alpha}.
\end{align}

Combining all error sources yields
\begin{align*}
    \mathcal{E} & \le  2\pi^2 \sigma^2 \eta_0^{1/3} (Dh)^{-2/3}+\frac{3\pi^2 \Gamma(5/3)}{2^{8/3}(3-\alpha)} \sigma^2 \eta_0^{1/3} (Dh)^{-2/3}+\frac{2\eta_0^2 \sigma^2 Dh}{2\alpha-1} N^{1-2\alpha} \\
    & \le C_6(Dh)^{-2/3}+C_2DN^{1-2\alpha},
\end{align*}
where $C_6$ and $C_2$ depend only on $\alpha,\sigma,\eta_0$.
\end{proof}

Now we focus on $\text{Var}_1$ and $\text{Var}_2$. These terms differ only through the contribution of $\E[\w^{\top}(t)\Hm \w(t)]$, and $\E[\w^{\top}(t)\Hm \w(t)]+\sigma^2$ is exactly the expected training loss. We first show that the expected training loss satisfies the following lemma.
\begin{lemma}\label{thm:c12}
    $$\E \left[\w^{\top}(t)\Hm \w(t)\right]\le \min\left( \frac{C'}{t^{1-\frac{1}{\alpha_{\min}}}}+\frac{C''}{N^{\alpha_{\min}\left(1-\frac{1}{\alpha_{\max}} \right)}}+\frac{C'''}{D^{1-\frac{1}{\alpha_{\min}}}},L\right),$$ where $L$ is defined in Assumption~\ref{assm:c5} and $C',C'',C'''$ are constants that only depend on $\alpha,H$.
\end{lemma}
\begin{proof}
Since $\Hm=\sum h_i \mathbf{A}_i$, we have
\begin{align*}
    \w^{\top}(t)\Hm \w(t)=\sum\limits_{i=1}^K h_i\w^{\top}(t) \mathbf{A}_i \w(t).
\end{align*}
Now we focus on a certain domain $j$. By Theorem~\ref{thm:c3}, we have
\begin{align*}
    \E[ \w^{\top}(t)\mathbf{A}_j \w(t)]&\le \underbrace{\sum\limits_{k: \, r_k>N} a_k}_{\text{Approx}} \\
     &+\underbrace{\sum\limits_{k:\, r_k\le N} \exp(-2\lambda_kt)a_k}_{\text{Bias}} \\
     &+\underbrace{\sum\limits_{k: \, r_k\le N} a_k \lambda_k\int_0^{t}\exp(-2\lambda_k(D\eta_0-s))\gamma(t)\left( \sigma^2_{\max}+ \frac{C_0}{\lh} \E \left[\w(s)^{\top}\Hm \w(s)\right]\right)\df t}_{\text{Var}_2 \text{ Replacing } h_{\tau}\sigma^2_{\tau} \text{ with } \sigma_{\max}^2},
\end{align*}
where $a_k=u_k^{\top}\mathbf{A}_j u_k$.
Directly applying Theorem~\ref{thm:d4} and Lemma~\ref{thm:propertyx}, we have
\begin{align*}
    \sum\limits_{k: \, r_k>N} a_k &\le \frac{(x^*_j(h))^{1-\alpha_j}}{\alpha_j-1}+\frac{C_{13}}{N^{\alpha_{\min}}}\\
    & \le \mathcal{O}\left(\frac{1}{N^{\frac{\alpha_j-1}{\alpha_j}\alpha_{\min}}}\right)
\end{align*}
where the $\mathcal{O}$ notation does not hide any term that depends on $h$.

By Theorem~\ref{thm:d5}, we have
\begin{align*}
    \sum\limits_{k:\, r_k\le N} \exp(-2\lambda_kD\eta_0)a_k \le \frac{\Gamma\left(1-\frac{1}{\alpha_{j}}\right)}{\alpha_{j}(2\eta_0)^{1-\frac{1}{\alpha_{j}}}} \frac{1}{\left(h_{j}t\right)^{1-\frac{1}{\alpha_{j}}}}+\frac{C_9}{t}+\frac{C_{10}}{N^{\alpha_{\min}\left(1-\frac{1}{\alpha_{j}}\right)}}
\end{align*}
 Let $\sigma'^2:=\sum\limits_{i\in [K]}h_i \sigma^2_i+ \frac{C_0}{\lh} L$. The last term can be upper-bounded as follows:
 \begin{align*}
     &\sum\limits_{k: \, r_k\le N} a_k \lambda_k\int_0^{t}\exp(-2\lambda_k(D\eta_0-s))\gamma(t)\left(\sum\limits_{i\in [K]}h_i \sigma^2_i+ \frac{C_0}{\lh} \E \left[\w(s)^{\top}\Hm \w(s)\right]\right)\df t \\&\le \sum\limits_{k: \, r_k\le N} a_k \lambda_k\int_0^{D\eta_0}\exp(-2\lambda_k(D\eta_0-s))\gamma(t)\sigma'^2\df t. \\
     & \le \sum\limits_{k \in [K]}h_j k^{-2\alpha_j}\int_0^{D\eta_0}\exp(-2hk^{-\alpha_j}(D\eta_0-s))\gamma(t)\sigma'^2\df t.
 \end{align*}
 By Lemma~\ref{lemac12}, we have
 \begin{align*}
     &\sum\limits_{k \in [K]}h_j k^{-2\alpha_j}\int_0^{D\eta_0}\exp(-2hk^{-\alpha_j}(D\eta_0-s))\gamma(t)\sigma'^2\df t \\
     &\le \sigma'^2C_{\alpha}I_{\gamma}\eta_0^{1/\alpha_j}(Dh_j)^{1/\alpha_j-1}+C_6(Dh_j)^{-2/3}.
 \end{align*}

    Combining these three terms, we have
    \begin{align*}
        \E[\w^{\top}(t)\mathbf{A}_j \w(t)] &\le \mathcal{O}\left(\frac{1}{N^{\frac{\alpha_j-1}{\alpha_j}\alpha_{\min}}}\right)+\frac{\Gamma\left(1-\frac{1}{\alpha_{j}}\right)}{\alpha_{j}(2\eta_0)^{1-\frac{1}{\alpha_{j}}}} \frac{1}{\left(h_{j}t\right)^{1-\frac{1}{\alpha_{j}}}}+\frac{C_9}{t}+\frac{C_{10}}{N^{\alpha_{\min}\left(1-\frac{1}{\alpha_{j}}\right)}} \\
        &+\sigma'^2C_{\alpha}I_{\gamma}\eta_0^{1/\alpha_j}(Dh_j)^{1/\alpha_j-1}+C_6(Dh_j)^{-2/3}.
    \end{align*}
    Since $\E \left[\sum\limits_{i=1}^K h_i\w^{\top}(t) \mathbf{A}_i \w(t)\right] \le \max_{j\in [K]}\E[\w^{\top}(t) \mathbf{A}_j \w(t)]$ and $t<D\eta_0 \le N^{\frac{\alpha_{\min}}{1+\varepsilon}}\eta_0$, extracting minimum exponents on $t,N,D$, we have 
    \begin{align*}
        \E \left[\sum\limits_{i=1}^K h_i\w^{\top}(t) \mathbf{A}_i \w(t)\right]\le \frac{C'}{t^{1-\frac{1}{\alpha_{\min}}}}+\frac{C''}{N^{\alpha_{\min}\left(1-\frac{1}{\alpha_{\max}} \right)}}+\frac{C'''}{D^{1-\frac{1}{\alpha_{\min}}}},
    \end{align*}
    where $C',C'',C'''$ are constants that depend only on $\alpha,H$. With Assumption~\ref{assm:c5}, we complete the proof.
\end{proof}

After obtaining a rough bound on $\w^{\top}(t)\Hm \w(t)$, we refine the bound on $\text{Var}_2$. The next lemma analyzes the crucial term in $\text{Var}_2-\text{Var}_1$.

\begin{lemma}\label{lema:c15}
Let $h, \alpha, \alpha_2, \eta_0, C > 0$ be positive constants. For an integer $N \in \mathbb{N}^+$, and a continuous scale parameter $D > 0$, consider the sum:
$$S_2 = \sum_{k=1}^N h k^{-2\alpha} \int_0^{D\eta_0} \exp\left(-2hk^{-\alpha}(D\eta_0-t)\right) \gamma(t) \min(C,t^{-\alpha_2}) dt,$$
where $G: \mathbb{R} \to \mathbb{R}$ is the inverse function of $y \mapsto y + \sin y$, and $\gamma(t) = \eta_0\left(1+\cos G\left(\frac{t\pi}{D\eta_0} \right) \right)$.

The summation $S_2$ is strictly bounded by:
$$S_2 \le C_3 (Dh)^{-2/3} + C_4 D N^{1-2\alpha}+C_5(Dh)^{1/\alpha-1}D^{-1},$$
where $C_3,C_4,C_5$ are constants that depend only on $\alpha,C,\eta_0$.

\end{lemma}

\begin{proof}
Let $T = D\eta_0$. Applying the change of variable $s = T-t$, we define the function

$$f(x) := h x^{-2\alpha} \int_0^T \exp\left(-2hx^{-\alpha}s\right) \gamma(T-s) \min(C, (T-s)^{-\alpha_2}) \df s.$$

The target summation can be exactly written as $S_2 = \sum_{k=1}^N f(k)$.

We define the continuous principal integral as $\mathcal{I} := \int_0^\infty f(x) \df x$ (integrating from 0 to infinity yields simpler results) and
\begin{align}\label{eq:edeco}
\mathcal{E} &:= \left(\sum_{k=1}^N f(k) - \int_1^N f(x)\df x\right) - \int_0^1 f(x)\df x - \int_N^\infty f(x)\df x.
\end{align}

To bound $\left|\sum_{k=1}^N f(k) - \int_1^N f(x)\df x\right|$, we first take the derivative of $f(x)$ with respect to $x$:
\begin{align*}
f'(x) &= \int_0^T h x^{-2\alpha-1} \exp\left(-2hx^{-\alpha}s\right) \left( -2\alpha + 2\alpha h s x^{-\alpha} \right) \gamma(T-s) \min(C, (T-s)^{-\alpha_2}) \df s.
\end{align*}
Therefore,
\begin{align*}
|f'(x)| &\le \int_0^T h x^{-2\alpha-1} \exp\left(-2hx^{-\alpha}s\right) \left| -2\alpha + 2\alpha h s x^{-\alpha} \right| \gamma(T-s) \min(C, (T-s)^{-\alpha_2}) \df s \\
&\le \int_0^T h x^{-2\alpha-1} \exp\left(-2hx^{-\alpha}s\right) \left( 2\alpha + 2\alpha h s x^{-\alpha} \right) \gamma(T-s) C \df s.
\end{align*}

By Equation~\ref{eq:edeco}, we obtain: 
\begin{align*}
|\mathcal{E}| &\le \int_1^N |f'(x)| \df x + \int_0^1 f(x) \df x + \int_N^\infty f(x) \df x \\
&\le 2\pi^2 C \eta_0^{1/3} (Dh)^{-2/3} + \frac{3\pi^2 \Gamma(5/3)}{2^{8/3}(3-\alpha)} C \eta_0^{1/3} (Dh)^{-2/3} + \frac{2\eta_0^2 C h}{2\alpha-1} D N^{1-2\alpha} \\
&= C_3 (Dh)^{-2/3} + C_4 DN^{1-2\alpha}.
\end{align*}
The second inequality follows the same steps as in~\Cref{eq:291,eq:292,eq:293}.
 
Next, we evaluate $\mathcal{I}$. Swapping the order of integration yields:
\begin{align*}
\mathcal{I} &=\int_0^\infty f(x) \df x= \int_0^T \gamma(T-s) \min(C, (T-s)^{-\alpha_2}) \left( \int_0^\infty h x^{-2\alpha} \exp\left(-2hx^{-\alpha}s\right) \df x \right) \df s.
\end{align*}
Applying the substitution $v = x^{-\alpha}$ with $\df x = -\frac{1}{\alpha} v^{-1/\alpha - 1} \df v$, the inner integral evaluates to $C_\alpha h^{1/\alpha-1} s^{1/\alpha-2}$, where $C_\alpha = \frac{2^{1/\alpha - 2}}{\alpha} \Gamma(2 - 1/\alpha)$. Reversing the initial temporal substitution ($t = T-s$), the principal integral becomes exactly:
\begin{align*}
\mathcal{I} &= C_\alpha h^{1/\alpha-1} \int_0^T (T-t)^{1/\alpha-2} \gamma(t) \min(C, t^{-\alpha_2}) \df t.
\end{align*}

Since $\min(C, t^{-\alpha_2})$ is a constant when $t$ is small, we split the integration domain at $t = T/2$ such that $$\mathcal{I} = \underbrace{C_\alpha h^{1/\alpha-1} \int_0^{\frac{T}{2}} (T-t)^{1/\alpha-2} \gamma(t) \min(C, t^{-\alpha_2}) \df t}_{\mathcal{I}_1} +\underbrace{C_\alpha h^{1/\alpha-1} \int_{\frac{T}{2}}^{T} (T-t)^{1/\alpha-2} \gamma(t) \min(C, t^{-\alpha_2}) \df t}_{\mathcal{I}_2} .$$

For the first half-domain $t \in [0, T/2]$, the term $(T-t) \ge T/2$. Since $1/\alpha - 2 < 0$, we have $(T-t)^{1/\alpha-2} \le (T/2)^{1/\alpha-2}$. Furthermore, $\gamma(t) \le 2\eta_0$. We bound the integral of the cutoff function over this local domain by its global integral over $(0, \infty)$:
\begin{align*}
\int_0^{T/2} \min(C, t^{-\alpha_2}) \df t &\le \int_0^{C^{-1/\alpha_2}} C \df t + \int_{C^{-1/\alpha_2}}^\infty t^{-\alpha_2} \df t \\
&= C^{1 - 1/\alpha_2} + \frac{1}{\alpha_2 - 1} C^{1 - 1/\alpha_2} \\
&= \frac{\alpha_2}{\alpha_2 - 1} C^{1 - 1/\alpha_2}.
\end{align*}
Multiplying these individual maximum bounds, we obtain:
\begin{align*}
\mathcal{I}_1 &\le C_\alpha h^{1/\alpha-1} \left( \frac{T}{2} \right)^{1/\alpha-2} (2\eta_0) \left( \frac{\alpha_2}{\alpha_2 - 1} C^{1 - 1/\alpha_2} \right) \\
&= \frac{2^{3-1/\alpha} \alpha_2 C_\alpha \eta_0^{1/\alpha-1}}{\alpha_2-1} C^{1-1/\alpha_2} \frac{(Dh)^{1/\alpha-1}}{D}.
\end{align*}

For the second half-domain $t \in [T/2, T]$, the variable is bounded away from zero. Consequently, the minimum function is strictly bounded by its algebraic tail: $\min(C, t^{-\alpha_2}) \le t^{-\alpha_2} \le (T/2)^{-\alpha_2} = 2^{\alpha_2} T^{-\alpha_2}$. Factoring out this upper bound allows us to conservatively extend the remaining integral to the full domain $[0, T]$:
\begin{align*}
\mathcal{I}_2 &\le 2^{\alpha_2} T^{-\alpha_2} C_\alpha h^{1/\alpha-1} \int_{T/2}^T (T-t)^{1/\alpha-2} \gamma(t) \df t \\
&\le 2^{\alpha_2} T^{-\alpha_2} C_\alpha h^{1/\alpha-1} \int_0^T (T-t)^{1/\alpha-2} \gamma(t) \df t.
\end{align*}
Applying the dimensionless change of variable $v = 1 - t/T$ sets $\df t = -T \df v$ and $T-t = Tv$, fully recovering the constant $I_\gamma$ from Lemma~\ref{lemac12}:
\begin{align*}
\int_0^T (T-t)^{1/\alpha-2} \gamma(t) \df t &= T^{1/\alpha-1} \eta_0 \int_0^1 v^{1/\alpha-2} \left[1+\cos G(\pi(1-v))\right] \df v \\
&= T^{1/\alpha-1} \eta_0 I_\gamma.
\end{align*}
Substituting this back and recalling $T=D\eta_0$, we obtain:
\begin{align*}
\mathcal{I}_2 &\le 2^{\alpha_2} (D\eta_0)^{-\alpha_2} C_\alpha h^{1/\alpha-1} \left( (D\eta_0)^{1/\alpha-1} \eta_0 I_\gamma \right) \\
&= 2^{\alpha_2} C_\alpha I_\gamma \eta_0^{1/\alpha-\alpha_2} \frac{(Dh)^{1/\alpha-1}}{D^{\alpha_2}}.
\end{align*}

Summing the evaluated components gives the final strict bound:
\begin{align*}
S_2 &\le \mathcal{I}_1 + \mathcal{I}_2 + |\mathcal{E}| \\
&\le C_3 (Dh)^{-2/3} + C_4 D N^{1-2\alpha} + \frac{2^{3-1/\alpha} \alpha_2 C_\alpha \eta_0^{1/\alpha-1}}{\alpha_2-1} C^{1-1/\alpha_2} \frac{(Dh)^{1/\alpha-1}}{D} + 2^{\alpha_2} C_\alpha I_\gamma \eta_0^{1/\alpha-\alpha_2} \frac{(Dh)^{1/\alpha-1}}{D^{\alpha_2}} \\
&\le C_3 (Dh)^{-2/3} + C_4 D N^{1-2\alpha} + C_5(Dh)^{1/\alpha-1}D^{-1},
\end{align*}
where $C_3, C_4, C_5$ are absolute constants that only depend on $\alpha, \alpha_2, C, \eta_0$, completing the proof.
\end{proof}

We are now ready to prove Theorem~\ref{thm:c10}.

Let 
\begin{align*}
    \mathcal{E}_1:=\text{Var}_1-h_{\tau}\sigma^2_{\tau}  C_{\alpha_{\tau}} I_\gamma \eta_0^{1/\alpha_{\tau}} (Dh_{\tau})^{1/\alpha_{\tau} - 1},
\end{align*}
where $C_{\alpha_{\tau}}$ and $I_\gamma$ are the  constants defined in Lemma~\ref{lemac12}. Then by Lemma~\ref{lemac12}, $\mathcal{E}_1$ satisfies
\begin{align*}
    |\mathcal{E}_1| &\le C_6(Dh_{\tau})^{-2/3}+C_2DN^{1-2\alpha_{\tau}}+\int_0^{D\eta_0}\exp(-2H^{-\alpha_{\max}}(D\eta_0-t))\gamma(t)\sigma'^2\df t \\
    &\le C_6(Dh_{\tau})^{-2/3}+C_2DN^{1-2\alpha_{\tau}}+\frac{\pi^2}{2}  (D\eta_0)^{-2/3} \Gamma(5/3) (2h_{\tau}H^{-\alpha_{\max}})^{-5/3} \\
    &\le C_1(Dh_{\tau})^{-2/3}+C_2DN^{1-2\alpha_{\tau}}.
\end{align*}
The integral
\begin{equation*}
    \int_0^{D\eta_0}\exp\left(-2H^{-\alpha_{\max}}(D\eta_0-t)\right)
    \gamma(t)\sigma'^2\df t
\end{equation*}
is induced by the error of the shared head. Moreover,
\begin{equation*}
    C_1=C_6+\frac{\pi^2}{2}\eta_0^{-2/3}\Gamma(5/3)H^{5\alpha/3},
\end{equation*}
which does not depend on $h$.

By Lemma~\ref{thm:c12}, we have
\begin{equation*}
\begin{aligned}
R(t):={}&\min\left(\frac{C'}{t^{1-\frac{1}{\alpha_{\min}}}},L\right)
+\frac{C''}{N^{\alpha_{\min}\left(1-\frac{1}{\alpha_{\max}}\right)}} \\
&+\frac{C'''}{D^{1-\frac{1}{\alpha_{\min}}}}.
\end{aligned}
\end{equation*}
\begin{align*}
	     \text{Var}_1\le \text{Var}_2
         \le \text{Var}_1+\frac{C_0}{\lh}\sum a_k\lambda_k
         \int_0^{D\eta_0}\exp(-2\lambda_k(D\eta_0-t))
         \gamma(t)R(t)\df t.
\end{align*}
By Lemma~\ref{lema:c15}, we have
\begin{align*}
    &\frac{C_0}{\lh}\sum a_k\lambda_k
    \int_0^{D\eta_0}\exp(-2\lambda_k(D\eta_0-t))
    \gamma(t)R(t)\df t \\
    &\le (C_1+C_3)(Dh_{\tau})^{-2/3}
    +C_4 D N^{1-2\alpha_{\tau}}
    + C_5(Dh_{\tau})^{1/\alpha_{\tau}-1}D^{-1} \\
    &+\left(h_{\tau}\sigma^2_{\tau} C_{\alpha_{\tau}} I_\gamma \eta_0^{1/\alpha_{\tau}} (Dh_{\tau})^{1/\alpha_{\tau} - 1}+|\mathcal{E}_1| \right)\cdot \left(\frac{C''}{N^{\alpha_{\min}\left(1-\frac{1}{\alpha_{\max}} \right)}}+\frac{C'''}{D^{1-\frac{1}{\alpha_{\min}}}} \right).
\end{align*}

Therefore, 
\begin{align*}
    \text{Var}_2&\le \text{Var}_1+(C_1+C_3)(Dh_{\tau})^{-2/3}
    +C_4 D N^{1-2\alpha_{\tau}}
    + C_5(Dh_{\tau})^{1/\alpha_{\tau}-1}D^{-1}\\ &+\left(h_{\tau}\sigma^2_{\tau}  C_{\alpha_{\tau}} I_\gamma \eta_0^{1/\alpha_{\tau}} (Dh_{\tau})^{1/\alpha_{\tau} - 1}+|\mathcal{E}_1| \right)\cdot \left(\frac{C''}{N^{\alpha_{\min}\left(1-\frac{1}{\alpha_{\max}} \right)}}+\frac{C'''}{D^{1-\frac{1}{\alpha_{\min}}}} \right) \\
    & \le \text{Var}_1+\frac{C_{14}}{D^{ 1-\frac{1}{\alpha_{\tau}}+\varepsilon}}.
\end{align*}
The last inequality follows by taking the minimum exponent of $D$; consequently, $C_{14}$ is independent of $h$.
\subsection{Proof of the Main Theorem} \label{app:main_theorem}

Combining Theorem~\ref{thm:d4}, Theorem~\ref{thm:d5}, and Theorem~\ref{thm:c10} yields our final main theorem under the following additional assumption.

\begin{assumption} \label{assump:range}
We assume that the value of the term $ \frac{\Gamma\left(1-\frac{1}{\alpha_{\tau}}\right)}{\alpha_{\tau}(2\eta_0)^{1-\frac{1}{\alpha_{\tau}}}}$ is much larger than $h_{\tau}\sigma^2_{\tau} C_{\alpha_{\tau}} I_\gamma \eta_0^{1/\alpha_{\tau}}$ for all $h\in \Delta^{K-1}$, so that 
$$A(h) := \frac{\Gamma\left(1-\frac{1}{\alpha_{\tau}}\right)}{\alpha_{\tau}(2\eta_0)^{1-\frac{1}{\alpha_{\tau}}}}+h_{\tau}\sigma^2_{\tau}  C_{\alpha_{\tau}} I_\gamma \eta_0^{1/\alpha_{\tau}}  \approx \text{a constant } A. $$
Note that when $\alpha_\tau \in (1,3)$, $\eta_0 \in[0.1,1]$, and $\sigma_\tau^2\le 1$, we have 
$A(h)\in [8,9]$.
\end{assumption}

Then our main theorem can be formally stated as follows.

\begin{theorem}[Formal Statement of Theorem~\ref{thm:multidomain_loss}]

Consider the projected linear regression model defined in \Cref{app:theory_des} and an arbitrary domain $\tau \in[K]$. Let $x^*$ be the solution to Problem~\ref{eq:quant_opt} with parameters $c_i=\frac{1}{\alpha_i-1}$ and $b_i=\alpha_i-1$. Then, for any $\epsilon>0$, there exist $N_0$ and $D_0$ such that for all $N>N_0$, $D>D_0$, and $h$ satisfying Assumption~\ref{assm:d2}, the expected test loss on domain $\tau$ is bounded by
\begin{align*}
    L_{\tau}(h \mid N, D) &\ge c_\tau(x^*(h,N)_{\tau})^{-b_\tau}(1-\epsilon) + \frac{A(h)}{\left(Dh_{\tau}\right)^{1-\frac{1}{\alpha_{\tau}}}}(1-\epsilon) + \sigma^2_\tau, \\
    L_{\tau}(h \mid N, D) &\le c_\tau(x^*(h,N)_{\tau})^{-b_\tau}(1+\epsilon) + \frac{A(h)}{\left(Dh_{\tau}\right)^{1-\frac{1}{\alpha_{\tau}}}}(1+\epsilon) + \sigma^2_\tau,
\end{align*} 
where 
$A(h) := \frac{\Gamma\left(1-\frac{1}{\alpha_{\tau}}\right)}{\alpha_{\tau}(2\eta_0)^{1-\frac{1}{\alpha_{\tau}}}}+h_{\tau}\sigma^2_{\tau} C_{\alpha_{\tau}} I_\gamma \eta_0^{1/\alpha_{\tau}}$. With \Cref{assump:range}, we have $A(h) \approx$ a constant $A$, and thus 
\begin{align*}
    L_{\tau}(h \mid N, D) &\ge c_\tau(x^*(h,N)_{\tau})^{-b_\tau}(1-\epsilon) + \frac{A}{\left(Dh_{\tau}\right)^{1-\frac{1}{\alpha_{\tau}}}}(1-\epsilon) + \sigma^2_\tau, \\
    L_{\tau}(h \mid N, D) &\le c_\tau(x^*(h,N)_{\tau})^{-b_\tau}(1+\epsilon) + \frac{A}{\left(Dh_{\tau}\right)^{1-\frac{1}{\alpha_{\tau}}}}(1+\epsilon) + \sigma^2_\tau,
\end{align*} 
for a constant $A$.
\end{theorem}

\begin{proof}
Applying the discretization bound in the proof of Theorem~\ref{thm:d4} and the lower bound on $x_\tau^*$ from Lemma~\ref{thm:propertyx}, we obtain the following relative approximation error for all $h$ satisfying Assumption~\ref{assm:d2}:
\begin{align*}
\frac{\left|\text{Approx}-\frac{(x_\tau^*)^{1-\alpha_\tau}}{\alpha_\tau-1}\right|}
{\frac{(x_\tau^*)^{1-\alpha_\tau}}{\alpha_\tau-1}}
&\le \frac{(K+2)(\alpha_\tau-1)}{x_\tau^*} \\
&\le C_{\mathrm{A}}N^{-\frac{\alpha_{\min}}{\alpha_\tau}},
\end{align*}
where $C_{\mathrm{A}}$ is independent of $h$.

Next, by Theorem~\ref{thm:d5}, for all $h$ satisfying Assumption~\ref{assm:d2}, the relative bias is bounded by:
$$ \begin{aligned}
    \frac{\left|\text{Bias}-\frac{\Gamma\left(1-\frac{1}{\alpha_{\tau}}\right)}{\alpha_{\tau}(2\eta_0)^{1-\frac{1}{\alpha_{\tau}}}} \frac{1}{\left(Dh_{\tau}\right)^{1-\frac{1}{\alpha_{\tau}}}}\right|}{\frac{\Gamma\left(1-\frac{1}{\alpha_{\tau}}\right)}{\alpha_{\tau}(2\eta_0)^{1-\frac{1}{\alpha_{\tau}}}} \frac{1}{\left(Dh_{\tau}\right)^{1-\frac{1}{\alpha_{\tau}}}}}&\le \frac{\frac{C_9}{D}+\frac{C_{10}}{N^{\alpha_{\min}\left(1-\frac{1}{\alpha_{\tau}}\right)}}}{\frac{\Gamma\left(1-\frac{1}{\alpha_{\tau}}\right)}{\alpha_{\tau}(2\eta_0)^{1-\frac{1}{\alpha_{\tau}}}} \frac{1}{\left(Dh_{\tau}\right)^{1-\frac{1}{\alpha_{\tau}}}}}\\
    &\le C_{11} D^{-1/\alpha_{\tau}}+C_{12}D^{-\varepsilon(1-1/\alpha_{\tau})},
\end{aligned} $$
where we define $C_{11} := C_9 \cdot \left(\frac{\Gamma\left(1-\frac{1}{\alpha_{\tau}}\right)}{\alpha_{\tau}(2\eta_0)^{1-\frac{1}{\alpha_{\tau}}}}\right)^{-1}$ which is independent of $h$.

Similarly, define the leading variance term from Theorem~\ref{thm:c10} as
\begin{equation*}
    V_\tau(h,D):=
    h_\tau\sigma_\tau^2 C_{\alpha_{\tau}} I_\gamma \eta_0^{1/\alpha_{\tau}}
    (Dh_{\tau})^{\frac{1}{\alpha_{\tau}}-1}.
\end{equation*}
For all $h$ satisfying Assumption~\ref{assm:d2}, the relative variance error satisfies
\begin{align*}
    \frac{\left|\text{Var}-V_\tau(h,D)\right|}{V_\tau(h,D)}
    &\le C_7D^{-\left(\frac{1}{\alpha_{\tau}}-\frac{1}{3}\right)}
    +C_8D^{-2(1+\varepsilon)\left(1-\frac{1}{\alpha_{\tau}}\right)},
\end{align*}
where $C_7$ and $C_8$ are constants independent of $h$ and $\epsilon$.

%For any $h$ satisfying Assumption~\ref{assm:d2}, 
Under Assumption~\ref{assm:c4}, all decay exponents above are strictly positive: $\frac{\alpha_{\min}}{\alpha_\tau}>0$, $\frac{1}{\alpha_\tau}>0$, $\varepsilon(1-\frac{1}{\alpha_\tau})>0$, $\frac{1}{\alpha_\tau}-\frac{1}{3}>0$, and $2(1+\varepsilon)(1-\frac{1}{\alpha_\tau})>0$. Therefore, we can choose sufficiently large $D_0$ and $N_0$ such that, uniformly over $h$,
\begin{align*}
    C_{\mathrm{A}}N_0^{-\frac{\alpha_{\min}}{\alpha_\tau}} &\le \epsilon, \\
    C_{11} D_0^{-1/\alpha_{\tau}}+C_{12}D_0^{-\varepsilon(1-\frac{1}{\alpha_\tau})} &\le \epsilon, \\
    C_7D_0^{-\left(\frac{1}{\alpha_\tau}-\frac{1}{3}\right)}
    +C_8D_0^{-2(1+\varepsilon)\left(1-\frac{1}{\alpha_\tau}\right)} &\le \epsilon.
\end{align*}

Consequently, for all $N>N_0$ and $D>D_0$ satisfying Assumption~\ref{assm:c6}, for all $h$ satisfying Assumption~\ref{assm:d2}, the relative errors are uniformly bounded by $\epsilon$. This implies that the components fall within the following intervals:
$$ \begin{aligned}
    \text{Approx} &\in \left[\frac{(x^*_{\tau})^{1-\alpha_{\tau}}}{\alpha_{\tau}-1}(1-\epsilon), \, \frac{(x^*_{\tau})^{1-\alpha_{\tau}}}{\alpha_{\tau}-1}(1+\epsilon)\right], \\
    \text{Bias} &\in \left[\frac{\Gamma\left(1-\frac{1}{\alpha_{\tau}}\right)}{\alpha_{\tau}(2\eta_0)^{1-\frac{1}{\alpha_{\tau}}}} \frac{1}{\left(Dh_{\tau}\right)^{1-\frac{1}{\alpha_{\tau}}}}(1-\epsilon), \, \frac{\Gamma\left(1-\frac{1}{\alpha_{\tau}}\right)}{\alpha_{\tau}(2\eta_0)^{1-\frac{1}{\alpha_{\tau}}}} \frac{1}{\left(Dh_{\tau}\right)^{1-\frac{1}{\alpha_{\tau}}}}(1+\epsilon)\right],\\
    \text{Var} &\in \left[V_\tau(h,D)(1-\epsilon), \, V_\tau(h,D)(1+\epsilon)\right].
\end{aligned} $$

Finally, by Theorem~\ref{thm:c3}, adding these three components together completes the proof.
\end{proof}

\section{Proof of Proposition~\ref{thm:bilevel_gradient}}
\label{app:proof_gradient}

The objective function for the bi-level optimization is given by $\mathcal{J}(h) = \sum_{i=1}^K w_i [ A_i (D h_i)^{-a_i} + c_i (x_i^*(h))^{-b_i} ]$. To find the gradient with respect to a mixture weight $h_k$, we apply the chain rule, noting that the data term depends explicitly on $h_k$ while the capacity term depends implicitly on $h$ through the optimal allocation $x^*(h)$:
\begin{equation}
    \frac{d \mathcal{J}}{d h_k} = - w_k a_k A_k D^{-a_k} h_k^{-a_k - 1} + \sum_{i=1}^K w_i \frac{\partial L_i^{\text{cap}}}{\partial x_i} \frac{\partial x_i^*(h)}{\partial h_k}.
\end{equation}
The first term is obtained by direct differentiation. To evaluate the second term (the capacity gradient), we invoke the KKT conditions of the inner optimization problem $\min_{x} \sum h_i c_i x_i^{-b_i}$ subject to $\sum x_i = N$. The stationarity condition implies $\lambda = h_i b_i c_i x_i^{-(b_i+1)}$, which simplifies the marginal capacity loss to $\frac{\partial L_i^{\text{cap}}}{\partial x_i} = - b_i c_i x_i^{-b_i-1} = -\frac{\lambda}{h_i}$.

Next, we derive the Jacobian $\frac{\partial x_i}{\partial h_k}$ using implicit differentiation. Taking the logarithm of the stationarity condition yields $\ln \lambda = \ln h_i + \text{const} - (b_i+1) \ln x_i$. Differentiating both sides with respect to $h_k$ gives:
\begin{equation}
    \frac{1}{\lambda} \frac{\partial \lambda}{\partial h_k} = \frac{\delta_{ik}}{h_i} - \frac{b_i+1}{x_i} \frac{\partial x_i}{\partial h_k}.
\end{equation}
Letting $\gamma_i = \frac{x_i}{b_i+1}$, we can express the sensitivity of the capacity allocation as $\frac{\partial x_i}{\partial h_k} = \gamma_i ( \frac{\delta_{ik}}{h_i} - \frac{1}{\lambda} \frac{\partial \lambda}{\partial h_k} )$. We solve for the unknown multiplier derivative by differentiating the capacity constraint $\sum_j x_j = N$ with respect to $h_k$, which implies $\sum_j \frac{\partial x_j}{\partial h_k} = 0$. Summing the sensitivity equation over all $j$ yields:
\begin{equation}
    \sum_{j=1}^K \gamma_j \left( \frac{\delta_{jk}}{h_j} - \frac{1}{\lambda} \frac{\partial \lambda}{\partial h_k} \right) = 0 \quad \implies \quad \frac{1}{\lambda} \frac{\partial \lambda}{\partial h_k} = \frac{\gamma_k}{Z h_k},
\end{equation}
where $Z = \sum_{j=1}^K \gamma_j$. Substituting this back, we obtain the Jacobian:
\begin{equation}
    \frac{\partial x_i}{\partial h_k} = \gamma_i \left( \frac{\delta_{ik}}{h_i} - \frac{\gamma_k}{Z h_k} \right).
\end{equation}

Finally, we substitute the marginal loss $-\frac{\lambda}{h_i}$ and the Jacobian back into the chain rule equation for the capacity term:
\begin{equation}
\begin{aligned}
    \frac{\partial \mathcal{J}_{\text{cap}}}{\partial h_k} &= \sum_{i=1}^K w_i \left( -\frac{\lambda}{h_i} \right) \gamma_i \left( \frac{\delta_{ik}}{h_i} - \frac{\gamma_k}{Z h_k} \right) \\
    &= -\lambda \left( \frac{w_k \gamma_k}{h_k^2} - \frac{\gamma_k}{Z h_k} \sum_{i=1}^K \frac{w_i \gamma_i}{h_i} \right).
\end{aligned}
\end{equation}
We define the weighted baseline $\bar{R} = \frac{1}{Z} \sum_{i=1}^K \frac{w_i \gamma_i}{h_i}$. Factoring out the common terms and substituting $\gamma_k = \frac{x_k^*}{b_k+1}$, the capacity gradient simplifies to $\frac{\lambda x_k^*}{h_k (b_k+1)} (\bar{R} - \frac{w_k}{h_k})$. Combining this with the data gradient completes the proof.
\qed

%%%%%%%%%%%%%%%%%%%%%%%%%%%%%%%%%%%%%%%%%%%%%%%%%%%%%%%%%%%%%%%%%%%%%%%%%%%%%%%
%%%%%%%%%%%%%%%%%%%%%%%%%%%%%%%%%%%%%%%%%%%%%%%%%%%%%%%%%%%%%%%%%%%%%%%%%%%%%%%

\end{document}